\documentclass{article}
\usepackage{color}
\usepackage{amsmath,amssymb,bm, bbm}
\usepackage{algorithm,algpseudocode}
\usepackage{wrapfig}
\usepackage{booktabs}
\usepackage{caption,subcaption}
\usepackage{tikz, pgf}
\usepackage{afterpage}
\usetikzlibrary{bayesnet}
\usetikzlibrary{calc}
\usepackage[colorlinks=true, plainpages = false, citecolor = black, urlcolor = black, filecolor = black, linkcolor = black]{hyperref}
\usepackage{natbib}
\usepackage{amsthm}

\definecolor{DarkRed}{rgb}{0.75,0,0}
\definecolor{DarkGreen}{rgb}{0,0.5,0}
\definecolor{DarkBlue}{rgb}{0,0,0.5}
\definecolor{DarkPurple}{rgb}{0.5,0,0.5}

\newcommand{\indices}{\mathcal{I}}
\newcommand{\clusa}{\mathcal{A}}
\newcommand{\krn}[2]{\kappa\!\left(#1, #2\right)}
\newcommand{\krnmat}[2]{\kappa^{#1 #2}}
\newcommand{\ind}[1]{\mathbbm{1}\!\left[#1\right]}
\newcommand{\dtk}{ {{\Delta t}_{kt}} }
\newcommand{\dtkn}{ {{\Delta t}_{k(t+1)}} }

\newcommand{\tr}[1]{ \mathrm{tr}\left[#1\right] }
\DeclareMathOperator*{\argmin}{argmin}
\DeclareMathOperator*{\argmax}{argmax}

\newtheorem{theorem}{Theorem}
\newtheorem{lemma}{Lemma}

\makeatletter
\newenvironment{proofof}[1]{\par
  \pushQED{\qed}%
  \normalfont \topsep6\p@\@plus6\p@\relax
  \trivlist
  \item[\hskip\labelsep
        \bfseries
    Proof of #1\@addpunct{.}]\ignorespaces
}{%
  \popQED\endtrivlist\@endpefalse
}
\makeatother

\begin{document}

\title{Dynamic Clustering Algorithms via Small-Variance Analysis of Markov Chain Mixture Models}

\author{Trevor Campbell\\
       Computer Science and Artificial Intelligence Laboratory\\
       Massachusetts Institute of Technology\\
       Cambridge, MA, USA\\
       \vspace{.3cm}
        \texttt{tdjc@mit.edu} \\
       Brian Kulis\\
       Department of Electrical and Computer Engineering\\
       Department of Computer Science\\
       Boston University\\
       Boston, MA, USA\\
      \vspace{.3cm}
        \texttt{bkulis@bu.edu}\\    
      Jonathan P.~How\\
       Laboratory for Information and Decision Systems\\
       Massachusetts Institute of Technology\\
       Cambridge, MA, USA\\
        \texttt{jhow@mit.edu} 
       }


\maketitle

\begin{abstract}
Bayesian nonparametrics are a class of probabilistic models in which the model
size is inferred from data. A
recently developed methodology in this field is \emph{small-variance
asymptotic analysis}, a mathematical technique for deriving 
learning algorithms that capture much of the flexibility of Bayesian nonparametric inference 
algorithms, but are simpler to implement and less computationally expensive. 
Past work on small-variance analysis of Bayesian
nonparametric inference algorithms has exclusively considered
batch models trained on a single, static dataset, which are incapable of
capturing time evolution in the latent structure of the data. This work presents
a small-variance analysis of the  
maximum a posteriori filtering problem for a temporally varying mixture model with a Markov dependence structure, 
which captures temporally evolving clusters
within a dataset. Two clustering algorithms
result from the analysis: D-Means, an iterative clustering algorithm
for linearly separable, spherical clusters; and SD-Means, a spectral 
clustering algorithm derived from a kernelized, relaxed version of the clustering
problem. Empirical results from experiments demonstrate the advantages of 
using D-Means and SD-Means over contemporary clustering 
algorithms, in terms of both computational cost and clustering accuracy.
\end{abstract}


\ifdefined\ieeetpamistyle
\IEEEraisesectionheading{\section{Introduction}\label{sec:introduction}}
\IEEEPARstart{B}{ayesian}
\else
\section{Introduction}
Bayesian
\fi
nonparametrics (BNPs) are a class of priors
which, when used in a probabilistic model, capture uncertainty in the number of
parameters comprising the model. Examples of BNP priors are
the Dirichlet process (DP)~\citep{Ferguson73_ANSTATS,Teh10_EML}, which is often
used in mixture models with an uncertain number of clusters, and the Beta
process (BP)~\citep{Hjort90_AS,Thibaux07_AISTATS}, which is often used 
in latent feature models with an uncertain number of features.
These models are powerful in their capability
to capture complex structures in data without requiring explicit model
selection; however,
they often require complicated and computationally expensive algorithms
for inference~\citep{Neal00_JCG, DoshiVelez09_ICML, Carvalho10_BA, Blei06_BA,
Hoffman12_AX}. This has hindered the wide-scale adoption of BNPs in
many contexts, such as mobile robotics, where performing inference in real-time on large volumes of
streaming data is crucial for timely decision-making.

A recent development in the BNP literature is the mathematical technique of \emph{small-variance asymptotic
analysis}. This technique typically involves examining the limiting behavior of
a BNP inference algorithm as the variance in some part of the model 
is reduced to zero. This yields a ``hard'' learning algorithm (e.g.~one in which cluster or latent
feature assignments are deterministic) that is simple and computationally
tractable, while retaining key characteristics of the original inference
procedure. Such small-variance
asymptotic analyses have been presented for the Gibbs sampling algorithm
for DP mixtures of Gaussians~\citep{Kulis12_ICML} and DP mixtures of exponential
families~\citep{Jiang12_NIPS}. It is also possible to use a similar
mathematical approach to analyze the BNP model
joint probability density itself, typically yielding a cost function
that is related to maximum a posteriori (MAP) estimation for the model 
in the small-variance limit. This approach has been 
applied to the DP Gaussian mixture and BP Gaussian
latent feature models~\citep{Broderick13_ICML}, and infinite-state hidden Markov
models~\citep{Roychowdhury13_NIPS}. 

Despite these advances, past small-variance analyses have focused solely on \emph{static} models, 
i.e.~models for which inference is performed on a single dataset, and whose
parameters are assumed to be fixed for all datapoints in that dataset.
The algorithms derived from previous small-variance asymptotic analyses
thus are not applicable when the data contains a time-evolving latent structure.
Furthermore, because these algorithms must be run on a single dataset, they become intractable as the amount
of collected data increases; they cannot break large datasets into smaller
subsets and process each incrementally, as they have no mechanism for storing
a reduced representation of each subset or transferring information between
them. Thus, they cannot be used on a large volume of streaming, evolving
data collected over a long duration,
such as that obtained by autonomous robotic systems making observations of
complex, dynamic environments~\citep{Endres05_RSS, Luber04_RSS, Wang08_RSS}, or
obtained by monitoring evolution of interactions in social
networks~\citep{Leskovec05_KDD}.

Thus, this work addresses the problem of clustering 
temporally evolving datastreams:
  Given a batch-sequential stream of vectors (i.e.~a sequence of
  vector datasets $\{y_{it}\}_{i=1}^{N_t}$, $t=1, 2, \dots$), discover
  underlying latent cluster centers $\{\theta_{kt}\}_{k=1}^{K_t}$, $t=1, 2, \dots$ and
  their temporal evolution in a manner that remains computationally tractable
  as the total amount of data $N = \sum_{\tau=1}^t N_\tau$ increases over time.
In order to tackle this problem, we employ a small-variance analysis of
the MAP filtering problem for a mixture model with a
covariate-dependent stochastic process prior.
While there are many stochastic processes that capture covariate-dependence and can serve as a prior over
temporally evolving mixtures~\citep{MacEachern99_ASA,Ahmed08_SDM,Blei11_JMLR,Caron07_UAI}, 
this work employs a temporally varying mixture model with a Markov chain time-dependence~\citep{Lin10_NIPS} 
(henceforth referred to as an MCMM). 
This particular model has two major advantages over past work: it explicitly 
captures the dynamics of the latent cluster structure through birth, death, and transition processes, and
allows the processing of data in a batch-sequential fashion.\footnote{It should be noted that the original work 
referred to the MCMM as a dependent Dirichlet process~\citep{Lin10_NIPS,Campbell13_NIPS}. This is actually not the case;
the transition operation on the underlying Poisson point process destroys the ability to form a one-to-one correspondence between
atoms at different timesteps~\citep{Chen13_ICML}. However, the fact that the MCMM lacks this link 
to a previously-developed model does not detract from its important ability to model cluster motion, death, and birth.}
In this paper, we show that analyzing the small-variance limit of the MAP
filtering problem for a MCMM with Gaussian likelihood (MCGMM) yields fast, flexible clustering algorithms for
datastreams with temporally-evolving latent structure.

The specific contributions of this work are as follows:
\begin{itemize}
  \item In Section \ref{sec:asymptotics}, we analyze the small-variance
    limit of the maximum a posteriori filtering problem for the MCGMM, and show that it results in an optimization
      problem with a K-Means-like clustering cost.  We then develop coordinate descent
      steps for this cost in Section \ref{sec:dynmeans}, yielding a batch-sequential clustering algorithm for data whose
    clusters undergo birth, death, and motion processes. Finally, we provide a
    reparameterization of the algorithm
    that is intuitive to tune in practice, and discuss practical aspects of
    efficient implementation.
  \item In Sections \ref{sec:krnspec}-\ref{sec:spectraldynmeans}, we first develop a
    kernelization of the small-variance MAP filering problem, including a
    recursive approximation for old cluster centers with theoretical guarantees 
    on the approximation quality. This extends the applicability of the
    small variance analysis to
    nonlinearly separable and nonspherical clusters. 
    Next, we derive a spectral relaxation of
    the kernelized minimization problem, providing a lower bound on the optimal 
    clustering cost. Based on the relaxation, we present an algorithm that generates a feasible clustering
    solution via coordinate descent and linear programming.
    The combination of kernelization, relaxation, and generation of a feasible
    solution yields a spectral counterpart to the iterative algorithm that is
    more flexible and robust to local optima, at the expense of increased
    computational cost.
  \item Experimental results are presented for both algorithms. Synthetic
    data (moving Gaussians, moving rings, and Gaussian processes) provide a rigorous examination of the performance
    trade-offs in terms of labelling accuracy and
    computation time for both algorithms alongside contemporary
    clustering algorithms. Finally, applications in clustering commercial aircraft
    trajectories, streaming point cloud data, and video color quantization are presented.
\end{itemize}
This paper extends previous work on Dynamic Means~\citep{Campbell13_NIPS}.
Proofs for all the theorems presented herein are provided
in the appendices.

\section{Background: Markov Chain Mixture Model} 
The Dirichlet process (DP) is a prior over discrete probability
measures~\citep{Sethuraman94_SS, Ferguson73_ANSTATS}. In general,
if $D$ is DP distributed, then realizations are probability measures
$D = \sum_k \pi_k \delta_{\theta_k}$ consisting of a countable number 
of discrete atoms $\delta_{\theta_k}$, each with measure $1$ and location
$\theta_k$ in some space, and weights $\pi_k$, with $\sum_k \pi_k =
1$~\citep{Sethuraman94_SS}. The two arguments to a Dirichlet process are
the \emph{concentration parameter}
$\alpha$, which determines how
the weights $\pi_k$ are distributed amongst the atoms, and
the \emph{base measure} $\mu$, which is the distribution from which atom locations $\theta_k$ are
sampled independently.  When used as a prior over mixture models,
the locations $\theta_k$ can be thought of as the parameters of the clusters,
while the weights $\pi_k$ are the weights in the categorical distribution of
labels. The reader is directed to~\citep{Teh10_EML} for an introduction to Dirichlet
processes.

Since $D$ is a probability measure, it is possible to take samples
$\tilde\theta_i$, $i=1, \dots, N$ independently from it. If $D$ is marginalized
out, the distribution of the samples follows the exchangeable probability partition function (EPPF)~\citep{Pitman95_PTRF}
\begin{align}
\hspace{-.4cm}p(\tilde\theta_1, \dots, \tilde\theta_N) &\propto \alpha^{K-1}\prod_{k : n_k > 0} \Gamma(n_k) \quad \text{where}\quad n_k= \sum_{i=1}^N \ind{\tilde\theta_i=\theta_k} \quad K = \sum_k \ind{n_k > 0}.
\end{align}
Alternatively, we specify the conditional distributions of the Chinese restaurant process (CRP) that 
sequentially builds up a set of cluster labels 
$z_i\in\mathbb{N}$, $i=1, \dots, N$ via
\begin{align}
p(z_i=k | z_1, \dots, z_{i-1}) &\propto \left\{\begin{array}{ll}
			n_k &  k\leq K, \quad \text{where}\quad n_k = \sum_{j=1}^{i-1}\ind{z_j=k}\\
			 \alpha & k=K+1
				\end{array}\right., \label{eq:labelcrp}
\end{align}
where the first line denotes the probability of observation $i$ joining an already instantiated cluster, while the second
denotes the probability of observation $i$ creating a new cluster.

The Markov chain mixture model (MCMM) described in the following, originally developed by~\citet{Lin10_NIPS}\footnote{We have modified the original
model to account for later work, in which the link to the dependent Dirichlet process was shown to be false~\citep{Chen13_ICML}.}, is heavily based on
the DP and CRP but includes a Markov chain dependence structure between clusters at sequential time steps $t\in\mathbb{N}$. 
Define the current set of labels $\mathbf{z}_t =
\{z_{it}\}_{i=1}^{N_t}$, the current batch of data $\mathbf{y}_t = \{y_{it}\}_{i=1}^{N_t}$, and
the current set of cluster centers $\bm{\theta}_{t} =
\{\theta_{kt}\}_{k=1}^{K_t}$. Define $n_{kt}$ to be the number of datapoints assigned to cluster $k$ at timestep
$t$, and let $c_{kt} = \sum_{\tau=1}^{t-1}n_{k\tau}$.
Define $\dtk$ to be the number of timesteps since cluster $k$ last generated data, i.e.~$\dtk = t - \left(\max_{\tau\in\mathbb{N}} \tau \text{ s.t. } n_{k\tau}>0\right)$. 
Further, define the past data $\mathbf{y}_{<t} =
\{\mathbf{y}_\tau\}_{\tau=1}^{t-1}$, past labels $\mathbf{z}_{<t} =
\{\mathbf{z}_\tau\}_{\tau=1}^{t-1}$, and most recent knowledge of 
past cluster parameters $\bm{\theta}_{t-\dtk} =
\{\theta_{k(t-\dtk)}\}_{k=1}^{K_{t-1}}$.
Finally, let $\clusa_t$ be the set of active clusters at timestep $t$, i.e.~$\clusa_t = \{ k \in \mathbb{N}, 1\leq k \leq K_t : n_{kt} >
0\}$, and $\indices_{kt}$ be the set of data indices assigned to cluster $k$ at
timestep $t$, i.e.~$\indices_{kt}=\{ i \in \mathbb{N} : 1\leq i \leq N_t,
z_{it}=k\}$.  

The prior on label assignments is specified, similarly to the CRP, in terms of a sequence of conditional distributions as follows:
\begin{align}
p(z_{it}=k | \mathbf{z}_{<t}, \left\{z_{jt}\right\}_{j=1}^{i-1})&\propto \left\{\begin{array}{ll}
			\alpha & k=K_t+1\\
			c_{kt}+n_{kt} &  k \in \clusa_t\\
                        q^\dtk c_{kt} & k \notin \clusa_t
				\end{array}\right. . \label{eq:mcmmlabelprior}
\end{align}
The parameter $q\in[0, 1]$ controls the cluster death process in this model; in between each pair of timesteps $t\to t+1$, 
the survival of each cluster $k \leq K_t$ is determined by an independent Bernoulli random variable with success probability $q$. If the trial
is a success, the cluster survives; otherwise it is removed and no longer generates data for all $\tau > t$. 
The factor of $q^\dtk$ above reflects the fact that if $\dtk$ timesteps have elapsed since a cluster $k$ was observed, it must have survived
$\dtk$ Bernoulli trials of probability $q$ in order to generate data in the present timestep. If cluster $k$ is currently active in timestep $t$ (i.e.~$n_{kt}>0$),
then a new datapoint is assigned to it with probability proportional to the total amount of data assigned to it in the past; this is the typical
Bayesian nonparametric ``rich get richer'' behavior. Finally, a new cluster is created with probability proportional to $\alpha > 0$. Note that at
the first timestep $t=1$, the above label assignment model reduces to the CRP.

The parameters $\theta_{kt}$ have a random walk prior:
\begin{align}
\theta_{kt} | \bm{\theta}_{t-1}\overset{\text{indep}}{\sim} \left\{\begin{array}{ll}
                              H & k > K_{t-1}\\
                             T(\cdot | \theta_{k(t-1)}) & k \leq K_{t-1}
                               \end{array}\right. , \quad k=1, \dots, K_t .\label{eq:mcmmparamprior}
\end{align}
where $H$ is the prior distribution for generating a new cluster parameter the first time it is instantiated, while $T$ is the 
distribution of the random walk that each cluster parameter $k$ undergoes between each pair of timesteps $t\to t+1$.

Finally, each observed datapoint $y_{it}$, $i=1, \dots, N_t$ is sampled independently from the likelihood $F$ parameterized for its respective cluster, 
i.e.~$y_{it} \overset{\text{indep}}{\sim} F(\cdot | \theta_{z_{it}t})$.

\section{Asymptotic Analysis of the Markov Chain Gaussian Mixture Model}\label{sec:asymptotics}
\subsection{The MAP Filtering Problem}\label{subsec:mapfilter}
The MCMM is used as a prior on a temporally evolving Gaussian mixture model, where the points $\theta_{kt}$ are used as the means for the clusters.
The specific distributions used in this model are a Gaussian transition distribution $T\left(\cdot |
\theta'\right) = \mathcal{N}\left(\theta', \xi  I\right)$, a Gaussian parameter
prior $H(\cdot) = \mathcal{N}(\phi, \rho I)$, and a Gaussian likelihood $F(\cdot | \theta) = \mathcal{N}(\theta, \sigma I)$.
The MCMM with this collection of distributions is referred to as a Markov chain Gaussian mixture model (MCGMM).
In this work, we consider the \emph{maximum a posteriori filtering problem},
in which the goal is to sequentially cluster a stream of batches of data,
thereby tracking the positions of the evolving cluster centers given all
observations from past timesteps. 
Fixing learned labellings from past time steps, we wish to solve
\begin{align}
  \begin{aligned}
  \max_{\mathbf{z}_t, \bm{\theta}_t} \, \,&
  p\left(\mathbf{z}_t, \bm{\theta}_t | 
  \mathbf{y}_t, \mathbf{y}_{<t}, \mathbf{z}_{<t}\right)\\
  =\max_{\mathbf{z}_t, \bm{\theta}_t} \, \,&
  p\left(\mathbf{y}_t | \mathbf{z}_t, \bm{\theta}_t\right)p\left(\mathbf{z}_t | \mathbf{z}_{<t}\right)
  \int_{\bm{\theta}_{t-\dtk}} \hspace{-.7cm}p\left(\bm{\theta}_t |
  \bm{\theta}_{t-\dtk}\right) 
  p\left(\bm{\theta}_{t-\dtk} | \mathbf{z}_{<t},
  \mathbf{y}_{<t}\right)\\
  =\min_{\mathbf{z}_t, \bm{\theta}_t} \, \,&
  -2\sigma \log \left[p\left(\mathbf{y}_t | \mathbf{z}_t, \bm{\theta}_t\right) p\left(\mathbf{z}_t | \mathbf{z}_{<t}\right)
  \int_{\bm{\theta}_{t-\dtk}} \hspace{-.7cm}p\left(\bm{\theta}_t |
  \bm{\theta}_{t-\dtk}\right) 
  p\left(\bm{\theta}_{t-\dtk} | \mathbf{z}_{<t},
\mathbf{y}_{<t}\right)\right].
\end{aligned}\label{eq:mapoptimization}
\end{align}
The data negative log likelihood is
\begin{align}
  \begin{aligned}
    p\left(\mathbf{y}_t | \mathbf{z}_t, \bm{\theta}_t\right) &=
    \prod_{i=1}^{N_t} \left( 2\pi
    \sigma\right)^{-\frac{d}{2}}\exp\left(-\frac{\|y_{it}-\theta_{z_{it}t}\|^2}{2\sigma}\right)\\
  \therefore -2\sigma\log p\left(\mathbf{y}_t | \mathbf{z}_t,
  \bm{\theta}_t\right) &=\sigma N_t
  d\log\left(2 \pi
  \sigma\right) + \sum_{i=1}^{N_t} || y_{it} - \theta_{z_{it}t} ||^2 .
  \end{aligned}
\end{align}

The label prior negative log likelihood can be derived from 
equation (\ref{eq:mcmmlabelprior}),
\begin{align}
  \begin{aligned}
    p\left(\mathbf{z}_t | \mathbf{z}_{<t}\right) &\propto
    \prod_{k \in \clusa_t}
    \left(\alpha\Gamma\left(n_{kt}+1\right)\right)^{\ind{\dtk>0}}
    \left(q^\dtk\frac{\Gamma\left(c_{kt}+n_{kt}+1\right)}{\Gamma\left(c_{kt}\right)}\right)^{\ind{\dtk=0}}\\
   \therefore -2\sigma\log p\left(\mathbf{z}_t | \mathbf{z}_{<t}\right) &=
   2\sigma C - 2\sigma\sum_{k\in\clusa_t}
   \left\{\begin{array}{ll}
\scriptstyle 
\log \alpha +\log\Gamma\left(n_{kt}+1\right)  & \scriptstyle\dtk = 0\\
      \scriptstyle\dtk \log q + \log\Gamma\left(n_{kt}+c_{kt}+1\right)-\log\Gamma\left(c_{kt}\right) & \scriptstyle\dtk > 0
    \end{array}\right.  ,
  \end{aligned}
\end{align}
 where $C$ is the log normalization constant. 
Finally, suppose (this assumption will be justified in Section \ref{sec:gaussianrecursion}) that the cluster center tracking prior
$p\left(\bm{\theta}_{t-\dtk} | \mathbf{z}_{<t},
\mathbf{y}_{<t}\right)$ at time $t$ is a product of Gaussian distributions with means
$\phi_{kt}$ and corresponding covariances
$\rho_{kt}$. Then
the cluster center prior in the current timestep is
\begin{align}
  \begin{aligned}
  &\int_{\bm{\theta}_{t-\dtk}} \hspace{-.7cm}p\left(\bm{\theta}_t |
  \bm{\theta}_{t-\dtk}\right) 
  p\left(\bm{\theta}_{t-\dtk} | \mathbf{z}_{<t},
  \mathbf{y}_{<t}\right) \\
  =&\prod_{k=1}^{K_t}
  \left\{\begin{array}{ll}
\left(2\pi\rho\right)^{-\frac{d}{2}}\exp\left(-\frac{\|\theta_{kt}-\phi\|^2}{2\rho}\right)
  & \dtk = 0\\
  \left(2\pi(\xi\dtk+\rho_{kt})\right)^{-\frac{d}{2}}\exp\left(-\frac{\|\theta_{kt}-\phi_{kt}\|^2}{2(\xi\dtk+\rho_{kt})}\right)
  & \dtk > 0
\end{array}\right.\\
  \therefore -2\sigma\log &\int_{\bm{\theta}_{t-\dtk}} \hspace{-.7cm}p\left(\bm{\theta}_t |
  \bm{\theta}_{t-\dtk}\right) 
  p\left(\bm{\theta}_{t-\dtk} | \mathbf{z}_{<t},
  \mathbf{y}_{<t}\right) \\
=&\sum_{k=1}^{K_t}
  \left\{\begin{array}{ll}
    \sigma d\log\left(2\pi\rho\right)+\frac{\sigma\|\theta_{kt}-\phi\|^2}{\rho}
  & \dtk = 0\\
  \sigma
  d\log\left(2\pi(\xi\dtk+\rho_{kt})\right)+\frac{\sigma\|\theta_{kt}-\phi_{kt}\|^2}{\xi\dtk+\rho_{kt}}
  & \dtk > 0
\end{array}\right. .
  \end{aligned}
\end{align}

\subsection{Small-Variance Analysis}

We now have all the components of the MAP filtering problem cost function.
However, as it stands, the problem (\ref{eq:mapoptimization}) is a difficult combinatorial optimization 
due to the label prior. 
Thus, rather than solving it directly, we analyze its
small-variance limit. This will effectively remove the troublesome label prior cost components
and allow the derivation of fast K-Means-like coordinate descent updates.

Small-variance analysis in this model requires the scaling of the concentration $\alpha$,
the transition variance $\xi$, the subsampling probability $q$ and the 
old cluster center prior variances $\rho_{kt}$ with the
observation variance $\sigma$ in order to find
meaningful assignments of data to clusters. Thus, we set $\alpha$, $\xi$,
$q$, and $\rho_{kt}$
 to
\begin{align}
  \hspace{-.3cm}\begin{array}{c c c c}
    \alpha = (1+\rho/\sigma)^{d/2}\exp\left({-\frac{\lambda}{2\sigma}}\right),
&
  \xi = \tau \sigma,
  & 
  q =  \exp\left({-\frac{Q}{2 \sigma}}\right)
  & \text{and }
  \rho_{kt} = \frac{\sigma}{w_{kt}}.
\end{array} \label{eq:smallvariancescaling}
\end{align}

Finally, taking the limit $\sigma \to 0$ of (\ref{eq:mapoptimization}) yields 
the small-variance MAP filtering problem for the MCGMM,
\begin{align}
  \begin{aligned}
\min_{\mathbf{z}_t, \bm{\theta}_t} &
\sum_{k\in\clusa_t}\left(\overbrace{\lambda\ind{\dtk}}^{\text{New Cost}}+\hspace{-.3cm}\overbrace{Q\dtk}^{\text{Revival
Cost}}\hspace{-.3cm}+
  \overbrace{\gamma_{kt}||\theta_{kt} -
  \phi_{kt}||^2_2 +  \sum_{i\in \indices_{kt}} ||y_{it} -
\theta_{kt}||_2^2}^{\text{Weighted-Prior Sum-Squares
Cost}}\right)\label{eq:DMeansObj}\!.
\end{aligned}
\end{align}
The $\gamma_{kt}$ are used for notational brevity, and are defined as 
\begin{align}
  \gamma_{kt} = \left(w_{kt}^{-1} + \tau\dtk\right)^{-1}.
  \label{eq:gammaupdate}
\end{align}
The cost (\ref{eq:DMeansObj}) at timestep $t$ (henceforth denoted $J_t$) 
is a K-Means-like cost comprised of a number of components for each
active cluster $k \in \clusa_t$ (i.e.~those with data assigned to them): a penalty
for new clusters based on $\lambda$, a penalty for old clusters based
on $Q$ and $\dtk$, and finally a prior-weighted sum of squared
distance cost for all the observations in cluster $k$. 
The quantities $\gamma_{kt}$ and $\phi_{kt}$ transfer the knowledge about cluster $k$
from timestep $(t-\dtk)$ prior to incorporating any new data at timestep $t$.

\subsection{Recursive Update for Old Cluster Centers, Weights, and Times}\label{sec:gaussianrecursion}
Recall the assumption made in Section \ref{subsec:mapfilter} that $p\left(\bm{\theta}_{(t-\dtk)} |
\mathbf{z}_{<t}, \mathbf{y}_{<t}\right)$ is a product of Gaussian distributions
with means $\phi_{kt}$ and variances $\rho_{kt}=\frac{\sigma}{w_{kt}}$. We
now justify this assumption by showing that
if the assumption holds at timestep $t$, it will hold at timestep $t+1$,
and leads to a recursive update for $w_{kt}$ and $\phi_{kt}$.

Including the newly fixed labels from timestep $t$ yields
\begin{align}
  \begin{aligned}
  p\left(\bm{\theta}_t | \mathbf{z}_{<t}, \mathbf{z}_{t}, \mathbf{y}_{<t},
  \mathbf{y}_{t}\right) &\propto p\left(\mathbf{y}_t | \bm{\theta}_t,
  \mathbf{z}_t\right)\int_{\bm{\theta}_{t-\dtk}}p\left(\bm{\theta}_t |
  \bm{\theta}_{t-\dtk}\right) p\left(\bm{\theta}_{t-\dtk} | \mathbf{z}_{<t},
  \mathbf{y}_{<t}\right)\\
  &=
  \prod_{k=1}^{K_t}\mathcal{N}\left(\frac{\gamma_{kt}\phi_{kt}+\sum_{i\in
    \indices_{kt}}y_{it}}{\gamma_{kt}+n_{kt}},
  \frac{\sigma}{\gamma_{kt}+ n_{kt}}\right).
\end{aligned}\label{eq:recursivegaussian}
\end{align}
This shows that if the cluster center tracking prior is Gaussian with mean
$\phi_{kt}$ and variance $\rho_{kt}$ that scales with $\sigma$, 
then the new cluster center tracking prior after clustering the dataset at 
time $t$ is also Gaussian with a variance $\rho_{k(t+1)}$ that scales with
$\sigma$. Thus, the Gaussian assumption holds between timesteps, 
and the recursive update scheme for the old cluster centers
$\phi_{kt}$, the weights $w_{kt}$, and times 
since cluster $k$ was last observed $\dtk$ (applied after 
fixing the labels $\mathbf{z}_t$ at timestep $t$) is 
\begin{align}
  \begin{array}{rccc}
    \text{Condition} & \phi_{k(t+1)} & w_{k(t+1)} & \dtkn\\
    \midrule
    n_{kt} = 0 & \phi_{kt} & w_{kt} & \dtk+1\\
    n_{kt}>0 & \frac{\gamma_{kt}\phi_{kt}+\sum_{i\in\indices_{kt}}y_{it}}{\gamma_{kt}+n_{kt}} & \gamma_{kt}+n_{kt} & 1
  \end{array} \label{eq:wupdate}
\end{align}
where $n_{kt}$ is the number of datapoints assigned to cluster $k$ at 
timestep $t$, and $w_{kt} = \gamma_{kt} = 0$ for new clusters.

\section{Dynamic Means (D-Means)}\label{sec:dynmeans}
\subsection{Algorithm Description}
As shown in the previous section, the small-variance asymptotic limit of the
MAP filtering problem for the MCGMM is 
a hard clustering problem with a K-Means-like objective. Inspired by the
original K-Means algorithm~\citep{Lloyd82_IEEETIT}, we develop a coordinate descent algorithm
to approximately optimize this objective.

\paragraph{Label Update}
The minimum cost label assignment for datapoint $y_{it}$ is found by fixing all
other labels and all the cluster centers, and computing the cost for assignment
to each instantiated and uninstantiated cluster:
\begin{align}
  z_{it} = \argmin_{k}
          \left\{\begin{array}{l l}
          ||y_{it} - \theta_{kt}||^2 \quad &\text{if $\theta_k$ instantiated,
          i.e.~}k\in\clusa_t\\
          Q\dtk + \frac{\gamma_{kt}}{\gamma_{kt}+1}|| y_{it} - \phi_{kt}||^2 \quad &\text{if
            $\theta_k$ old, uninstantiated, i.e.~} k \notin \clusa_t\\
          \lambda \quad &\text{if $\theta_k$ new, i.e.~} k = K_t+1 \, .
        \end{array}\right.  \label{eq:labelupdate}
\end{align}
In this assignment step, $Q\dtk$ acts as a cost 
penalty for reviving old clusters that increases with the time
since the cluster was last seen, $\frac{\gamma_{kt}}{\gamma_{kt}+1} =
\frac{1}{w_{kt}^{-1}+\dtk \tau + 1}$ acts as a cost reduction to account for the possible motion of
clusters since they were last instantiated, and $\lambda$ acts as a cost penalty
for introducing a new cluster.

\paragraph{Parameter Update}
The minimum cost value for cluster center $k$ is found by fixing all other
cluster centers and all the data labels, taking the derivative and setting to
zero:
\begin{align}
  \theta_{kt} &= \frac{\gamma_{kt}\phi_{kt} + \sum_{i\in \indices_{kt}}y_{it}}
    {\gamma_{kt} + n_{kt}}.\label{eq:paramupdate}
\end{align}
In other words, the revived $\theta_{kt}$ is a weighted 
average of estimates using current
timestep data and previous timestep data. $\tau$ controls how much the current 
data is favored---as $\tau$ increases, the weight on current data increases,
which is explained by the fact
that uncertainty in where the old cluster center transitioned to increases with
$\tau$. It is also noted that if $\tau = 0$, this reduces to a simple weighted
average using the amount of data collected as weights, and if $\tau \to \infty$, 
this reduces to just taking the average of the current batch of data.

An interesting interpretation of this update, when viewed in combination
with the recursive update scheme for $w_{kt}$, is that it behaves like 
a standard Kalman filter in which $w_{kt}^{-1}$ serves as the 
current estimate variance, $\tau$ serves as the 
process noise variance, and $n_{kt}$
serves as the inverse of the measurement variance.

\paragraph{Coordinate Descent Algorithm}
Combining these two updates in an iterative scheme (shown in Algorithm
\ref{alg:dynmeans}) results in
Dynamic Means\footnote{Code available online:
  \url{http://github.com/trevorcampbell/dynamic-means}} (D-Means), an algorithm that
clusters sequential batches of observations, carrying
information about the means and the confidence in those means
forward from time step to time step. 
D-Means is guaranteed to converge to a local optimum in the objective
(\ref{eq:DMeansObj}) via typical monotonicity
arguments~\citep{Lloyd82_IEEETIT,Kulis12_ICML}. Applying Algorithm \ref{alg:dynmeans} to a sequence of batches of data 
yields a clustering procedure that retains much of the flexibility of the
MCGMM, in that it is able to track a set of
dynamically evolving clusters, and allows new clusters to emerge and old clusters to be forgotten. 
While this is its primary application, the 
sequence of batches need not be a temporal sequence. For example, Algorithm
\ref{alg:dynmeans} may be used as an any-time clustering algorithm for
large datasets, where $Q$ and $\tau$ are both set to 0, and the sequence of batches is generated by selecting
random disjoint subsets of the full dataset. In this application, the setting of
$Q$ and $\tau$ to 0 causes the optimal cluster means for the sequence of data
batches to converge to the optimal cluster means 
as if the full dataset were clustered in one large batch using
DP-means~\citep{Kulis12_ICML} (assuming all assignments in each timestep are
correct).
A final note to make is that D-Means is equivalent to the small-variance Gibbs sampling
algorithm for the MCGMM; interested readers may consult Appendix \ref{app:gibbsderiv} for a
brief discussion.
\begin{algorithm}[t!]
      \captionsetup{font=small}
      \caption{Dynamic Means}\label{alg:dynmeans}
      \small
  \begin{algorithmic}
    \Require $\left\{\{y_{it}\}_{t=1}^{N_t}\right\}_{t=1}^{t_{f}}$, $Q$, $\lambda$, $\tau$
    \State $K_0 \gets 0$
    \For{$t = 1 \to t_f$}
      \State $\{\gamma_{kt}\}_{k=1}^{K_{t-1}} \gets $ Eq.~(\ref{eq:gammaupdate})
      \State $K_t \gets K_{t-1}$
      \State $J^{\text{prev}}_t \gets \infty$, $J_t \gets \infty$
      \Repeat
      \State $J^{\text{prev}}_t \gets J_t$
      \State $\{z_{it}\}_{i=1}^{N_t} \gets$ Eq.~(\ref{eq:labelupdate})
      \State For each new cluster $k$ that was created, set $\gamma_{kt} \gets 0, \dtk \gets 0$ $K_t\gets K_t +1$
      \State For each new cluster $k$ that was destroyed, $K_t \gets K_t - 1$
      \State $\{\theta_{kt}\}_{k=1}^{K_t} \gets$ Eq.~(\ref{eq:paramupdate})
      \State $J_t \gets$ Eq.~(\ref{eq:DMeansObj})
      \Until{$J_t = J^{\text{prev}}_t$}
      \State $\{\phi_{k(t+1)}, w_{k(t+1)}, \dtkn\}_{k=1}^{K_t} \gets$ Eq.~(\ref{eq:wupdate})
    \EndFor
    \State \Return {$\{ \{z_{it}\}_{i=1}^{N_t}\}_{t=1}^{t_f},   
       \{\{\theta_{kt}\}_{k=1}^{K_t} \}_{t=1}^{t_f}$}
  \end{algorithmic}
\end{algorithm}

  \subsection{Practical Aspects and Implementation}
\subsubsection{Label Assignment Order}
One important caveat to note about D-Means is that, while it is a deterministic algorithm, its performance
depends on the order in which equation (\ref{eq:labelupdate}) is used
to assign labels. In practice this does not have a large effect, since D-Means creates clusters as needed
on the first round of label assignment, and so typically starts with a good
solution. However, if required, multiple random restarts of the algorithm with
different assignment orders may be used to mitigate the dependence. If this
is implemented, the lowest cost clustering out of all the random restarts
at each time step should be used to proceed to the next time step.

\subsubsection{Reparameterizing the Algorithm}
In order to use the Dynamic Means 
algorithm, there are three free parameters to
select: $\lambda$, $Q$, and $\tau$. While $\lambda$ represents
how far an observation can be from a cluster before it
is placed in a new cluster, and thus can be tuned 
intuitively, $Q$ and $\tau$ are not so straightforward. The 
parameter $Q$ effectively adds artificial distance between an observation and
an old cluster center, diminishing the ability for that cluster to explain the
observation and thereby capturing cluster death. The parameter $\tau$
artificially reduces the distance between an
observation and an old cluster center, enhancing the ability for that cluster to
explain the observation and thus accounting for cluster motion. How these two quantities affect
the algorithm, and how they interact with the setting of $\lambda$,
is hard to judge.

Instead of picking $Q$ and $\tau$ directly, the algorithm may be
reparameterized by picking $ T_Q, k_\tau \in \mathbb{R}_+$, $T_Q > 1$,
$k_\tau \geq 1$, and given
a choice of $\lambda$, setting
\begin{align}
\begin{aligned}
  Q =& \lambda/T_Q \quad
  \tau = \frac{T_Q(k_\tau-1) + 1}{T_Q-1}.
\end{aligned}
\end{align}
If $Q$ and $\tau$ are set in this manner, $T_Q$ represents the number
(possibly fractional) of
time steps a cluster can be unobserved before the label update (\ref{eq:labelupdate}) will never revive
that cluster, and $k_\tau \lambda$ 
represents the maximum squared distance away from a cluster 
center such that after a single time step, the
label update (\ref{eq:labelupdate}) will
revive that cluster. As $T_Q$ and $k_\tau$ 
are specified in terms of concrete algorithmic behavior, they are 
intuitively easier to set than $Q$ and $\tau$.

\subsubsection{Tuning Parameters}\label{subsubsec:tuning}
While D-Means has three tuning parameters ($\lambda$, $T_Q$ and $k_\tau$), in practice the algorithm is fast 
enough that a grid search is often sufficient. However,
in cases where this isn't possible, it is best to first tune $\lambda$ by running
D-Means on individual batches without considering any temporal effects
(i.e.~setting $Q = 0$ and $\tau = \infty$, since $\lambda$ only controls the
size of clusters), and then tune $T_Q$ and $k_\tau$ together afterwards with
fixed $\lambda$. Note that the size of clusters guides the choice 
of $\lambda$, the size of their transition steps
guides the choice of $k_\tau$, and
how frequently they are destroyed guides the choice of $T_Q$.

\subsubsection{Deletion of Clusters}\label{subsubsec:deletion}
When the revival penalty $Q\dtk$ exceeds $\lambda$ for an old cluster,
it will never be reinstantiated by the algorithm because it will always be 
less costly to create a new cluster. Therefore, an optimized implementation of
D-Means can safely remove such clusters from consideration to save memory 
and computation time.

\section{Kernelization \& Relaxation of the Small-Variance Filtering
Problem}\label{sec:krnspec}
D-Means tends to perform best when the clusters are roughly 
spherical and linearly separable. While it is possible to redefine the cost 
equation (\ref{eq:DMeansObj}) in terms of more general distance metrics or
divergences~\citep{Banerjee05_JMLR}, clusters still must be spherical and
linearly separable using whichever distance metric or divergence is
chosen. When this assumption is violated, D-Means tends to create too many or
too few clusters to properly explain the data.

This section addresses this limitation in D-Means through
\emph{kernelization}~\citep{Dhillon07_TPAMI,Hofmann08_ANSTATS} of
the small-variance MAP filtering problem (\ref{eq:DMeansObj}).
Rather than clustering data vectors themselves at each time step,
a nonlinear embedding map is first applied to the vectors, and the embedded data 
is clustered instead. Crucially, the embedding map need not be known explicitly;
it is specified implicitly through a kernel function. Given an appropriate
kernel function, nonlinearly separable and nonspherical clusters in the original
data vector space become spherical and linearly separable in the embedding
space, thereby allowing the use of D-Means in this space. 

Iterative clustering methods for kernelized costs are particularly 
susceptible to getting trapped in poor local cost minima~\citep{Dhillon07_TPAMI}. 
Thus, this section additionally discusses a 
\emph{spectral relaxation} of the clustering 
problem~\citep{Zha01_NIPS,Kulis12_ICML}, where some
of the clustering constraints are removed, yielding
a relaxed problem that may be solved globally via 
eigendecomposition. The relaxed solution is then
refined to produce a feasible clustering that is typically
less susceptible to local clustering minima. These two steps together
produce SD-Means, a spectral clustering algorithm for temporally evolving batch-sequential
data.

\subsection{Kernelization of the Cost}
Instead of clustering the data vectors directly, suppose 
we first map the vectors via the nonlinear embedding $\omega : \mathbb{R}^d
\to \mathbb{R}^n$. Then the least squares minimizer for the 
cluster centers in the new space is
\begin{align}
  \theta_{kt} &= \frac{\gamma_{kt}\phi_{kt} +
  \sum_{i\in\indices_{kt}}\omega(y_{it})}{\gamma_{kt}+n_{kt}} \in \mathbb{R}^n.\label{eq:nonlineartheta}
\end{align}
Substituting (\ref{eq:nonlineartheta}) and replacing $y_{it}$ with its nonlinear
embedding $\omega(y_{it})$ in (\ref{eq:DMeansObj}) yields the following 
clustering problem:
\begin{align}
\begin{aligned}
  \min_{\mathbf{z}_t} \sum_{i=1}^{N_t} \krnmat{Y}{Y}_{ii} + \sum_{k\in\clusa_t}
  \!\lambda\ind{\dtk}+Q\dtk+
    \frac{\gamma_{kt}n_{kt}\krnmat{\Phi}{\Phi}_{kk}}{\gamma_{kt}+n_{kt}} 
    -\!\sum_{i\in\indices_{kt}}\!\frac{2\gamma_{kt}\krnmat{Y}{\Phi}_{ik}+\sum_{j\in\indices_{kt}}\krnmat{Y}{Y}_{ij}}{\gamma_{kt}+n_{kt}}
  \end{aligned}\label{eq:krnmin}
\end{align}
where $\krnmat{Y}{Y}_{ij} = \omega(y_{it})^T\omega(y_{jt})$,
$\krnmat{Y}{\Phi}_{ik} = \omega(y_{it})^T\phi_{kt}$, and
$\krnmat{\Phi}{\Phi}_{kk} = \phi_{kt}^T\phi_{kt}$. Since $\phi_{kt}$ is not
known (as it lies in the embedding space), computing $\krnmat{Y}{\Phi}$
and $\krnmat{\Phi}{\Phi}$ requires that $\phi_{kt}$ is expressed
in terms of past data assigned to cluster $k$, 
\begin{align}
  \phi_{kt} &=
  \sum_{\tau\in\mathcal{T}_k}\left(\frac{1}{\gamma_{k\tau}}\prod_{s\in\mathcal{T}_k
  :
s \geq \tau}\frac{\gamma_{ks}}{\gamma_{ks}+n_{ks}}\right)
  \sum_{i\in\indices_{k\tau}}\omega(y_{i\tau}),
  \label{eq:phiexpansion}
\end{align}
where $\mathcal{T}_k$ is the set of all past timesteps where cluster $k$ was
active. Note
that the entire cost may now be specified in terms of dot products of data vectors in the embedding
space. Therefore, only the \emph{kernel function} $\krn{\cdot}{\cdot}
\!:\!\mathbb{R}^d\times \mathbb{R}^d
\to \mathbb{R}$ needs to be known, with the nonlinear embedding $\omega$ left
implicit.

\subsection{Recursive Approximation of Old Cluster Centers}
An issue with the above procedure is that computing $\krnmat{Y}{\Phi}$ and
$\krnmat{\Phi}{\Phi}$ via (\ref{eq:phiexpansion}) requires storing and computing
the kernel function with all past data. Over time, as the amount of 
observed data increases, this becomes intractable. Therefore, we instead
recursively approximate $\phi_{kt}$ with a linear combination of
a sparse subset of past data assigned to cluster $k$. Suppose
there is an approximation budget of $m$ past data vectors, and $\phi_{kt}$ 
is approximated as
\begin{align}
  \phi_{kt} &\simeq \sum_{j=1}^m a_j \omega(u_j),
\end{align}
where $a_j \in \mathbb{R}$, and $u_j \in \mathbb{R}^d$ (subscripts
$k$ and $t$ are suppressed for brevity). Then after
clustering the data $\{\omega(y_{it})\}_{i=1}^{N_t}$ at time $t$,
if $k\in\clusa_t$ the old center $\phi_{k(t+1)}$ is updated via
\begin{align}
  \phi_{k(t+1)} \simeq
  \frac{\gamma_{kt}\sum_{j=1}^m a_j
  \omega(u_j)+\sum_{i\in\indices_{kt}}\omega(y_{it})}{\gamma_{kt}+n_{kt}} =
  \sum_{j=1}^{m+n_{kt}} \bar{a}_j\omega(v_j)
  \label{eq:phiinitapprox}
\end{align}
where $\{u_j\}_{j=1}^m\bigcup\{y_{it}\}_{i\in\indices_{kt}}$ have been renamed
$\{v_j\}_{j=1}^{m+n_{kt}}$, $\bar{a}_j=
\frac{\gamma_{kt}}{\gamma_{kt}+n_{kt}} a_j \, \forall 1\leq j \leq m$, 
and $\bar{a}_{j} = \frac{1}{\gamma_{kt}+n_{kt}} \, \forall m+1\leq j \leq m+n_{kt}$.
As (\ref{eq:phiinitapprox}) requires storing $m+n_{kt}$ vectors, $\phi_{k(t+1)}$ must
be reapproximated with a subset containing $m$ vectors. This is accomplished
by solving a sparse regressor selection problem, 
\begin{align}
  \begin{aligned}
a^\star &\gets 
\argmin_{x\in\mathbb{R}^{m+n_{kt}}, \, \mathrm{card}(x) \leq m} 
  \left(\bar{a}-x\right)^TW\left(\bar{a}-x\right)\\
  \phi_{k(t+1)} &\gets\hspace{-.3cm} \sum_{j : |a^\star_j|> 0}\hspace{-.3cm}
a^\star_j\omega(v_j) \, \, .
  \end{aligned}
\label{eq:krnphiupdate}
\end{align}
where $W \in \mathbb{R}^{(m+n_{kt})\times (m+n_{kt})}$, $W_{ij} = \omega(v_i)^T\omega(v_j) = \krn{v_i}{v_j}$.
This optimization can be solved approximately using a greedy approach
(iteratively selecting the regressor that provides the largest reduction in
cost) or using $l_1$-regularized quadratic
programming~\citep{Tibshirani96_JRSTATS}.
The successive application of the approximation (\ref{eq:krnphiupdate})
is stable in the sense that relative error between the true old center
and its approximation does not grow without bound over time, as
shown by Theorem \ref{thm:phierrorbound}.
\begin{theorem}\label{thm:phierrorbound}
  Let $\phi^\star_{kt}$ be the true old
  center $k$ at time $t$ given by (\ref{eq:phiexpansion}),
  and $\phi_{kt}$ be the approximation given by 
  (\ref{eq:krnphiupdate}). If
  after every time step, (\ref{eq:krnphiupdate})
  is solved with objective value less than $\epsilon^2$,
  i.e.~$(\bar{a}-a^\star)^TW(\bar{a}-a^\star) <\epsilon^2$, 
  then $\|\phi_{kt}-\phi^\star_{kt}\|_2 < \epsilon
  \left(1+\frac{1}{\tau}\right)\, \, \forall t \in \mathbb{N}$.
\end{theorem}
\begin{proof}
  See Appendix \ref{app:phiapproxproof}.
\end{proof}

\subsection{Spectral Relaxation}
In order to solve (\ref{eq:krnmin}), we develop a spectral clustering method,
which requires the objective to be expressed
as a constrained matrix trace minimization~\citep{Zha01_NIPS,Kulis12_ICML}. 
However, it is not possible to include the old cluster penalties
$\sum_{k\in\clusa_t}Q\dtk$ in such a formulation. Thus, the penalty 
is modified as follows:
\begin{align}
  \sum_{k\in\clusa_t} Q\dtk \simeq \sum_{k\in\clusa_t}
  \frac{n_{kt}}{\gamma_{kt}+n_{kt}}Q\dtk.
\end{align}
This modification preserves the value of the original penalty 
when $n_{kt} = 0$ (the cluster is uninstantiated) or when $\dtk = 0$ (the cluster is
new), as it incurs a penalty of $0$ in these cases. 
Otherwise, if $\gamma_{kt}$ is large, the cluster was very popular in past timesteps;
therefore, instantiating the cluster has a low cost. Likewise, if
$\gamma_{kt}$ is small, the cluster was not popular, and
$\frac{n_{kt}}{\gamma_{kt}+n_{kt}} \approx 1$, so it is more costly to instantiate
it. Thus, the modified penalty makes sense in that ``the rich get richer'',
similar to the behavior of the Dirichlet process. Further, this modification
strictly reduces the cost of any clustering, and thus preserves the lower
bounding property of the exact spectral relaxation.

Given this modification, it is possible to rewrite the kernelized clustering
problem (\ref{eq:krnmin}) as a trace minimization,
\begin{align}
  \begin{aligned}
    \min_{Z} \tr{\krnmat{Y}{Y}} +
    \tr{\Gamma^{\frac{T}{2}} \krnmat{\Phi}{\Phi}\Gamma^{\frac{1}{2}}} + \tr{\Omega -
    \lambda I}-\tr{Z^T\hat{\Gamma}^{\frac{T}{2}}(G-\lambda
    I)\hat{\Gamma}^{\frac{1}{2}}Z}, 
  \end{aligned}\label{eq:spectralmin}
\end{align}
where $Z$, $G$ and $\hat{\Gamma}$ are defined as
\begin{align}
\label{eq:ZGdefine}
\hspace{-.3cm}Z = \hat{\Gamma}^{\frac{1}{2}}\left[\begin{array}{ccc}
    \! \! c_1 &\! \!\!\dots \!& \!\!c_{K_t} \!\!
\end{array}\right] \Xi ,
\,  \,
G =\hat{\Gamma}^{\frac{T}{2}}\left[\begin{array}{cc}
    \krnmat{Y}{Y}& \krnmat{Y}{\Phi}\\
    \left(\krnmat{Y}{\Phi}\right)^T& 
  \mathrm{diag}(\krnmat{\Phi}{\Phi}) + \Gamma^{-1}\Omega
  \end{array}\right]\hat{\Gamma}^{\frac{1}{2}}, \, \,
\hat{\Gamma} = \left[\begin{array}{cc}
    I & 0 \\
    0 & \Gamma
\end{array}\right].
\end{align}
Above, $c_k\in\{0, 1\}^{N_t+K_{t-1}}$ is an indicator vector
for cluster $k$ ($c_{ki} = 1, \, i \leq N_t$ indicates that datapoint $i$
is assigned to cluster $k$, and $c_{kk'} = 1, \, k' > N_t$ 
indicates that cluster $k$ is linked to old cluster $k'$, where at most one component
$c_{kk'}, \, k'> N_t$ can be set to 1),
the matrix $\Gamma\in\mathbb{R}^{{K_{t-1}}\times {K_{t-1}}}$ is a square diagonal matrix with the
$\gamma_{kt}$ along the diagonal for each old cluster $k \in \{1,
\dots,K_{t-1}\}$, and $\Xi\in\mathbb{R}^{K_t\times K_t}$ is a square matrix with
the value $\left(\gamma_{kt}+n_{kt}\right)^{-\frac{1}{2}}$ along the diagonal
for each cluster $k\in \{1, \dots, K_t\}$.

If the problem is relaxed by removing all constraints on $Z$
except for $Z^TZ = I$, (\ref{eq:spectralmin})
is equivalent to solving
\begin{align}
  &\begin{aligned}\label{eq:relaxedopt}
    Z^\star = &\argmax_{Z} \tr{Z^T(G-\lambda
    I)Z}\\
    &\quad \mathrm{s.t. }\, \, Z^TZ = I,
  \end{aligned}
\end{align}
for which the set of global optimum solutions is~\citep{Yu03_ICCV}
\begin{align}
  Z^\star &\in \left\{V^\star U : U\in\mathbb{R}^{\left|\clusa_t\right|\times \left|\clusa_t\right|}, \, U^T = U^{-1}\right\}\\
  V^\star &=\left[v_1, \dots, v_{\left|\clusa_t\right|}\right]\in
  \mathbb{R}^{(N_t+{K_{t-1}})\times {\left|\clusa_t\right|}}, \label{eq:UVdefine}
\end{align}
where $\{v_1, \dots, v_{\left|\clusa_t\right|}\}$ are unit eigenvectors of $G$ whose corresponding eigenvalues
are greater than $\lambda$.
Note that $\left|\clusa_t\right| \neq K_t$ in general; $K_t$ accounts for clusters
that were created in a previous time step but are not active in the current time
step. Therefore, $\left|\clusa_t\right| \leq K_t$, where equality occurs when all old clusters $k$
satisfy $n_{kt} > 0$.


\subsection{Finding a Feasible Solution}\label{subsec:specfeasible}
Given the set of global minimum solutions of the relaxed problem $\{V^\star U :
  U^T = U^{-1}\}$, the next steps are
to find the partitioning of data into clusters, and to find the correspondences
between the set of clusters and the set of old clusters.  
The partitioning of data into clusters is found
by by minimizing the Frobenius norm between a binary cluster indicator matrix $X$ and 
$\bar{V}^\star U$, where $\bar{V}^\star \in \mathbb{R}^{N_t \times
\left|\clusa_t\right|}$ is the row-normalization of $V^\star$ with the
last $K_{t-1}$ rows removed:
\begin{align}
  \begin{aligned}\label{eq:xuopt}
  X^\star, U^\star = \argmin_{X, U} \quad &  \|X - \bar{V}^\star U\|^2_F\\
  \mathrm{s.t.}\quad & U^T = U^{-1}, \, \, \sum_j X_{ij} = 1, \, \,  1 \leq i \leq N_t\\
              & U \in \mathbb{R}^{\left|\clusa_t\right| \times
              \left|\clusa_t\right|}, \, \, X \in \{0,
              1\}^{N_t\times\left|\clusa_t\right|}.
  \end{aligned}
\end{align}
This optimization is difficult to solve exactly, but the following approximate 
coordinate descent method works well in practice~\citep{Yu03_ICCV}: 
\begin{enumerate}
  \item{Initialize $U$ via the method described in~\citep{Yu03_ICCV}.}
  \item{Set $X = 0$. Then for all $1 \leq i \leq N_t$, set row $i$ of $X$ to an indicator for the
      maximum element of row $i$ of $\bar{V}^\star U$,
      \begin{align}\label{eq:XUpdate}
          j^\star &= \argmax_j \left(\bar{V}^\star U\right)_{ij}, \quad X_{ij^\star} \gets 1.
      \end{align} }
\item{Set $U$ using the singular value decomposition of
    $X^T\bar{V}^\star$,
        \begin{align}\label{eq:UUpdate}
            X^T\bar{V}^\star &\overset{\text{svd}}{=} R\Sigma W^T , \quad U \gets WR^T.
        \end{align}}
      \item {Compute $\|X-\bar{V}^\star U\|^2_F$. If it decreased from the previous
      iteration, return to 2.}
  \end{enumerate}
Given the solution to this problem, the current set of clusters are known -- 
define the temporary indices set $\indices_l = \{i : X_{il} = 1\}$, and the
temporary cluster count $n_l = \left|\indices_l\right|$. The 
final step is to link the clusters $1, \dots, \left|\clusa_{t}\right|$ to
old clusters by solving the following linear program:
\allowdisplaybreaks
\begin{align}
  \begin{aligned}
    c^\star = \argmin_{c} \quad & \sum_{l=1}^{\left|\clusa_t\right|}
    \sum_{k=1}^{K_{t-1}} c_{lk} \left(Q\dtk-\lambda 
    + \frac{\gamma_{kt}\zeta_{lk}}{\gamma_{kt}+n_l}\right)\\
  \mathrm{s.t.}\quad      &\sum_{k=1}^{K_{t-1}} c_{lk} \leq 1, \, \, 1 \leq j \leq
     \left|\clusa_t\right|\\
     &\sum_{l=1}^{\left|\clusa_t\right|} c_{lk} \leq 1, \, \, 1 \leq k \leq
     K_{t-1}\\
     &c \geq 0, \, \,  c \in \mathbb{R}^{\left|\clusa_{t}\right|\times K_{t-1}}.
   \end{aligned}\label{eq:clustermatching}
%
\end{align}
where $\zeta_{lk} = n_{l}\krnmat{\Phi}{\Phi}_{kk}-\sum_{i\in\indices_l}\left(2\krnmat{Y}{\Phi}_{ik}-\frac{1}{n_l}\sum_{j\in\indices_l}\krnmat{Y}{Y}_{ij}\right)$.
Theorem \ref{thm:clustermatching} shows that the final clustering of
the data at the current timestep $t$ can be constructed given the optimal solution $c^\star$ of the
linear program (\ref{eq:clustermatching}):
\begin{lemma}\label{lem:totallyunimodular}
  The linear optimization (\ref{eq:clustermatching}) has a totally unimodular constraint matrix.
\end{lemma}
\begin{theorem}\label{thm:clustermatching}
  The solution to the linear program (\ref{eq:clustermatching}) is the minimum
  cost matching, with respect to (\ref{eq:krnmin}), of old clusters to current
  clusters, where $c_{lk} = 1$ implies temporary cluster $l$ reinstantiates old
  cluster $k$, and
  $\sum_{k=1}^{K_{t-1}}c_{lk} = 0$ implies temporary cluster $l$ is a new
  cluster.
\end{theorem}
\begin{proof}
  See Appendix \ref{app:clustermatchingproof}.
\end{proof}

  \section{Spectral Dynamic Means (SD-Means)}\label{sec:spectraldynmeans}
\subsection{Algorithm Description}
\begin{algorithm}[t!]
  \captionsetup{font=small}
  \caption{\textsc{Spectral Dynamic Means}}\label{alg:spectraldynmeans}
  \small
\begin{algorithmic}
  \Require $\{ \{y_{it}\}_{i=1}^{N_t} \}_{t=1}^{t_f}$, $Q$, $\lambda$, $\tau$
  \State $K_0 \gets 0$
  \For{$t=1\to t_f$}
    \State $\{\gamma_{kt}\}_{k=1}^{K_{t-1}} \gets $ Eq. (\ref{eq:gammaupdate})
    \State $G \gets$ Eq. (\ref{eq:ZGdefine})
    \State $\{(v_i, \sigma_i)\} \gets $ Unit eigenvectors \& eigenvalues of $G$
    \State $V^\star \gets \left[v_1, \dots,
    v_{\left|\clusa_t\right|}\right]$ where $\sigma_i > \lambda, 1\leq i \leq \left|\clusa_t\right|$
    \State $\bar{V}^\star \gets $ Row normalization of $V^\star$ with last
    $K_{t-1}$ rows removed
    \State $U \gets $ Most orthogonal rows of $\bar{V}^\star$ (See~\citep{Yu03_ICCV} for details)
    \State $L_t \gets \infty$, $L^\text{prev}_t \gets \infty$
    \Repeat 
    \State $X \gets $ Eq. (\ref{eq:XUpdate})
    \State $U \gets $ Eq. (\ref{eq:UUpdate})
    \State $L_t \gets \|X-\bar{V}^\star U\|^2_F$
    \Until {$L_t = L^{\text{prev}}_t$}
    \State $\{\indices_l\}_{l=1}^{\left|\clusa_t\right|} \gets$ Extracted
    data partitioning from $X$ (See Section
    \ref{subsec:specfeasible} for details)
    \State $\{\indices_{kt}\}_{k=1}^{K_t} \gets$ Optimal cluster correspondence
    from (\ref{eq:clustermatching}) (See Section
    \ref{subsec:specfeasible} for details)
    \State For each new cluster $k$ that was created, set $\gamma_{kt}\gets 0$, $\dtk \gets 0$
    \State $\{\phi_{k(t+1)}\}_{k=1}^{K_t} \gets$ Eq.~(\ref{eq:krnphiupdate}),
    $\{w_{k(t+1)}, \dtkn\}_{k=1}^{K_t} \gets$ Eq.~(\ref{eq:wupdate})
  \EndFor
  \State \Return $\{ \{z_{it}\}_{i=1}^{N_t} \}_{t=1}^{t_f}$
\end{algorithmic}
\end{algorithm}
The overall Spectral Dynamic Means\footnote{Code available online:
  \url{http://github.com/trevorcampbell/dynamic-means}} (SD-Means) algorithm is shown in Algorithm
\ref{alg:spectraldynmeans}. At each time step, the similarity matrix $G$ in
(\ref{eq:ZGdefine})
is constructed using the current data and old cluster information.
Next, the eigendecomposition of $G$ is computed, and $\bar{V}^\star$ is created
by horizontally concatenating all unit eigenvectors whose eigenvalues exceed
$\lambda$, removing the last $K_{t-1}$ rows,  
and row-normalizing as outlined in Section \ref{subsec:specfeasible}.
Note that if the set $\{v_i : \sigma_i > \lambda\}$ is empty, one can set $V^\star$
to be the single eigenvector $ V^\star = v_{i^\star} =
v_{\argmax_i \sigma_i}$.
Then the steps (\ref{eq:XUpdate}) and (\ref{eq:UUpdate}) are iterated until 
a feasible partitioning of the data is found. Finally,
(\ref{eq:clustermatching}) is solved to link clusters to their counterparts in
past timesteps, and the old cluster information
$\phi_{kt}$, $w_{kt}$ and $\dtk$ for each $k\in\clusa_t$ is updated 
using the same procedure as in KD-Means.

\subsection{Practical Aspects and Implementation}
\subsubsection{Complexity vs.~Batch Size}
With traditional batch spectral clustering, given $N$ datapoints, the
dominating aspects of the computational complexity are forming a similarity matrix of
$\frac{N^2+N}{2}$ entries, and finding its eigendecomposition (generally an $O(N^3)$ computation). 
SD-Means processes smaller batches of size $M \ll N$ at a time, thereby reducing this
complexity. However, it must repeat these computations $\frac{N}{M}$ times to process
the same $N$ datapoints, and it also incorporates rows and columns for the $K$ old clusters in the
similarity matrices. Thus, the overall computational cost of SD-Means is
$\frac{NM+(2K+1)N}{2}$ similarity computations to form the matrices, and
$O(\frac{N}{M}(M+K)^3)$ computations for the eigendecompositions. 
Therefore, if $M \ll N$ and $K \ll N$, this represents a significant reduction in computational
time with respect to batch spectral clustering. Note that deletion of old
clusters is very important for practical implementations of SD-Means 
(discussed in Section \ref{subsubsec:deletion}), as the computational complexity
of processing each batch scales with $K^3$.

\subsubsection{Initialization of $U$}
While SD-Means is a deterministic algorithm, its performance
depends on the initialization of the orthogonal matrix $U$ defined in
(\ref{eq:UVdefine}). Like D-Means,
in practice this does not have a large effect when using
the initialization from~\citep{Yu03_ICCV}. 
However, if required, random restarts of the algorithm with
different initializations of $U$ may be used to 
mitigate the dependence. If this
is implemented, the clustering that provides the 
lowest final cost in (\ref{eq:krnmin})
out of all the random restarts should be used 
to proceed to the next time step.


\section{Related Work}
\paragraph {Small-variance Asymptotics}
The small-variance asymptotic analysis methodology was first applied to Bayesian
nonparametrics recently~\citep{Kulis12_ICML}, and as such the field is still quite
active. This first work developed asymptotic algorithms for the DP and hierarchical DP
mixture by considering the asymptotics of the Gibbs sampling algorithms for
those models, yielding K-Means-like clustering algorithms that infer the number of
clusters using a cost penalty similar to the AIC penalty~\citep{Akaike74_IEEETAC}.
This work has been extended to handle general exponential family
likelihoods~\citep{Jiang12_NIPS}, and it has been shown that the asymptotic
limit of MAP estimation (rather than the Gibbs sampling algorithm) 
yields similar results~\citep{Broderick13_ICML}. Most recently, this technique
has been applied to learning HMMs with an unknown number of
states~\citep{Roychowdhury13_NIPS}. D-Means and SD-Means, in contrast to these
developments, are the first results of small-variance asymptotics to a model
with dynamically evolving parameters.

\paragraph{Hard Clustering with an Unknown Number of Clusters}
Prior K-Means clustering algorithms that determine the number of clusters present
in the data have primarily involved a method for
iteratively modifying k using various statistical
criteria~\citep{Ishioka00_IDEAL, Pelleg00_ICML, Tibshirani01_JRSTATS}. In
contrast, D-Means and SD-Means derive this capability from a Bayesian
nonparametric model, similarly to the DP-Means algorithm~\citep{Kulis12_ICML}. 
In this sense, the relationship between the (Spectral) Dynamic 
Means algorithm and the dependent Dirichlet
process~\citep{Lin10_NIPS} is exactly that between the (Spectral) DP-Means
algorithm~\citep{Kulis12_ICML} and Dirichlet process~\citep{Ferguson73_ANSTATS}.

\paragraph{Evolutionary Clustering}
D-Means and SD-Means share a strong connection with evolutionary clustering algorithms.
Evolutionary clustering is a paradigm in which the cost function is comprised of
two weighted components: a cost for clustering the present data set, and a cost related to the comparison between the current clustering
and past clusterings~\citep{Chakraborti06_KDD,Chi07_KDD,Chi09_TransKDD,Tang12_IEEETKDE}. While the
weighting in earlier evolutionary algorithms is fixed, some authors have
considered adapting the weights online~\citep{Xu12_PhD,Xu12_DMKD}. The present 
work can be seen as a theoretically-founded extension of this class of 
algorithm that provides methods for automatic and adaptive prior weight
selection, forming correspondences between old and
current clusters, and for deciding when to introduce new clusters. 
Furthermore, the work presented herein is capable of explicit modeling of cluster
birth, death, and revival, which is not present in previous evolutionary clustering
work. Evolutionary clustering has been extended to include spectral
methods~\citep{Chi07_KDD}, but in addition to the aforementioned shortcomings of
evolutionary clustering, this method assumes that there is a known
correspondence between datapoints at different timesteps. If, as in the present
work, this assumption does not hold, spectral evolutionary clustering reduces to
simply clustering each batch of data individually at each timestep. 

\paragraph{Evolutionary Clustering with BNPs}
Evolutionary clustering has also been considered in the context of Bayesian
nonparametrics. Past work has focused on linking together multiple Dirichlet processes 
in a Markov chain by exponential smoothing on weights~\citep{Xu08a_ICDM,Xu08b_ICDM,Sun10_MLG} or
by utilizing on operations on the underlying Poisson processes~\citep{Lin10_NIPS}.
Inference for these models generally involves Gibbs sampling, which is
computationally expensive. The D-Means and SD-Means
algorithm, and graph clustering methods presented
herein, are closely related to the latter of these methods; indeed, they are 
derived from the small-variance asymptotics of the Gibbs sampler for this model.
However, in contrast to these methods, D-Means and SD-Means require no sampling steps, and have guarantees on convergence
to a local cost optimum in finite time.

\paragraph{Other Methods for Dynamic Clustering}
MONIC~\citep{Spiliopoulou06_KDD} and MC3~\citep{Kalnis05_SSTD} have the
capability to monitor time-varying clusters; however, these methods
require datapoints to be identifiable across timesteps, and determine cluster
similarity across timesteps via the commonalities between label assignments.
Incremental spectral clustering techniques often either make a similar correspondence
assumption~\citep{Ning10_PR}, or do not model cluster death or
motion~\citep{Valgren07_ICRA}. Both D-Means and SD-Means do not require such information, and track clusters
essentially based on temporal smoothness of the motion of parameters over time.
Finally, some sequential Monte-Carlo methods (e.g.~particle
learning~\citep{Carvalho10_SS} or multi-target
tracking~\citep{Hue02_TAES,Vermaak03_ICCV}) 
can be adapted for use in the present context, but suffer the typical drawbacks
(particle degeneracy and inadequate coverage in high dimensions) of particle filtering methods.  

\section{Experiments}\label{sec:experiments}
All experiments were run on a computer with an Intel i7 processor and 16GB of
memory.
\subsection{Synthetic Moving Gaussians}
\begin{figure}[t!]
\captionsetup{font=scriptsize}
\begin{center}
  \begin{subfigure}[t]{.32\linewidth}
  \includegraphics[width=\textwidth]{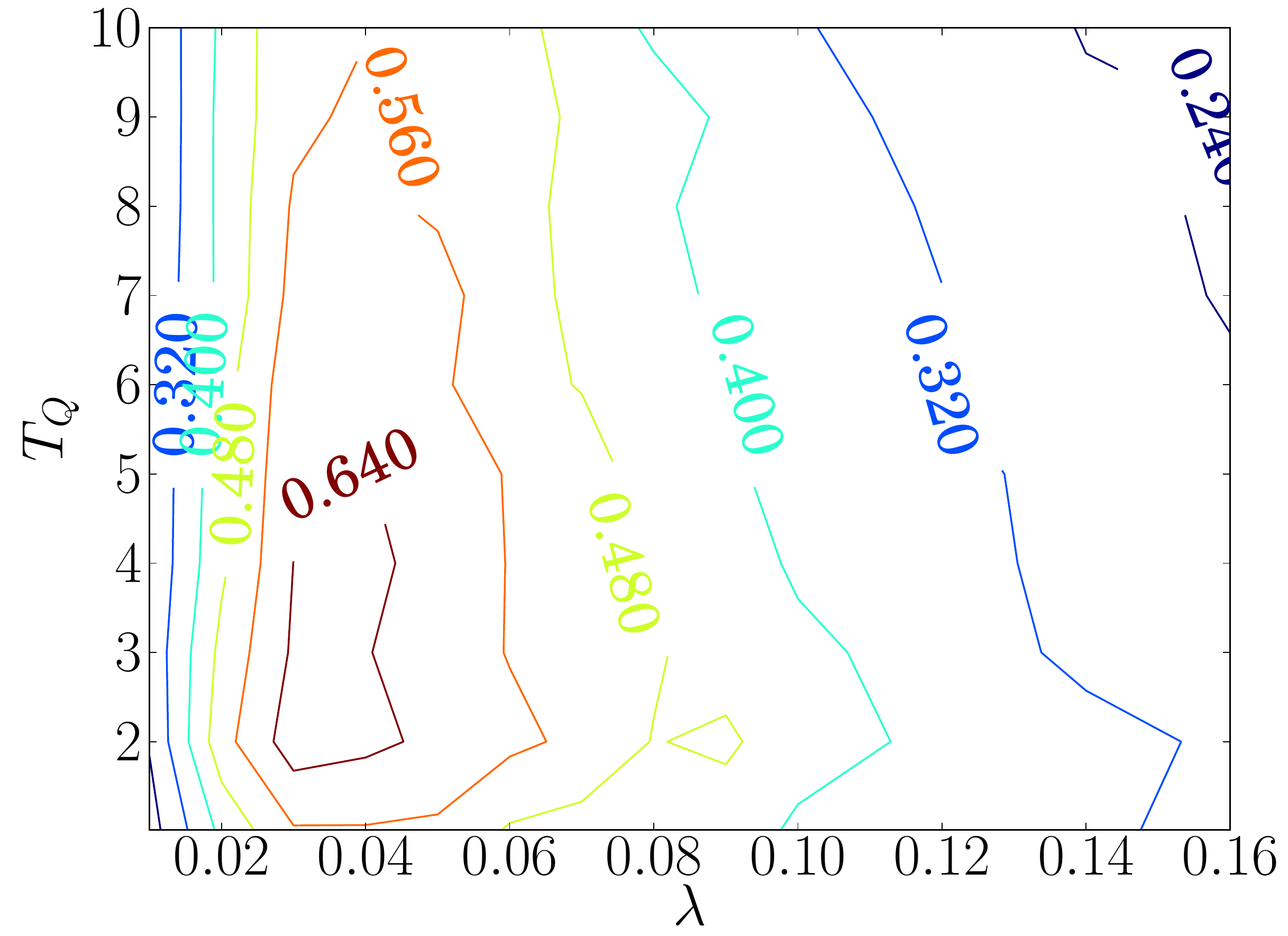}
    \caption{}\label{fig:TuneDynM-LamTQ}
  \end{subfigure}
  \begin{subfigure}[t]{.32\linewidth}
  \includegraphics[width=\textwidth]{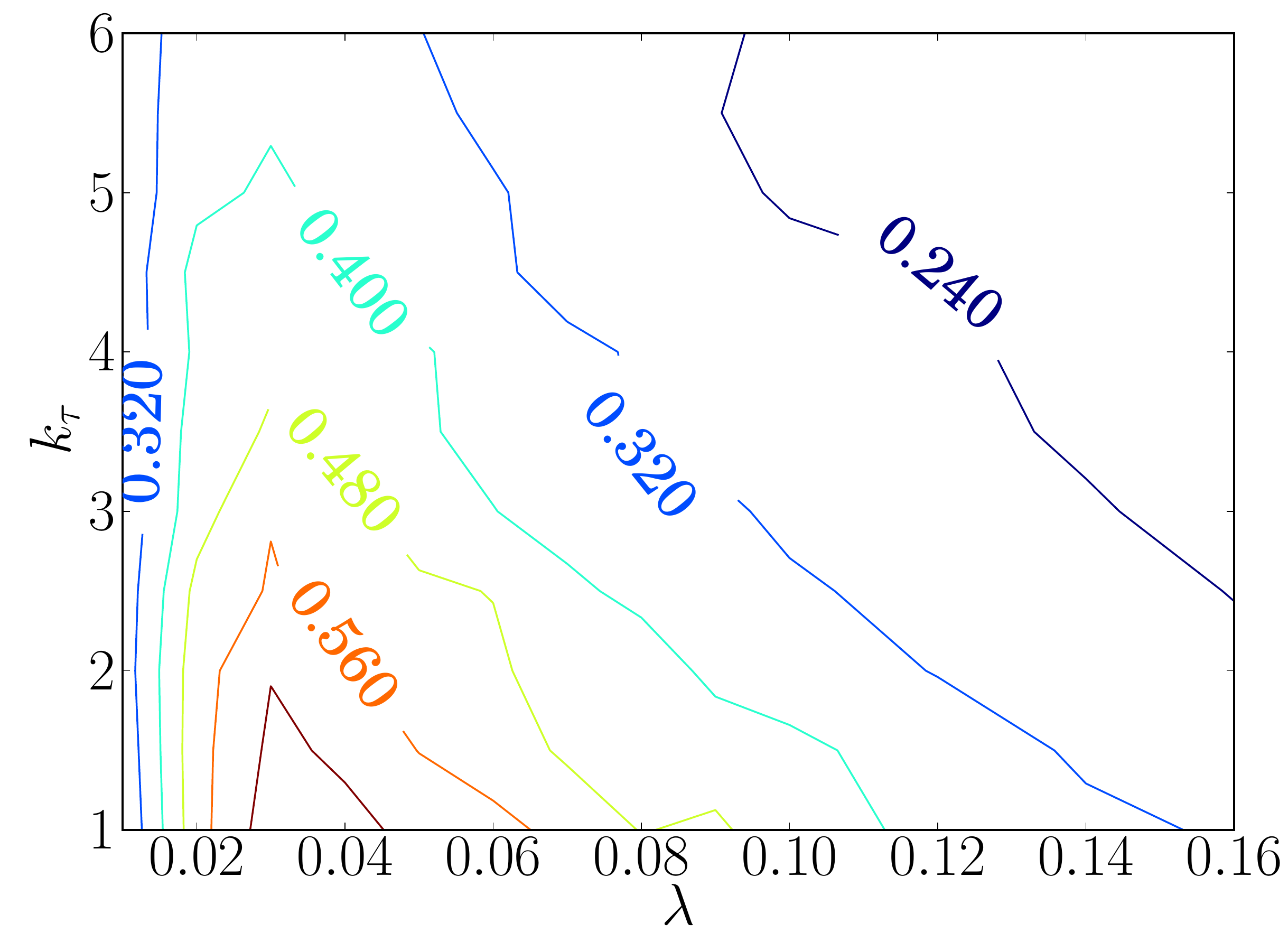}
    \caption{}\label{fig:TuneDynM-LamKT}
  \end{subfigure}
  \begin{subfigure}[t]{.32\linewidth}
  \includegraphics[width=\textwidth]{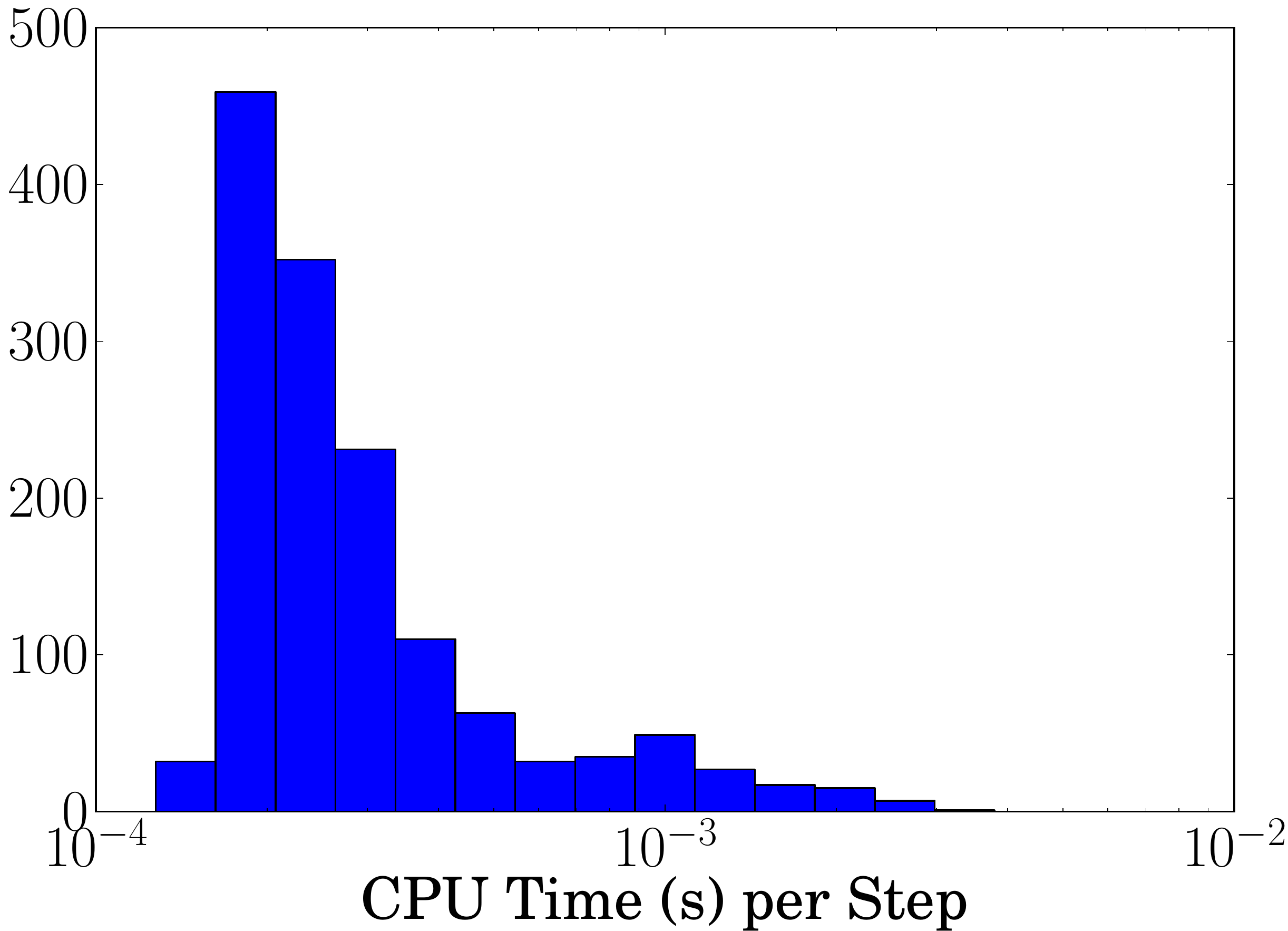}
    \caption{}\label{fig:TuneDynM-CpuT}
  \end{subfigure}
  \begin{subfigure}[b]{.32\linewidth}
  \includegraphics[width=1.06\textwidth]{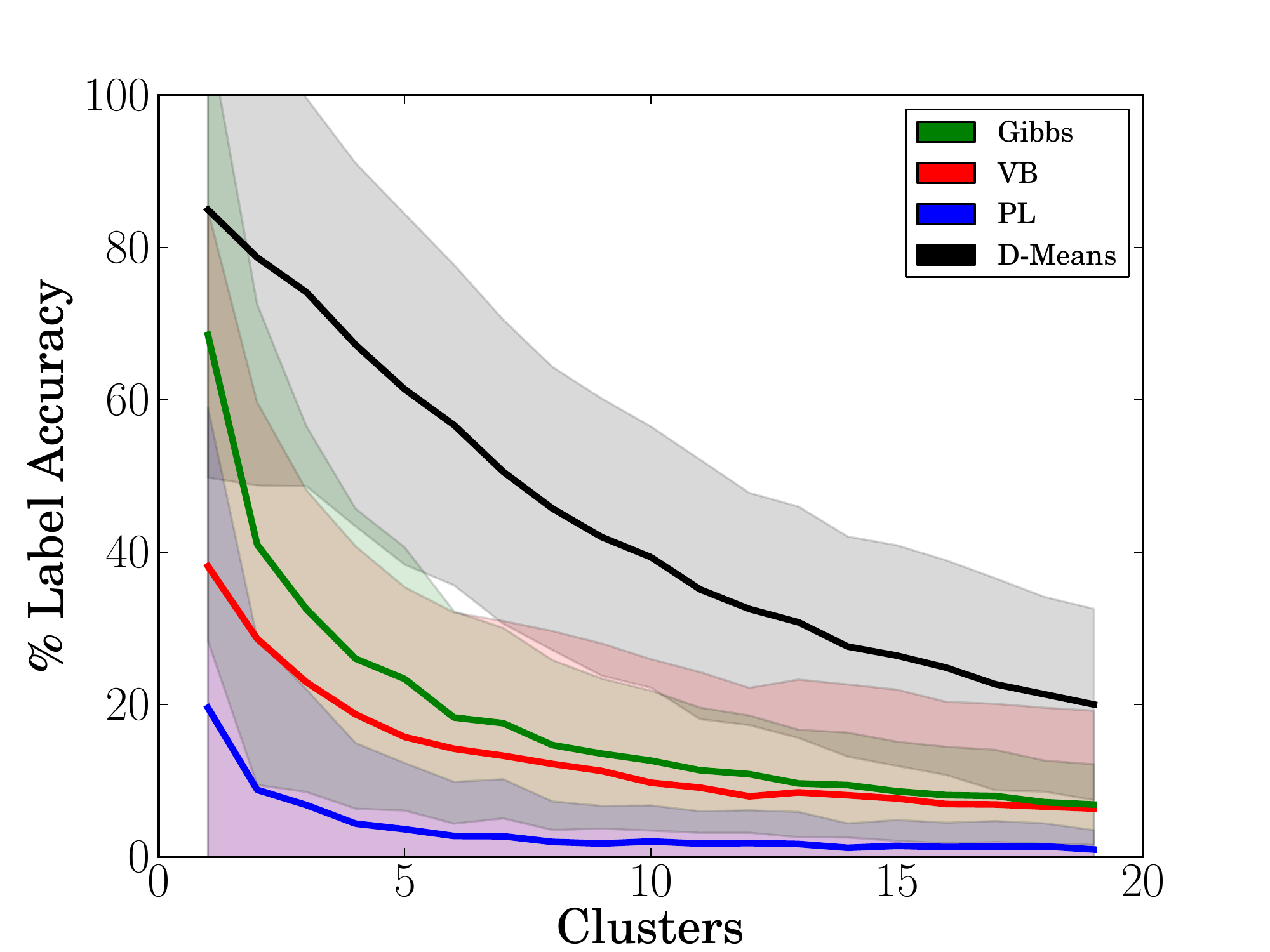}
    \caption{}\label{fig:DMvsDDP-acc}
  \end{subfigure}
  \begin{subfigure}[b]{.32\linewidth}
    \includegraphics[width=1.07\textwidth]{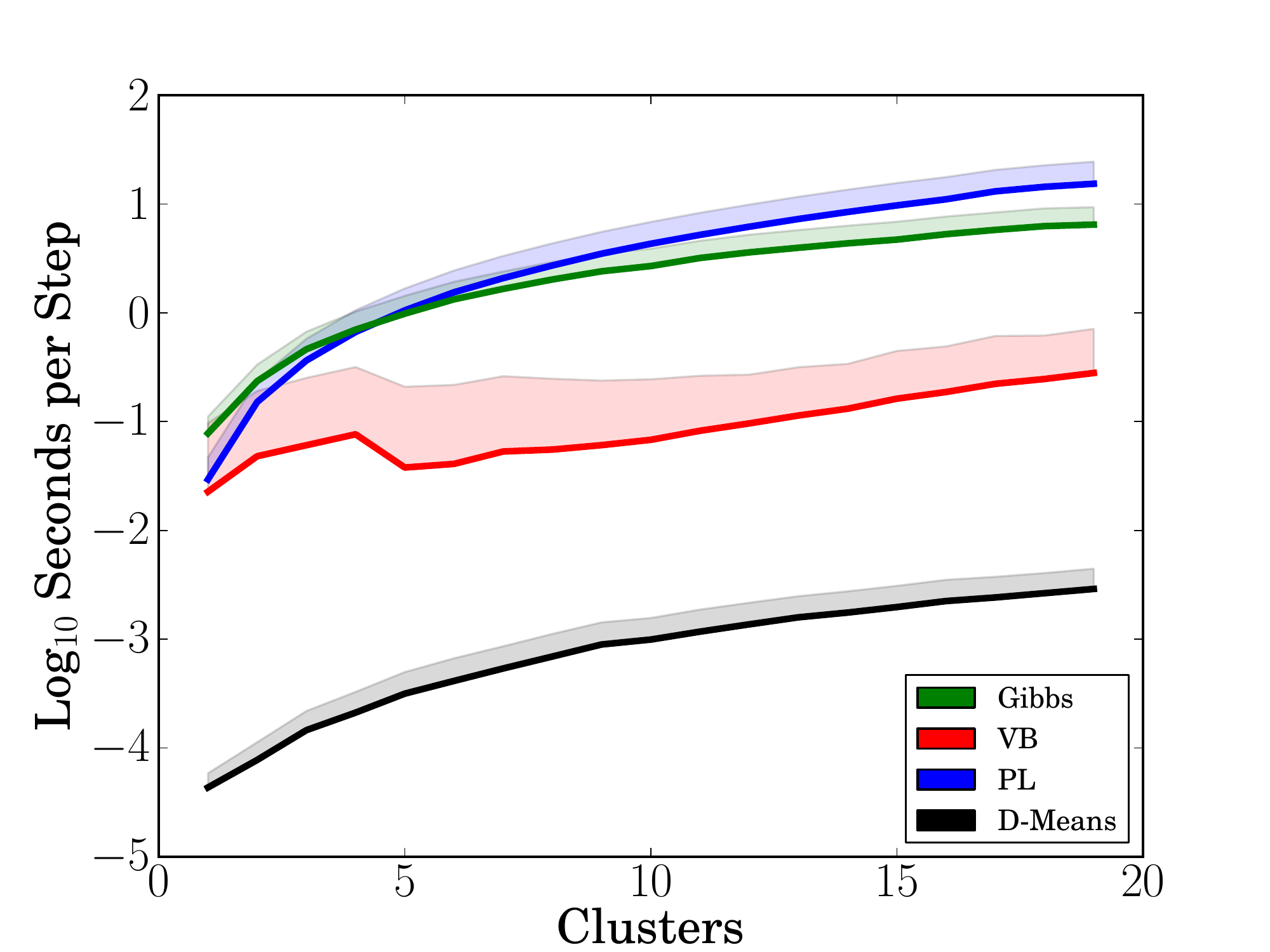}
    \caption{}\label{fig:DMvsDDP-time}
  \end{subfigure}
  \begin{subfigure}[b]{.32\linewidth}
    \includegraphics[width=.97\textwidth]{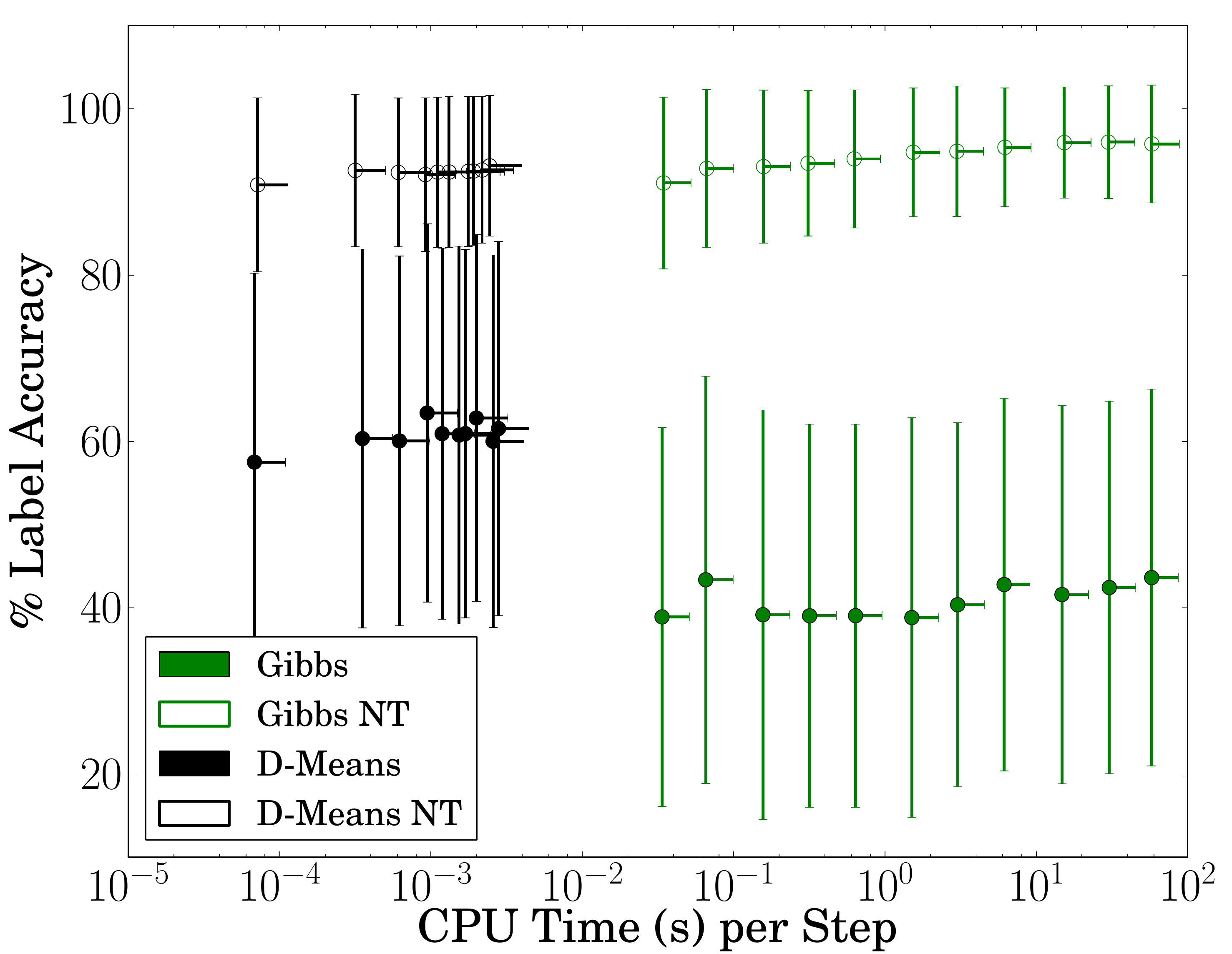}
    \caption{}\label{fig:DMvsDDP-accvtime}
  \end{subfigure}
  \begin{subfigure}[t]{.19\linewidth}
  \includegraphics[width=\columnwidth]{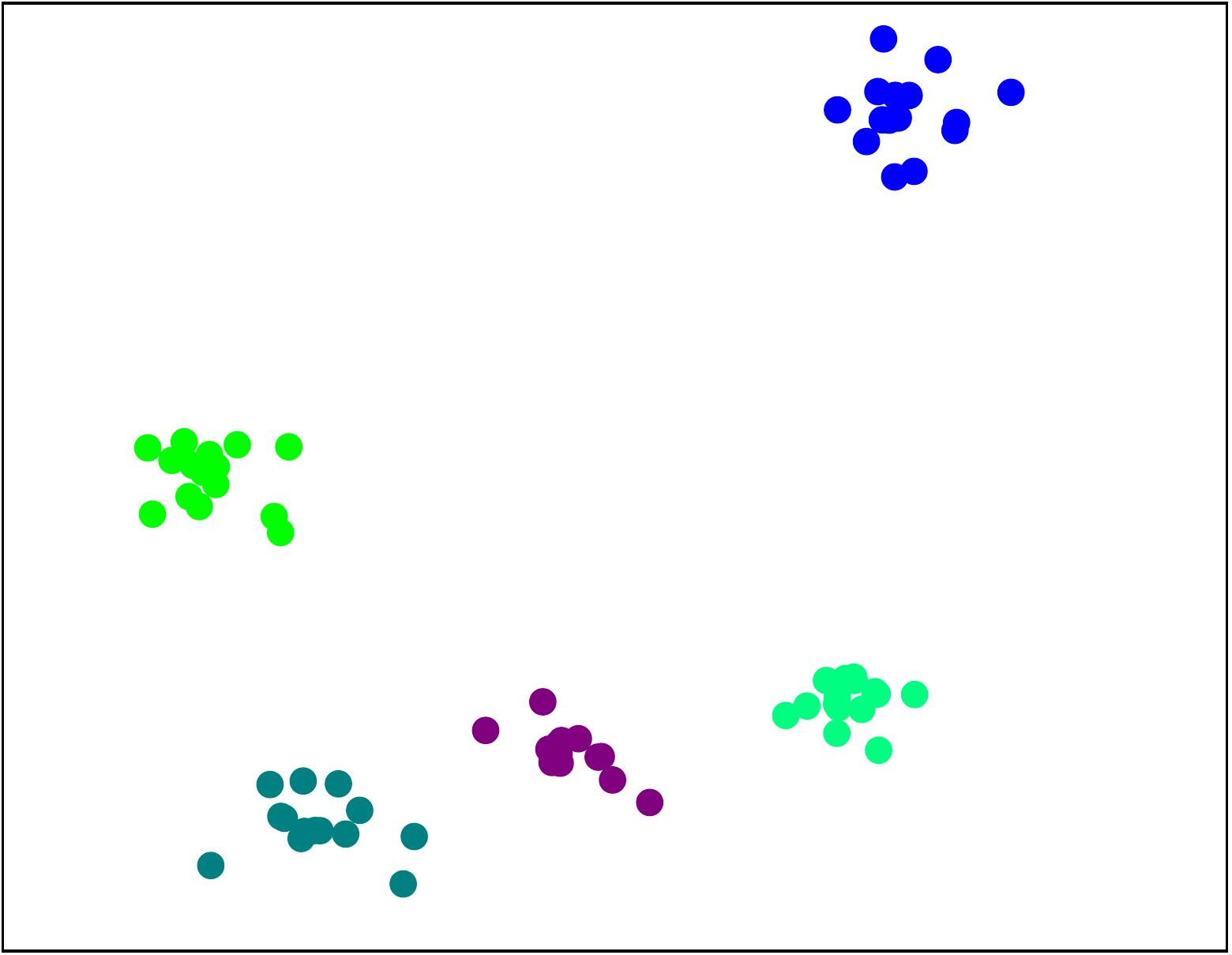}
  \end{subfigure}
  \begin{subfigure}[t]{.19\linewidth}
  \includegraphics[width=\columnwidth]{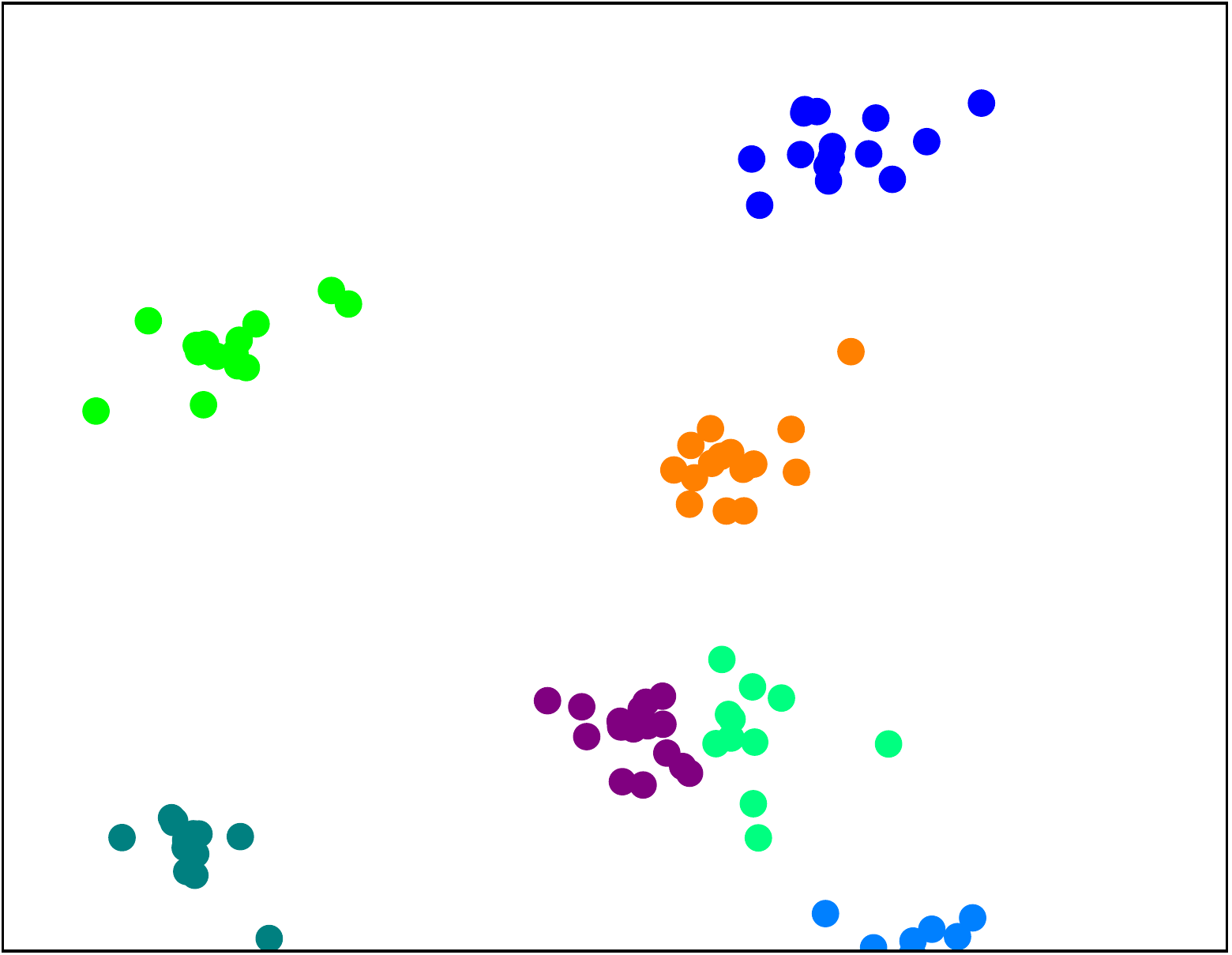}
  \end{subfigure}
  \begin{subfigure}[t]{.19\linewidth}
  \includegraphics[width=\columnwidth]{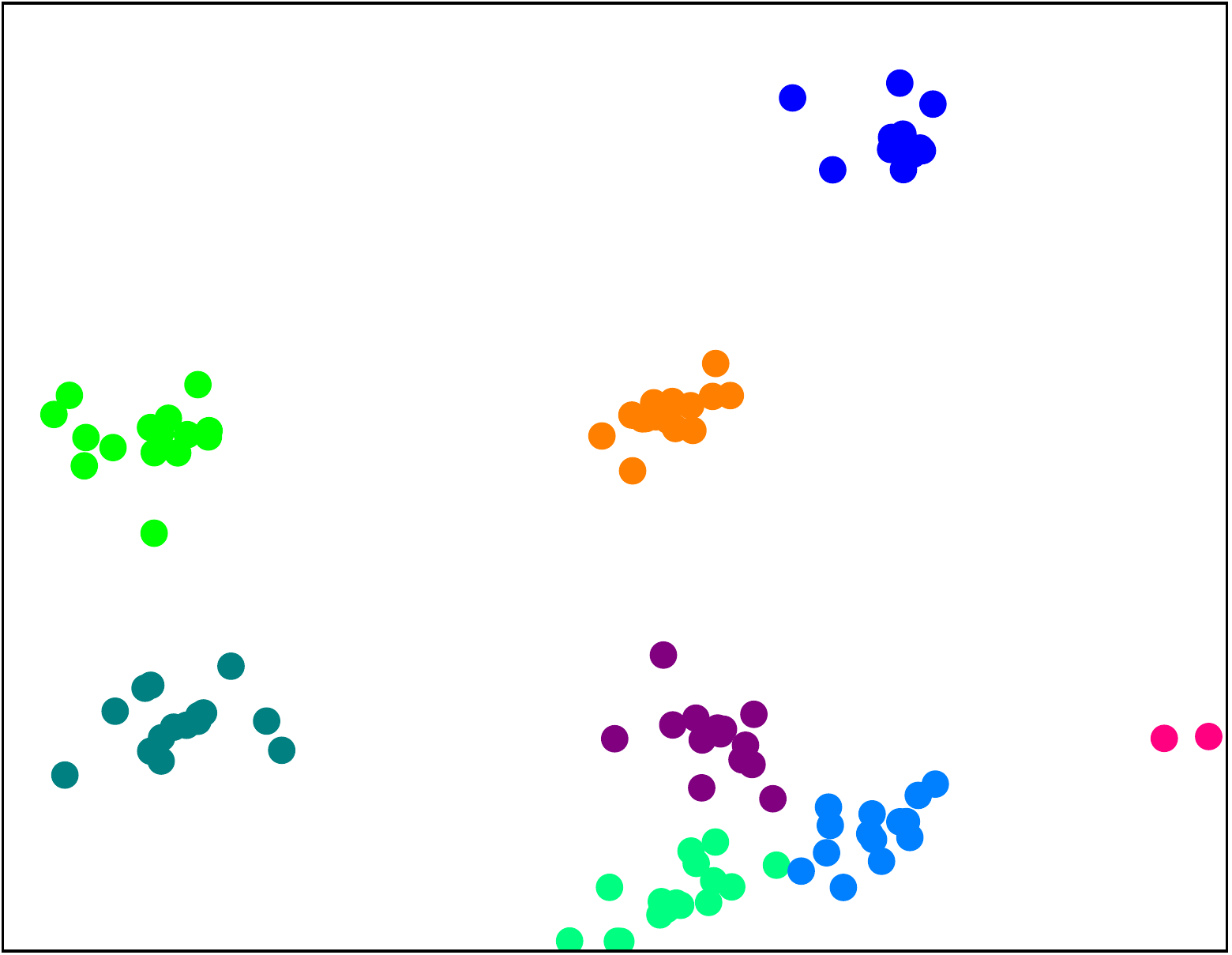}
  \caption{}\label{fig:gdmtrace}
  \end{subfigure}
  \begin{subfigure}[t]{.19\linewidth}
  \includegraphics[width=\columnwidth]{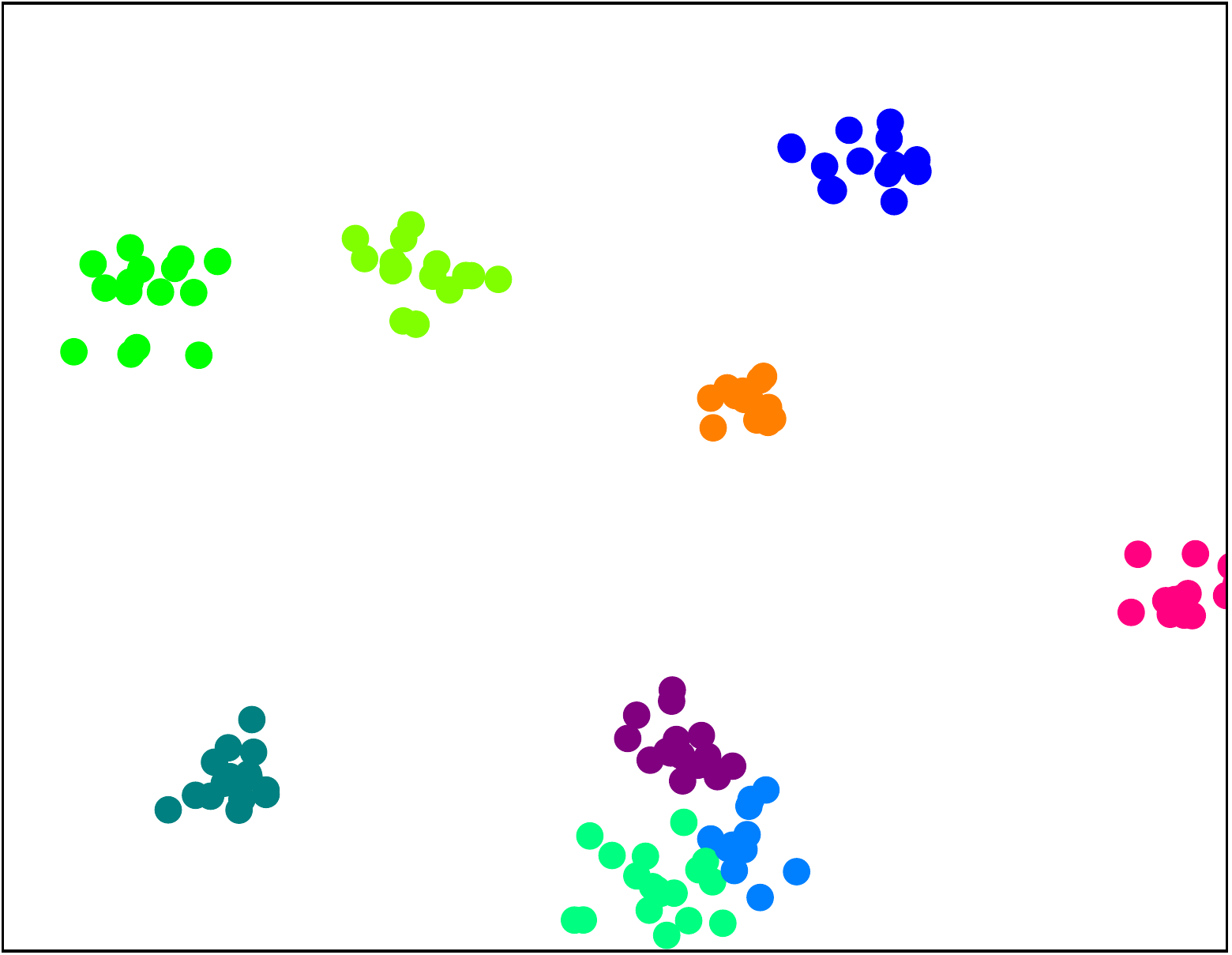}
  \end{subfigure}
  \begin{subfigure}[t]{.19\linewidth}
  \includegraphics[width=\columnwidth]{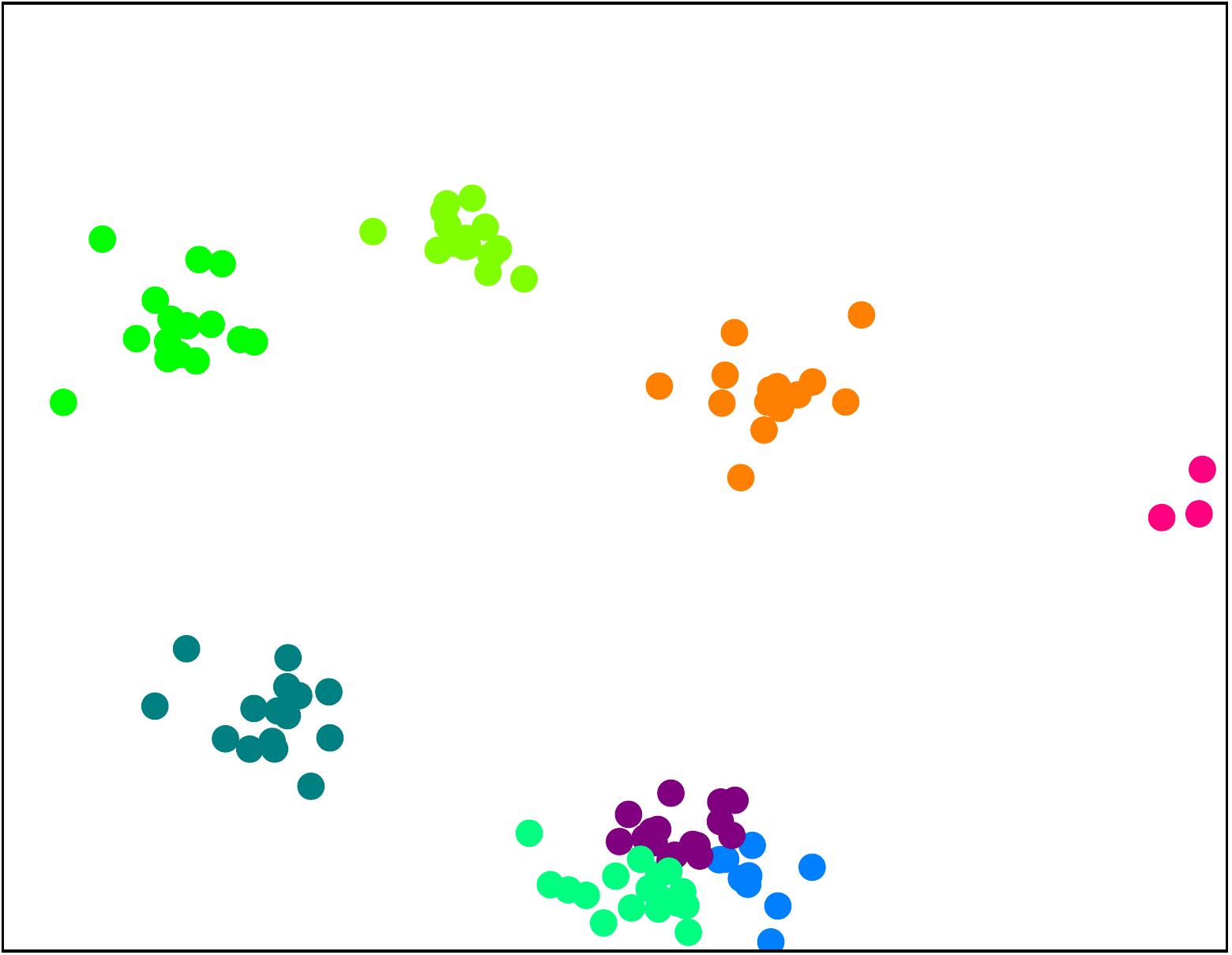}
  \end{subfigure}

  \caption{(\ref{fig:TuneDynM-LamTQ} - \ref{fig:TuneDynM-CpuT}): Accuracy
  contours and CPU time histogram for D-Means.
  (\ref{fig:DMvsDDP-acc} - \ref{fig:DMvsDDP-time}): Comparison
  with Gibbs sampling, variational inference, and particle learning. Shaded
  region indicates $1\sigma$ interval; in (\ref{fig:DMvsDDP-time}), only upper
  half is shown. (\ref{fig:DMvsDDP-accvtime}): Comparison of accuracy when
  enforcing (Gibbs, D-Means) and not enforcing (Gibbs NT, D-Means NT) correct
  cluster tracking. (\ref{fig:gdmtrace}): An example 5 sequential frames from clustering
  using D-Means; note the consistency of datapoint labelling (color) over time.
}\label{fig:SynthTest}
\end{center}
\end{figure}

The first experiment was designed to provide an empirical backing to the
theoretical developments in the foregoing sections. This was achieved through
the use of synthetic mixture model data, with known 
cluster assignments for accuracy comparisons.
In particular, moving Gaussian clusters on $[0,1]\times[0,1]$ 
were generated over a period of 100 time steps, with the
number of clusters fixed. 
At each time step, each cluster had 15 data points sampled from
a isotropic Gaussian distribution with a standard deviation of 0.05. 
Between time steps, the cluster centers moved randomly, with displacements
sampled from the same distribution. At each time
step, each cluster had a 0.05 probability of being destroyed, and upon
destruction a new cluster was created with a random location in the domain.
This experiment involved four algorithms: D-Means with 3 random
assignment ordering restarts;
Gibbs sampling on the full MCGMM with 5000 samples~\citep{Lin10_NIPS}; variational
inference to a tolerance of $10^{-16}$~\citep{Blei06_BA}; and particle
learning with 100 particles~\citep{Carvalho10_BA}.

First, the parameter space of each algorithm was searched for the
best average cluster label accuracy over 50 trials when the number of clusters was fixed to
5. The results of this parameter sweep for
 D-Means are shown in Figures \ref{fig:TuneDynM-LamTQ}--\ref{fig:TuneDynM-CpuT}. 
Figures~\ref{fig:TuneDynM-LamTQ} and \ref{fig:TuneDynM-LamKT} show how the average clustering
accuracy varies with the parameters, after fixing either $k_{\tau}$ or $T_Q$ to their
values at the maximum accuracy setting. D-Means
had a similar robustness with respect to variations in its parameters
as the other algorithms. The histogram in
Figure \ref{fig:TuneDynM-CpuT} demonstrates that the clustering speed is robust
to the setting of parameters. The speed of Dynamic Means, coupled with the
smoothness of its performance with respect to its parameters, makes it
well suited for automatic tuning~\citep{Snoek12_NIPS}.

Using the best parameter setting for D-Means ($\lambda = 0.04$, $T_Q = 6.8$, and
$k_\tau = 1.01$) and the true generative model parameters for Gibbs/PL/VB, 
the data were clustered in 50 trials with a varying number of
clusters present in the data.
In Figures \ref{fig:DMvsDDP-acc} and
\ref{fig:DMvsDDP-time}, the
labeling accuracy and clustering time
for the algorithms is shown versus the true number of clusters in the domain.
Since the domain is bounded, increasing the true number of clusters makes the
labeling of individual data more ambiguous, yielding lower accuracy for all
algorithms. These results illustrate that D-Means
outperforms standard inference algorithms in both label accuracy and
computational cost for cluster tracking problems.

Note that the sampling algorithms were
handicapped to generate 
Figure \ref{fig:DMvsDDP-acc}; the best posterior sample 
in terms of labeling accuracy 
was selected at each time step, which required knowledge of the true
labeling. Further, the accuracy computation involved solving a maximum 
matching problem between the learned and true data labels at each timestep,
and then removing all correspondences inconsistent with matchings from previous
timesteps, thus enforcing consistent cluster tracking over time. If inconsistent
correspondences aren't removed (i.e.~label accuracy computations
consider each time step independently), the other algorithms 
provide accuracies more comparable to those of D-Means. 
This effect is demonstrated in 
Figure \ref{fig:DMvsDDP-accvtime}, which shows the time/accuracy tradeoff
for D-means (varying the number of restarts) and MCGMM (varying the number of
samples) when label consistency is enforced, and when it is not enforced.

\subsection{Synthetic Moving Rings}
\begin{figure}[t!]
\captionsetup{font=scriptsize}
\begin{center}
  \begin{subfigure}[t]{.32\linewidth}
    \raisebox{.1cm}{\includegraphics[width=\textwidth]{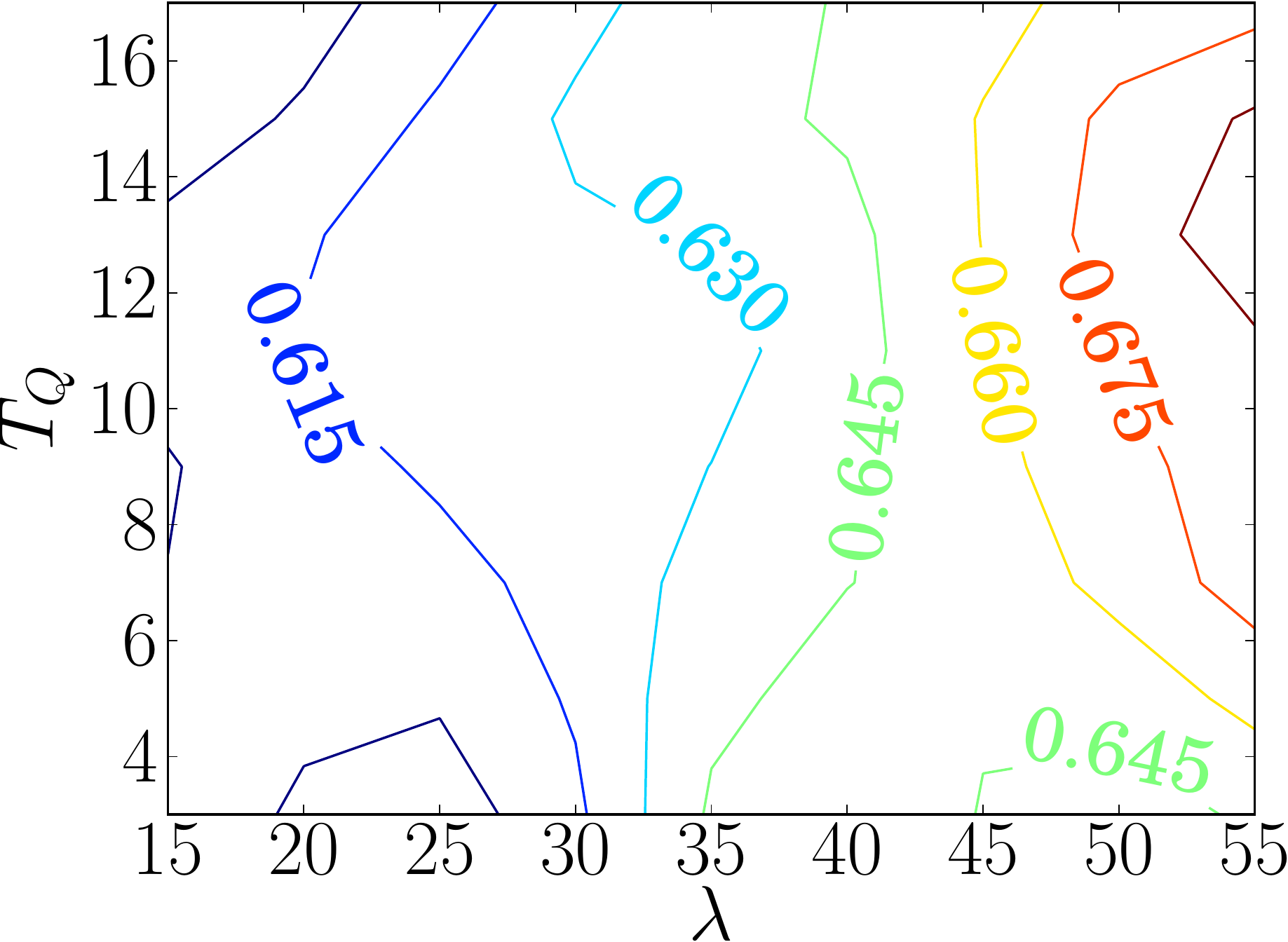}}
    \caption{}\label{fig:tunesdm-tqlamb}
  \end{subfigure}
  \begin{subfigure}[t]{.32\linewidth}
    \raisebox{.1cm}{\includegraphics[width=\textwidth]{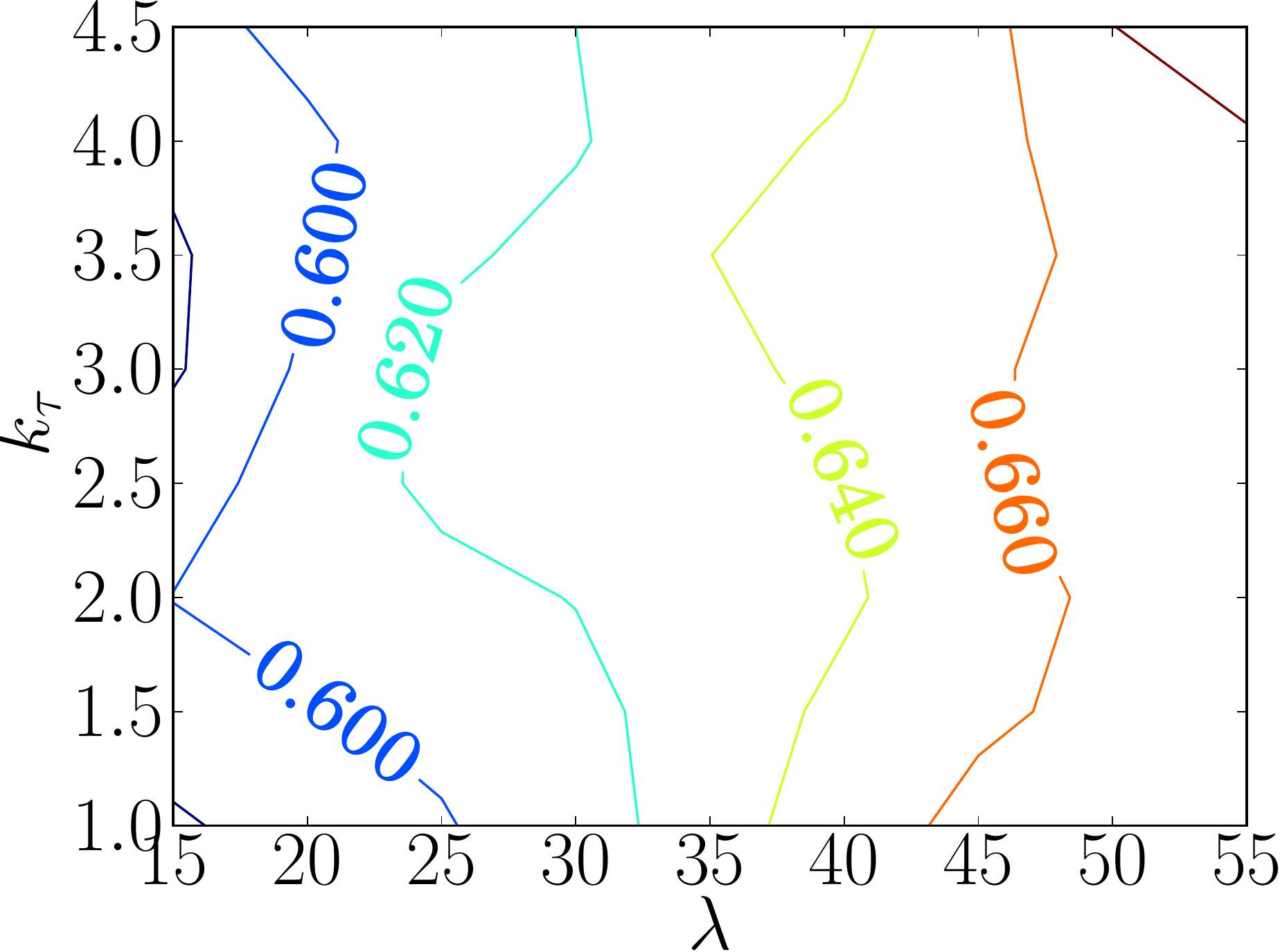}}
  \caption{}\label{fig:tunesdm-ktlamb}
  \end{subfigure}
  \begin{subfigure}[t]{.33\linewidth}
  \includegraphics[width=\textwidth]{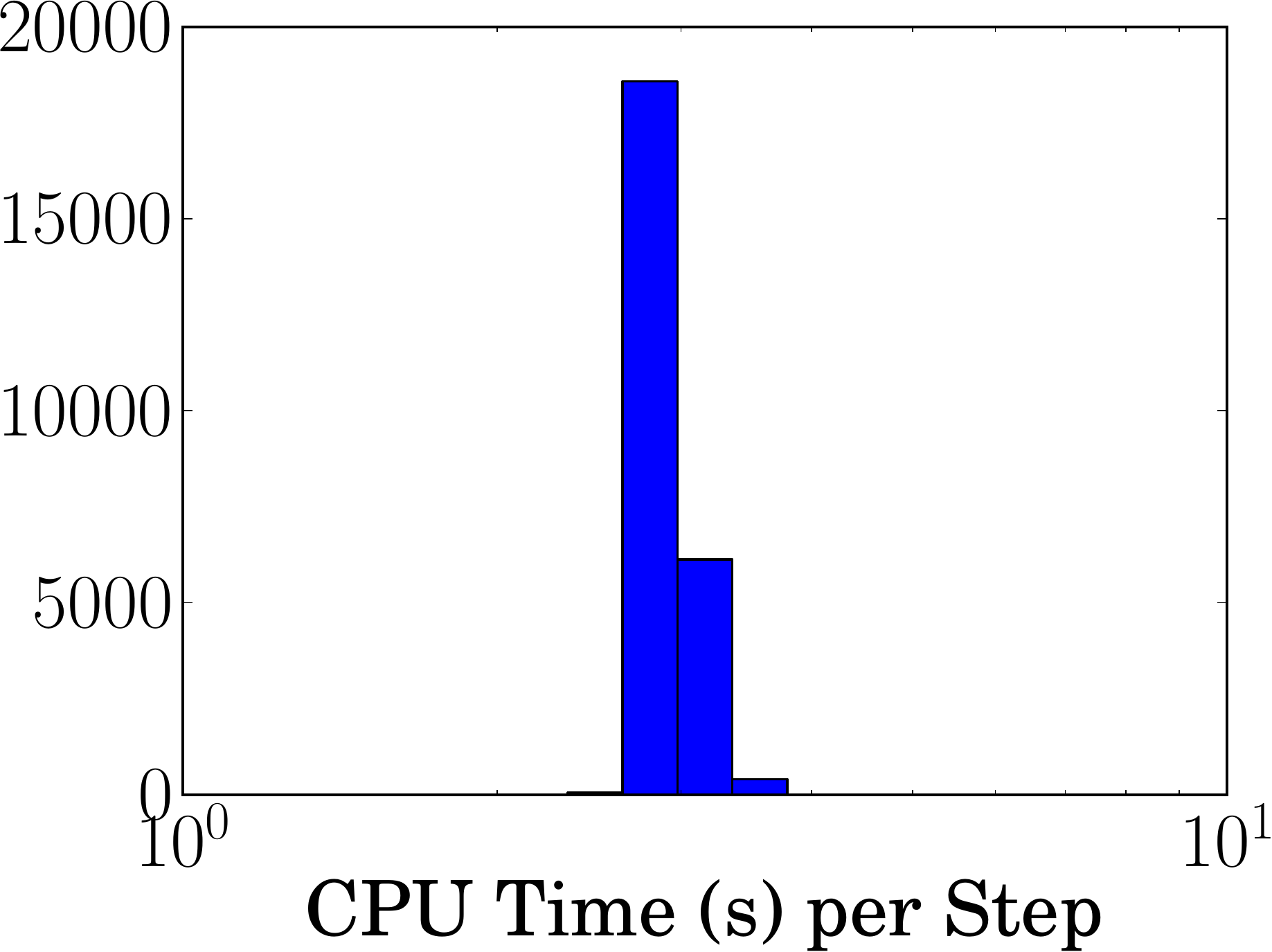}
  \caption{}\label{fig:tunesdm-cput}
  \end{subfigure}
  \begin{subfigure}[t]{.19\linewidth}
  \includegraphics[width=\columnwidth]{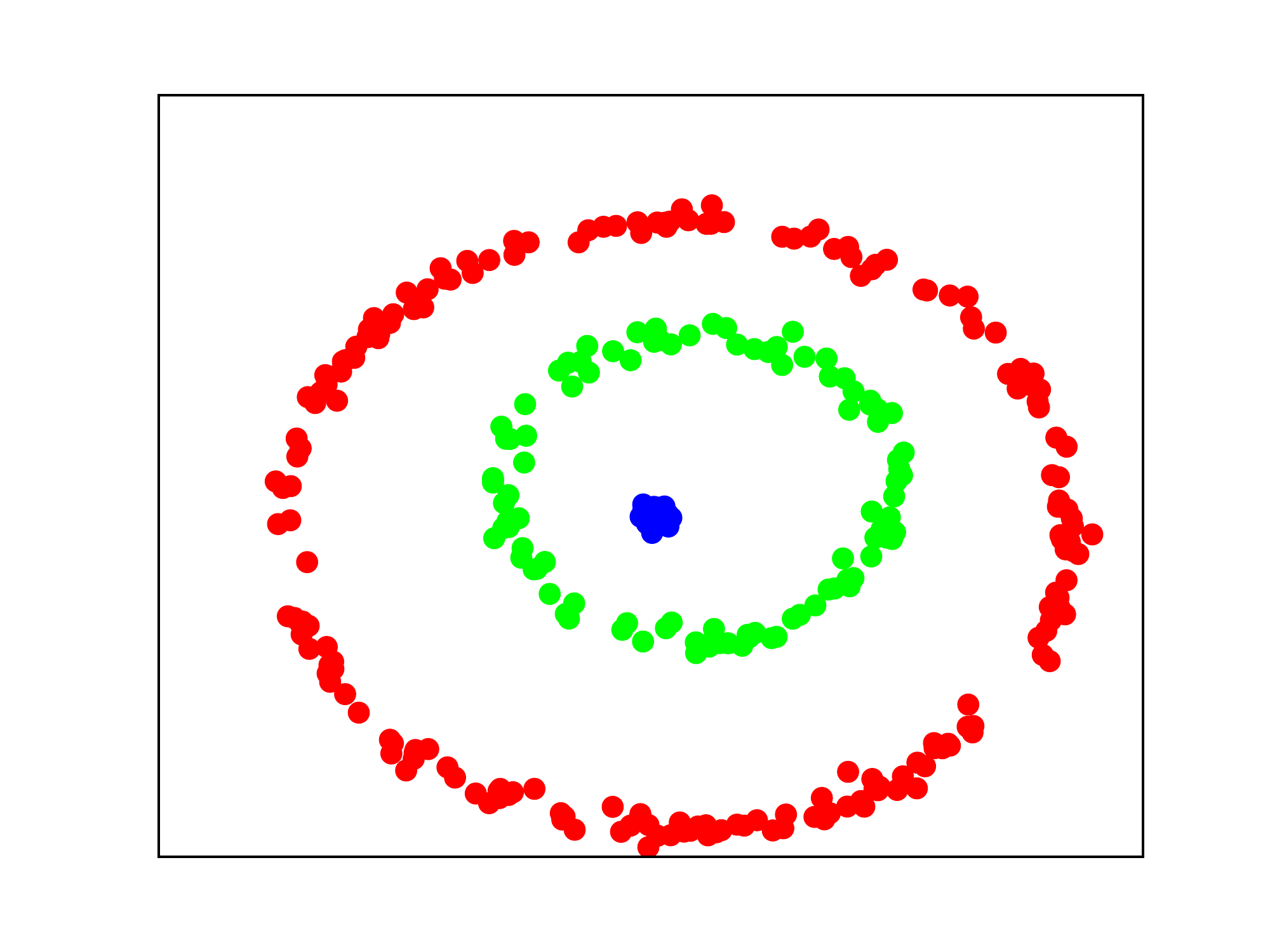}
  \end{subfigure}
  \begin{subfigure}[t]{.19\linewidth}
  \includegraphics[width=\columnwidth]{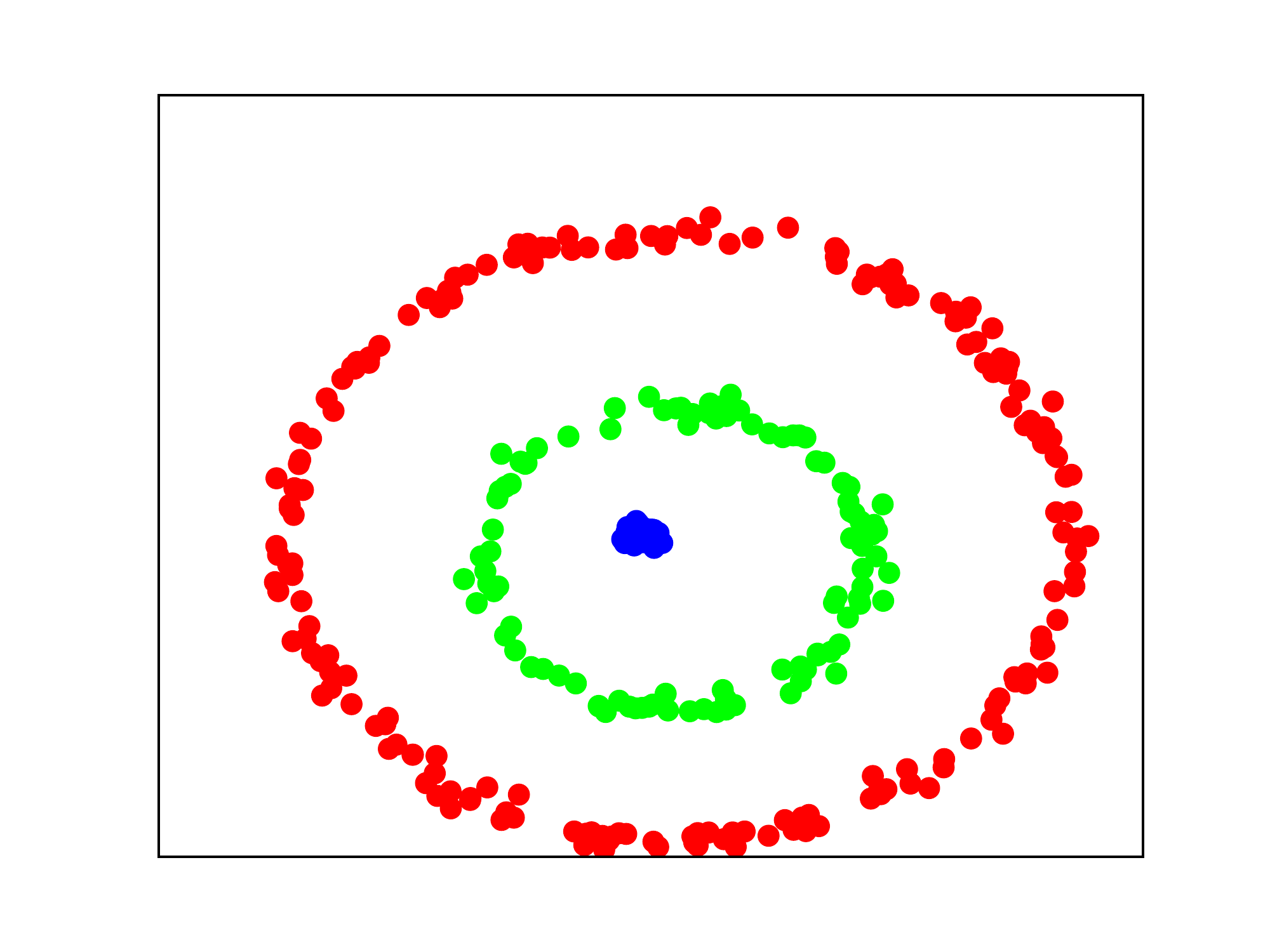}
  \end{subfigure}
  \begin{subfigure}[t]{.19\linewidth}
  \includegraphics[width=\columnwidth]{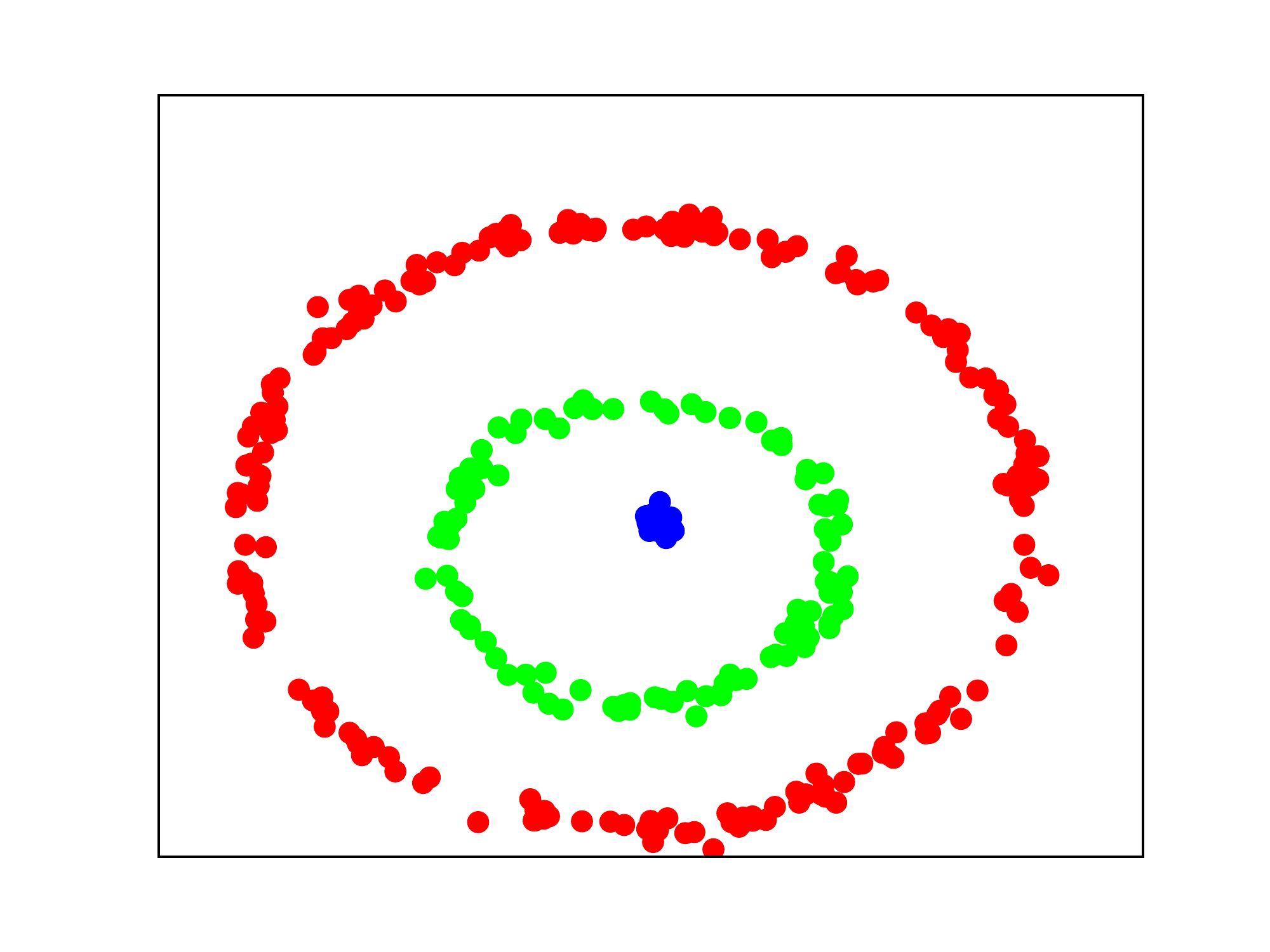}
  \caption{}\label{fig:sdmtrace}
  \end{subfigure}
  \begin{subfigure}[t]{.19\linewidth}
  \includegraphics[width=\columnwidth]{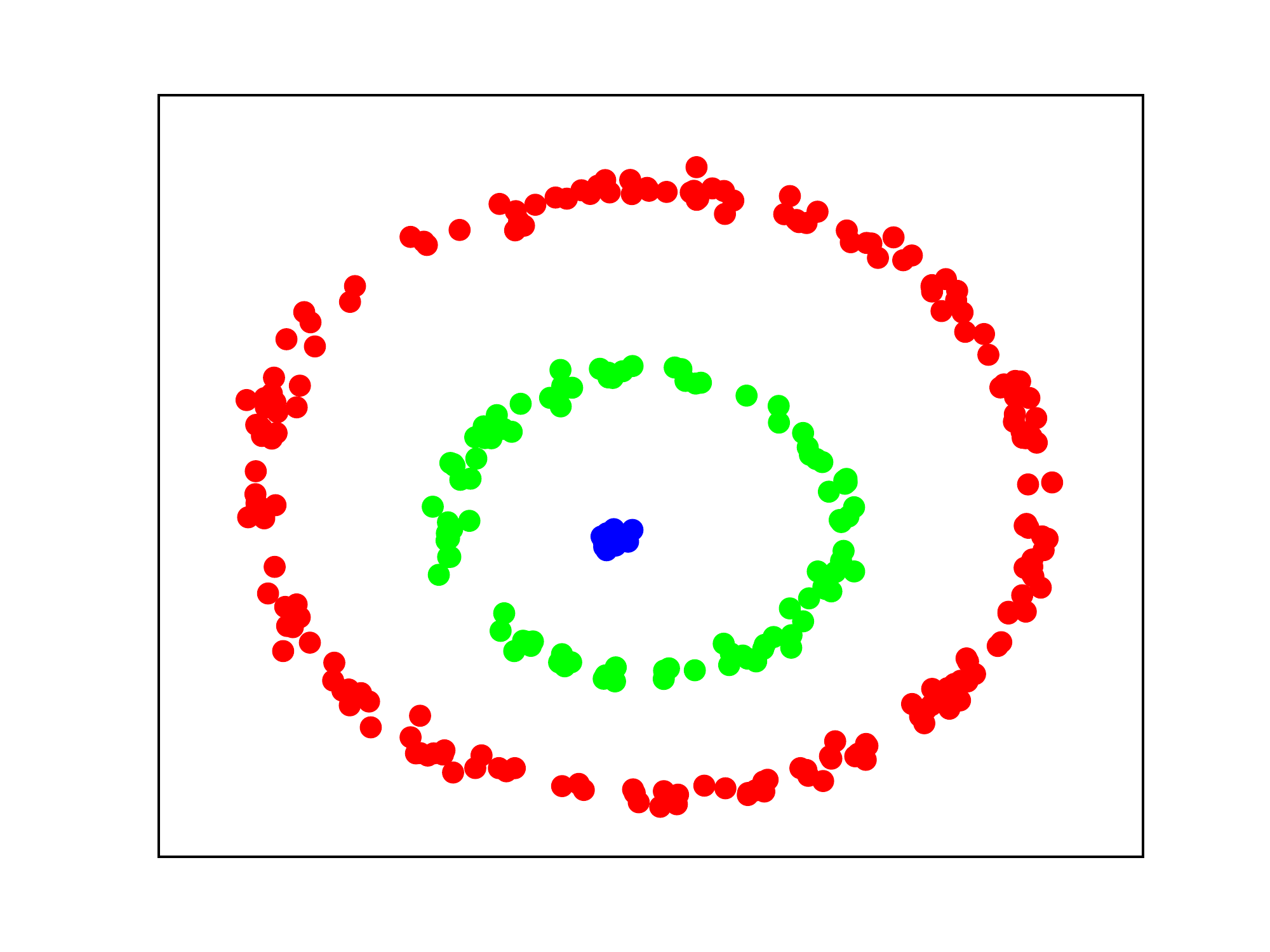}
  \end{subfigure}
  \begin{subfigure}[t]{.19\linewidth}
  \includegraphics[width=\columnwidth]{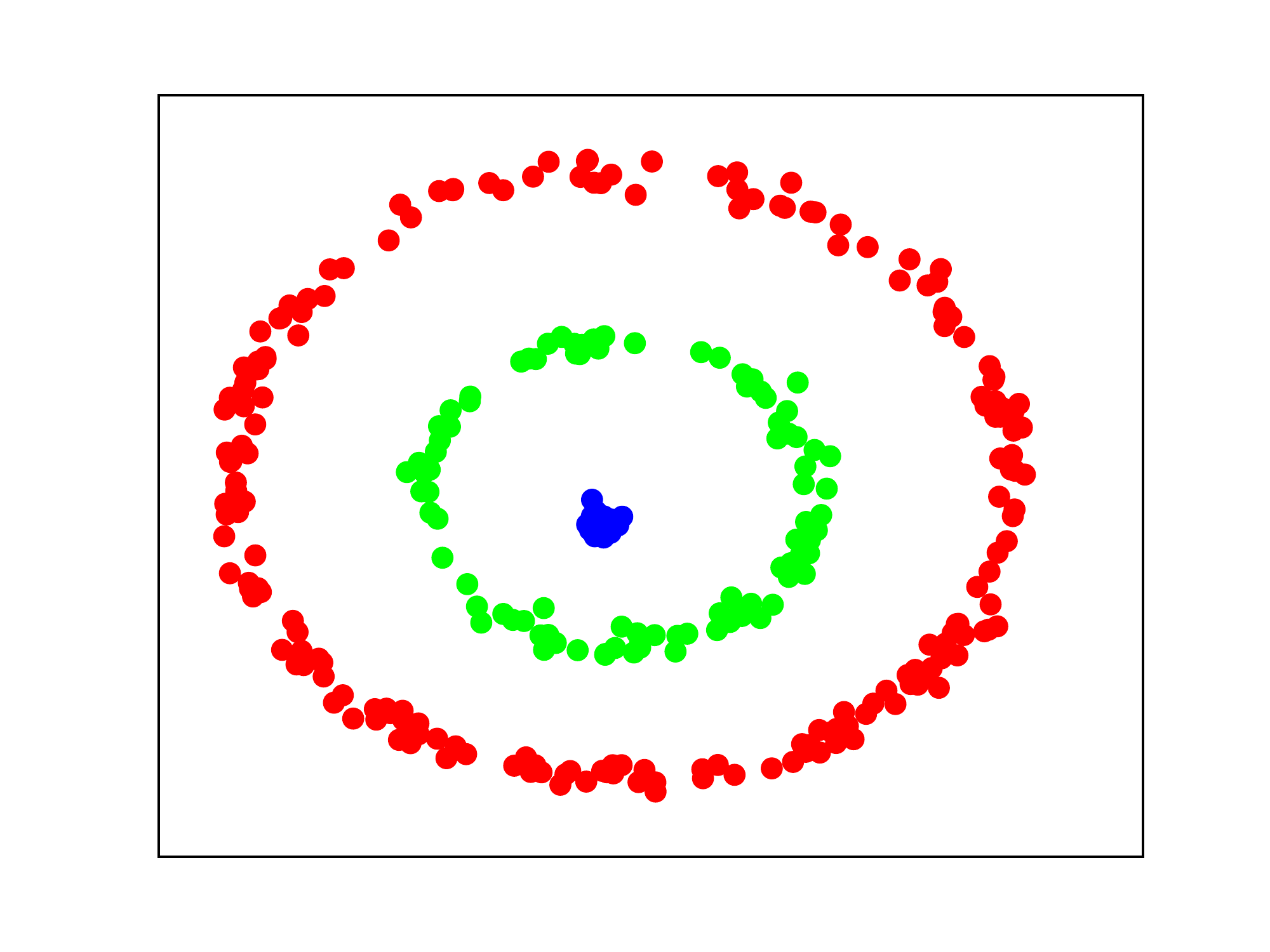}
  \end{subfigure}

  \begin{subfigure}[t]{.19\linewidth}
  \includegraphics[width=\columnwidth]{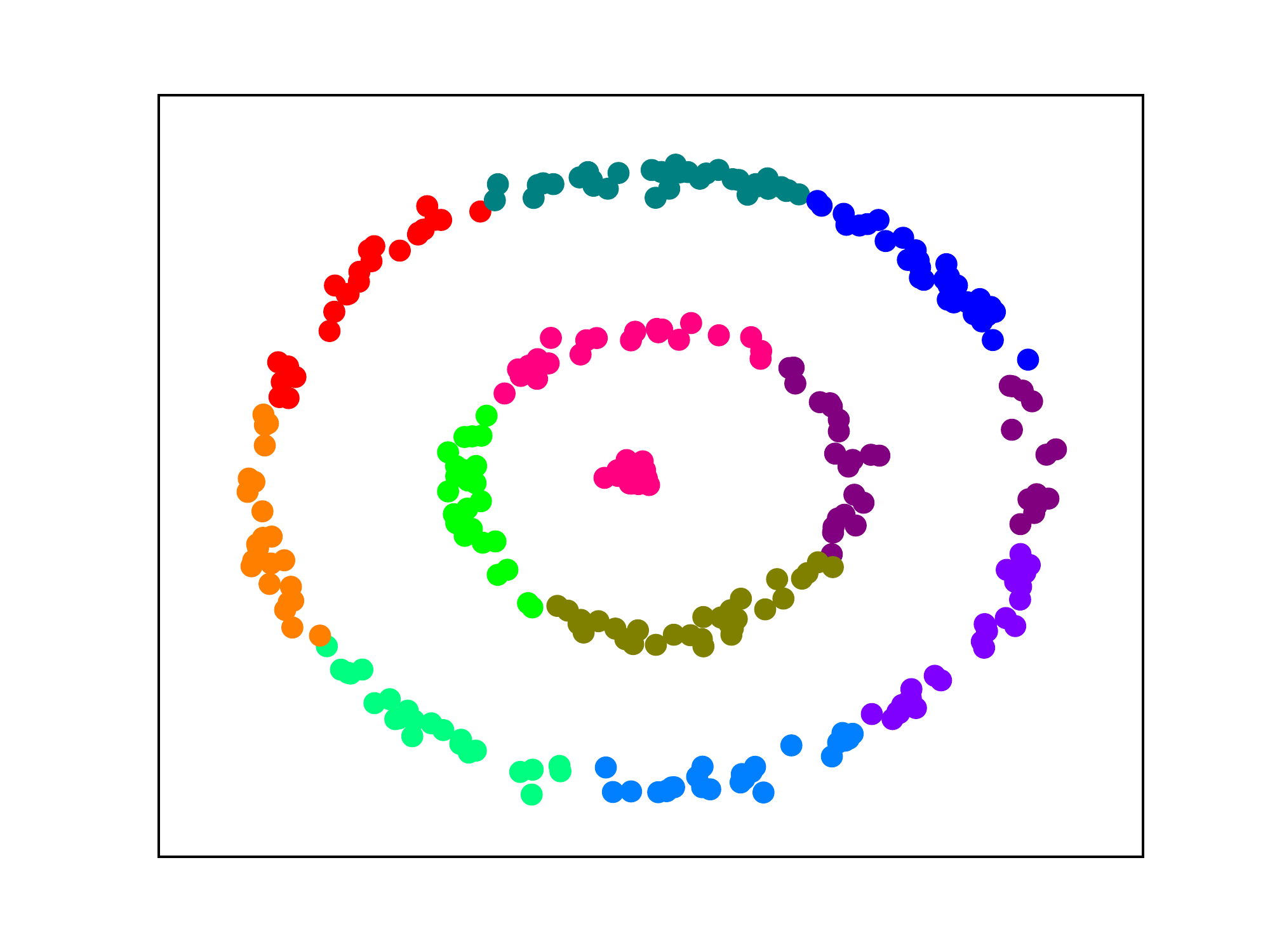}
  \end{subfigure}
  \begin{subfigure}[t]{.19\linewidth}
  \includegraphics[width=\columnwidth]{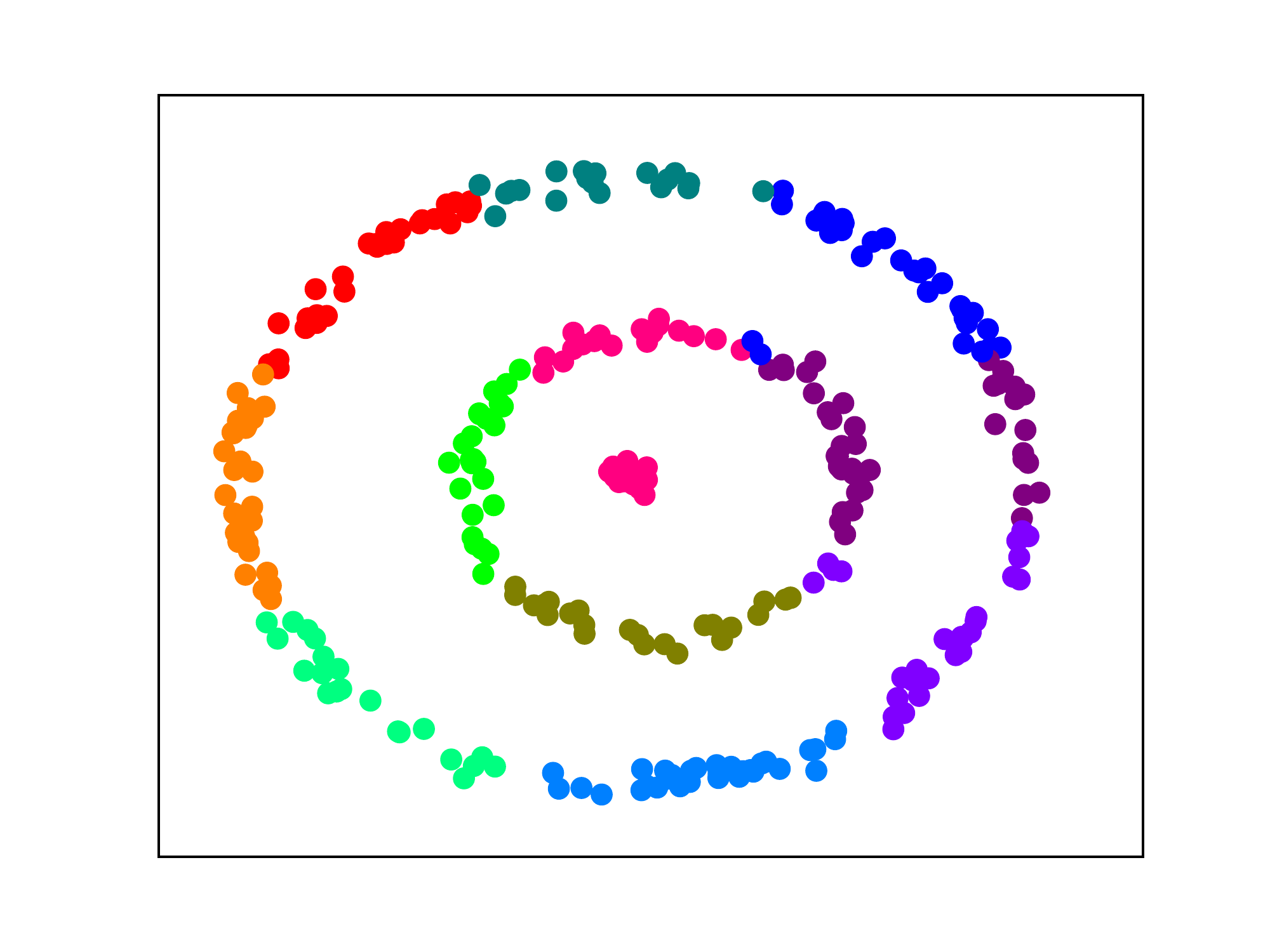}
  \end{subfigure}
  \begin{subfigure}[t]{.19\linewidth}
  \includegraphics[width=\columnwidth]{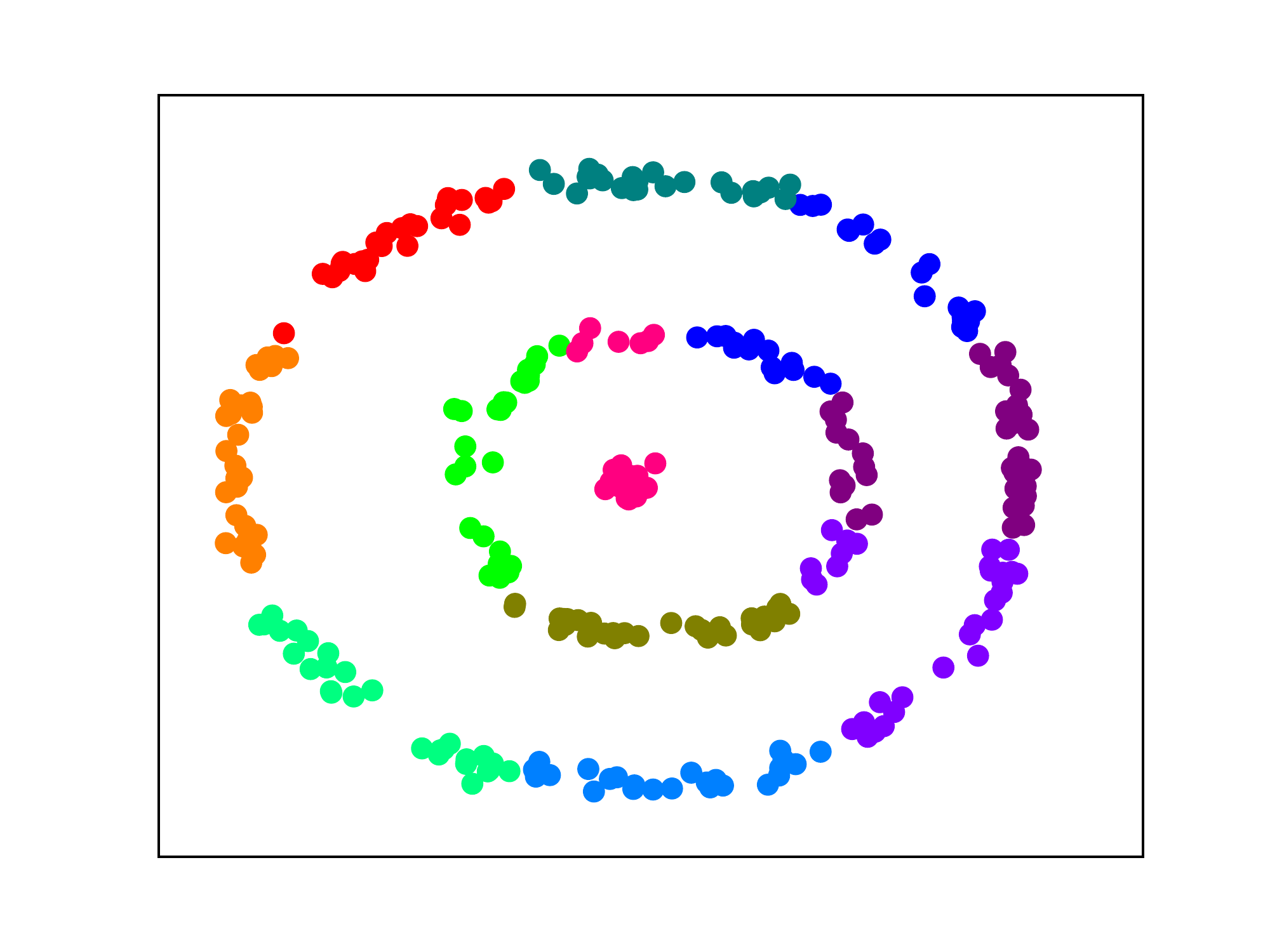}
  \caption{}\label{fig:dmtrace}
  \end{subfigure}
  \begin{subfigure}[t]{.19\linewidth}
  \includegraphics[width=\columnwidth]{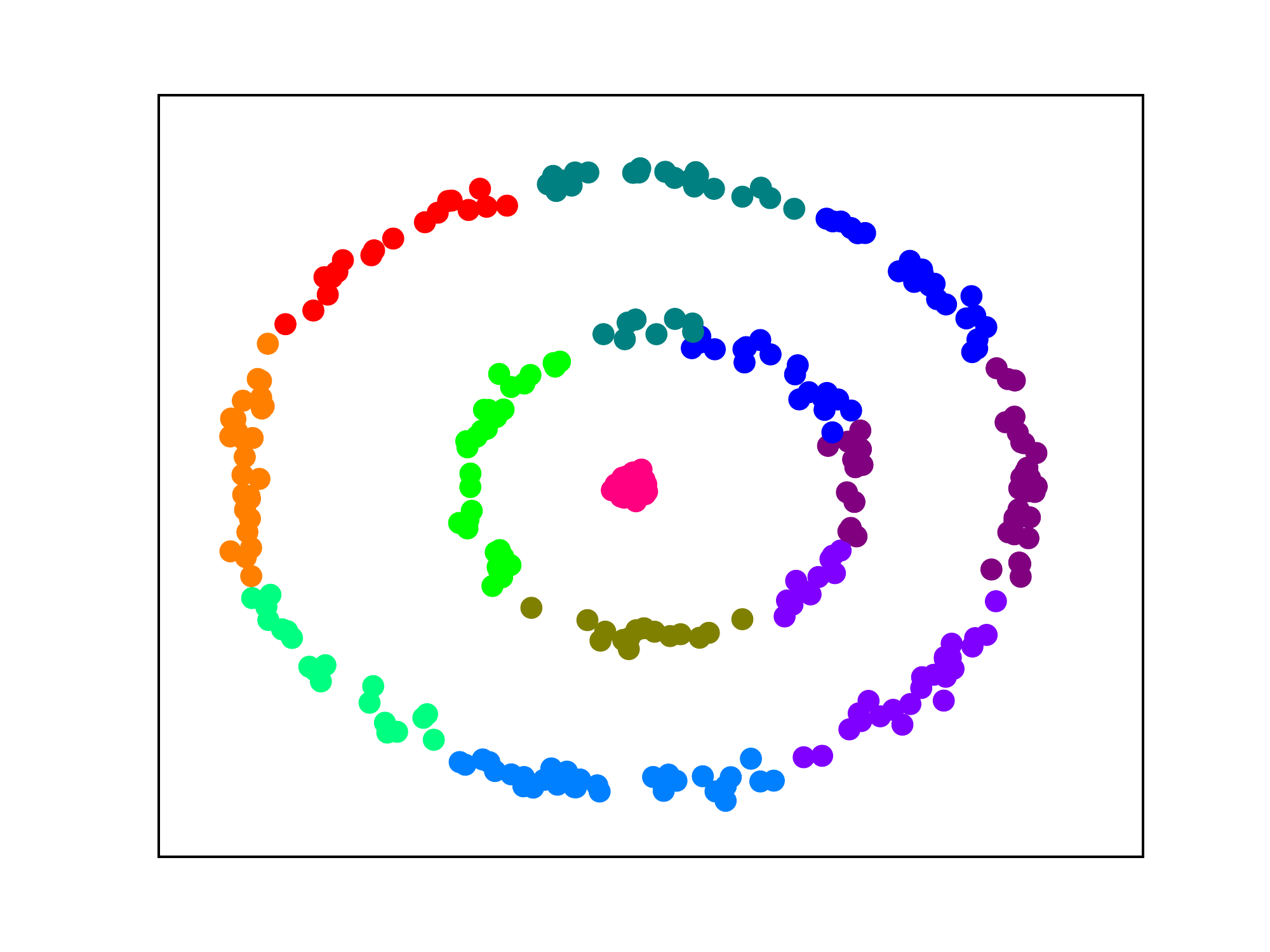}
  \end{subfigure}
  \begin{subfigure}[t]{.19\linewidth}
  \includegraphics[width=\columnwidth]{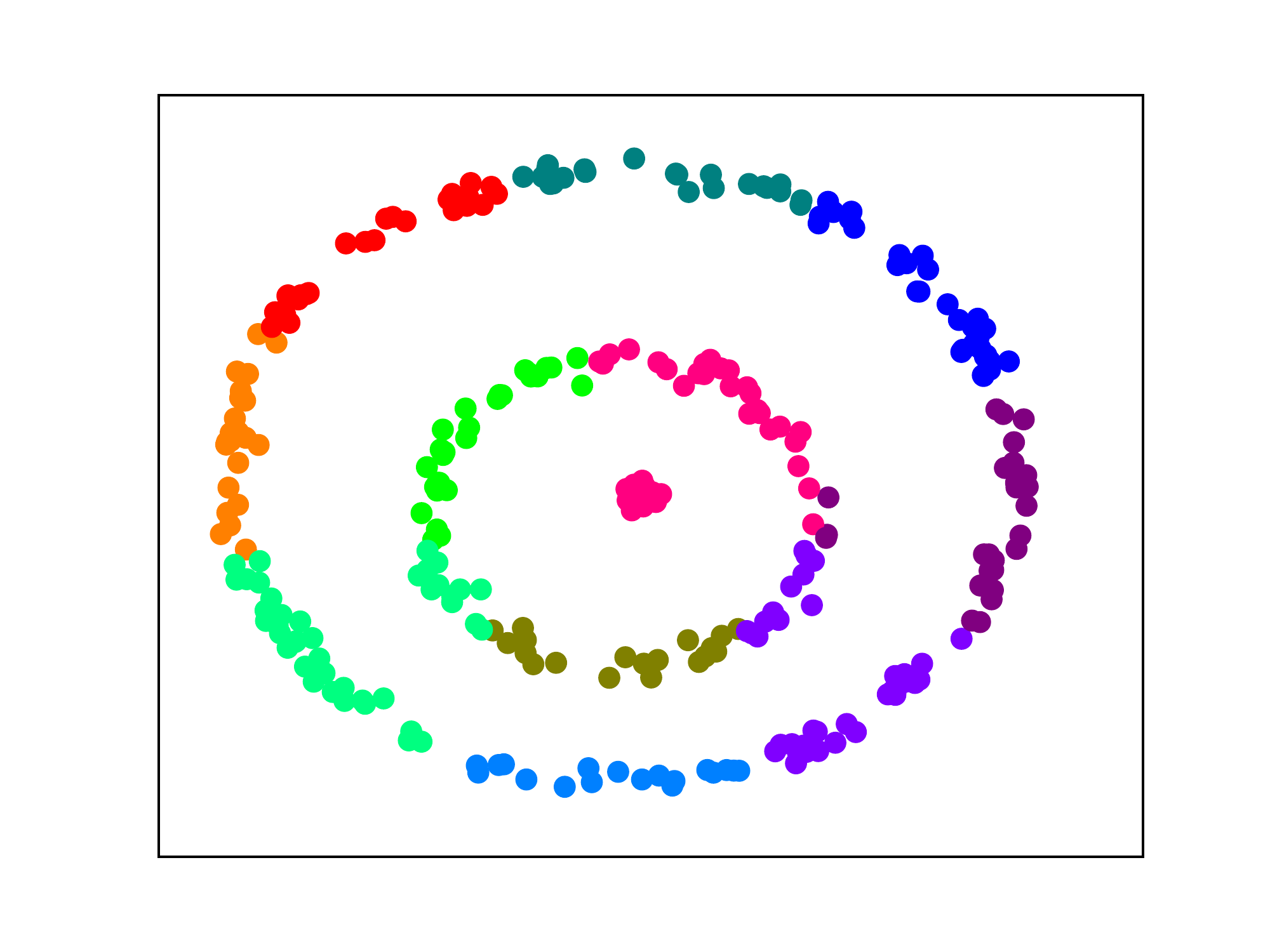}
  \end{subfigure}

  \caption{(\ref{fig:tunesdm-tqlamb} - \ref{fig:tunesdm-cput}): Accuracy
  contours and CPU time histogram for SD-Means on the synthetic ring data.
  (\ref{fig:sdmtrace} - \ref{fig:dmtrace}) An example 5 sequential frames from clustering
  using SD-Means (\ref{fig:sdmtrace}) and D-Means (\ref{fig:dmtrace}).
}\label{fig:SynthTestSDM}
\end{center}
\end{figure}

To illustrate the strengths of SD-Means with respect to D-Means, 
a second synthetic dataset of moving concentric rings was clustered in the same
$[0, 1] \times [0, 1]$ domain. At each
\begin{wrapfigure}[9]{r}{.2\textwidth}
  \centering 
  \captionsetup{font=scriptsize}
  \vspace*{-.07in}
  \includegraphics[width=.2\textwidth]{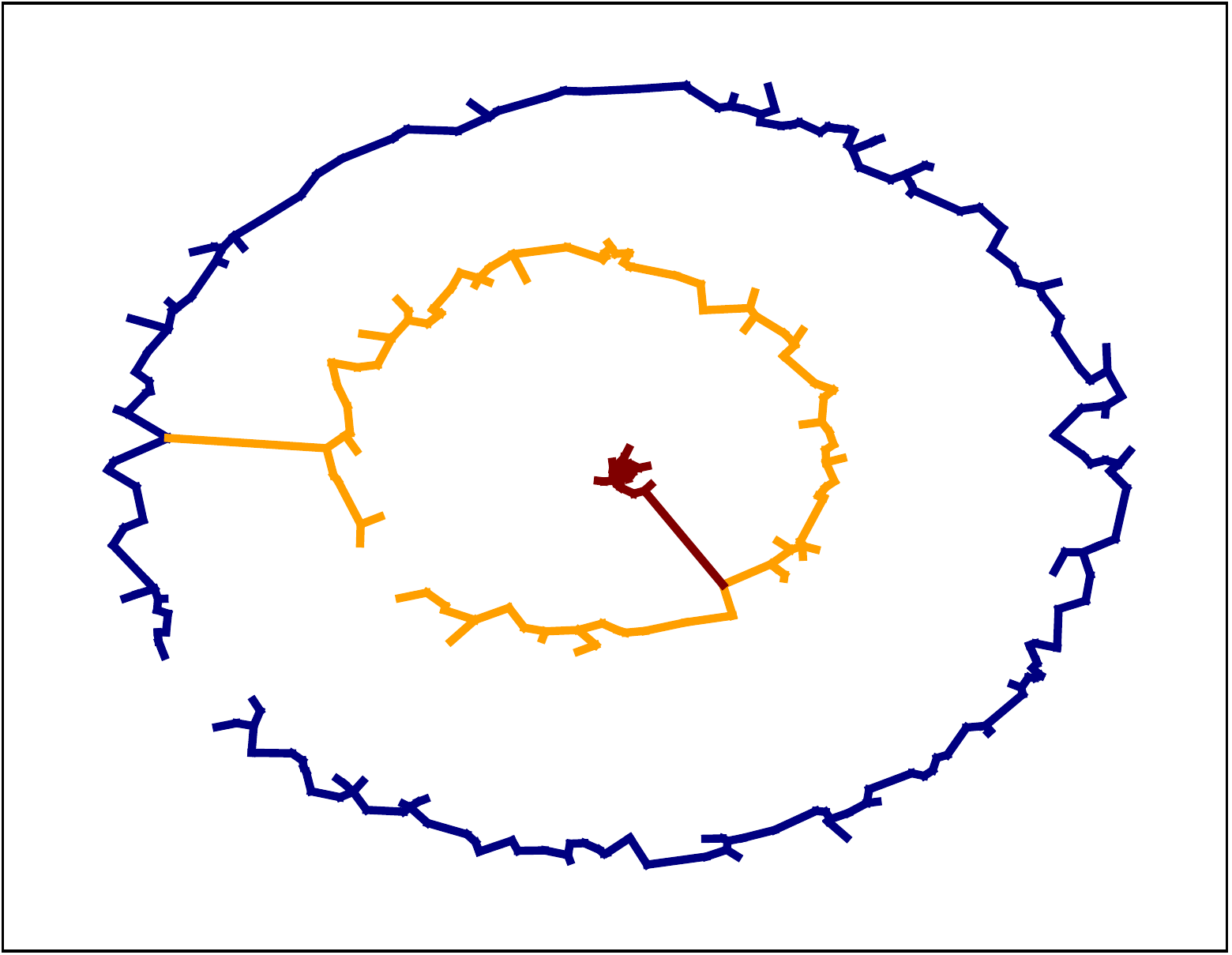}
  \caption{MST kernel to the center point.}\label{fig:ringkrn}
\end{wrapfigure}
  time step, 400 datapoints were generated uniformly
randomly on 3 rings of radius 0.4, 0.2, and 0.0 respectively,
with added isotropic Gaussian noise of standard deviation 0.03.
Between time steps, the three rings moved randomly, with displacements
sampled from an isotropic Gaussian of standard deviation 0.05.
The kernel function for SD-Means between two points $y_1, y_2$ was 
$\exp\left(-\frac{d^2}{2\omega^2}\right)$, where $\omega =0.07$ was a constant,
and $d$ was the sum of distances greater than $\omega$ along the
path connecting $y_1$ and $y_2$ through the minimum Euclidean spanning tree
of the dataset at each timestep. This kernel was used to capture long-range
similarity between points on the same ring. An example of the minimum spanning
tree and kernel function evaluation is shown in Figure \ref{fig:ringkrn}.
Finally, a similar parameter sweep to the previous experiment was used to find the 
best parameter settings for SD-Means ($\lambda = 55$, $T_Q = 13$, $k_\tau =
4.5$) and D-Means ($\lambda = 0.1$, $T_Q = 15$, $k_\tau = 1.1$).

Figure \ref{fig:SynthTestSDM} shows the results from this experiment averaged
over 50 trials. SD-Means
exhibits a similar robustness to its parameter settings as D-Means on the moving
Gaussian data, and is generally able to correctly cluster the moving rings. In contrast, 
D-Means is unable
to capture the rings, due to their nonlinear separability, and introduces many
erroneous clusters. Further, there is a significant computational cost to pay for the
\begin{wraptable}[5]{r}{.4\textwidth}
  \centering \vspace*{-.1in}
  \captionsetup{font=scriptsize}
  \caption{Mean computational time \& accuracy on synthetic ring
  data over 50 trials.}\label{tab:ringresults}
  {\small
  \begin{tabular}{l|c|l}
    \textbf{Alg.} & \textbf{\% Acc.} & \textbf{Time (s)}\\
    \hline
\rule{0pt}{10pt}SD-Means & $75.1$ & $2.9$\\ 
D-Means &  $18.7$& $8.4\times 10^{-4}$\\ 
  \end{tabular}
}
\end{wraptable}
flexibility of SD-Means, as expected. Table \ref{tab:ringresults} corroborates
these observations with numerical data. Thus, in practice, D-Means should be
the preferred algorithm unless the cluster structure of the data is not linearly
separable, in which case the extra flexibility of SD-Means is required.

\subsection{Synthetic Gaussian Processes} 
\afterpage{
\begin{figure}[t]
  \centering
  \captionsetup{font=scriptsize}
  \begin{subfigure}[b]{0.16\columnwidth}
      \centering
      \includegraphics[width=\columnwidth]{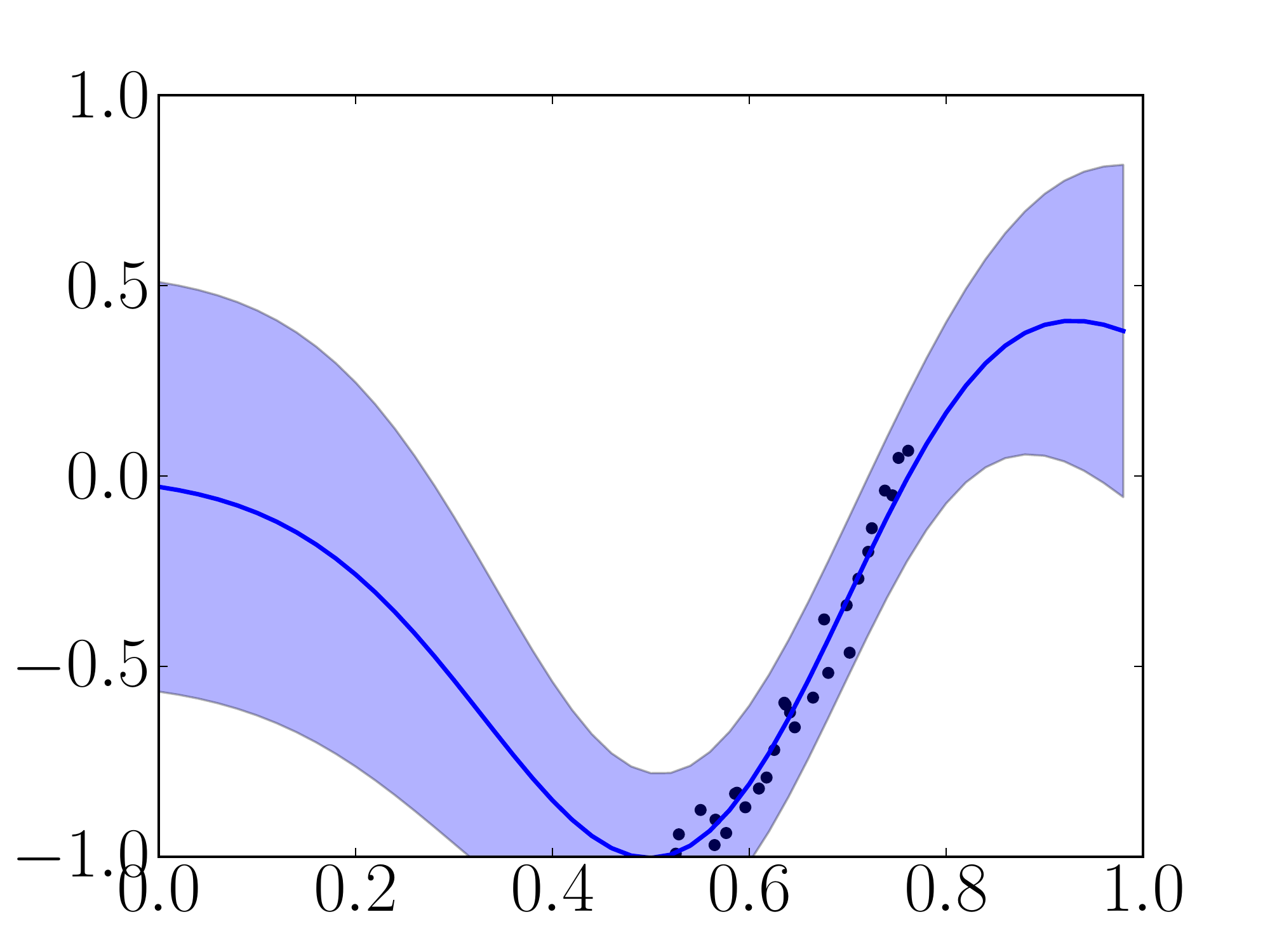}
    \end{subfigure}
\begin{subfigure}[b]{0.16\columnwidth}
      \centering
      \includegraphics[width=\columnwidth]{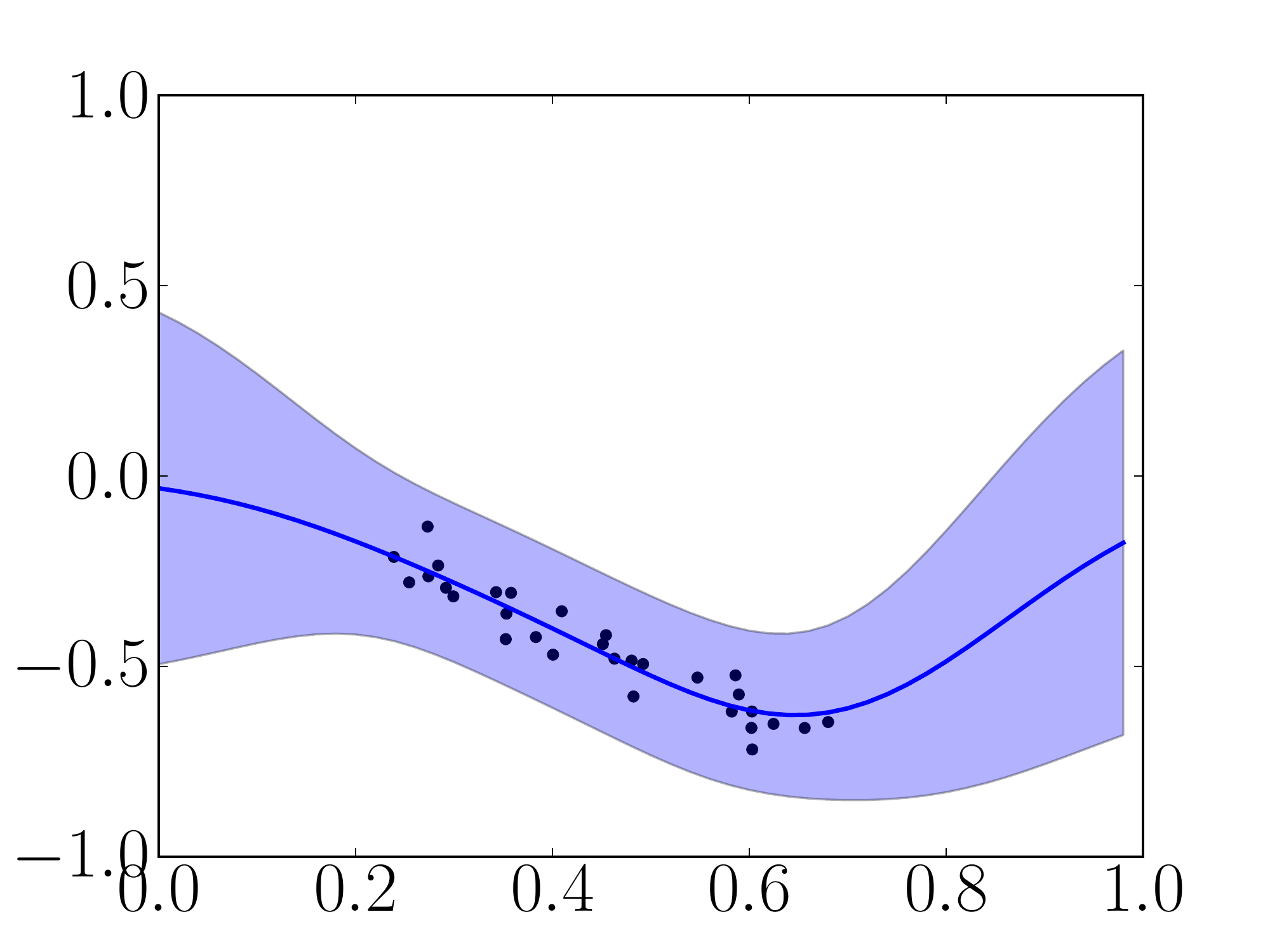}
    \end{subfigure}
\begin{subfigure}[b]{0.16\columnwidth}
      \centering
      \includegraphics[width=\columnwidth]{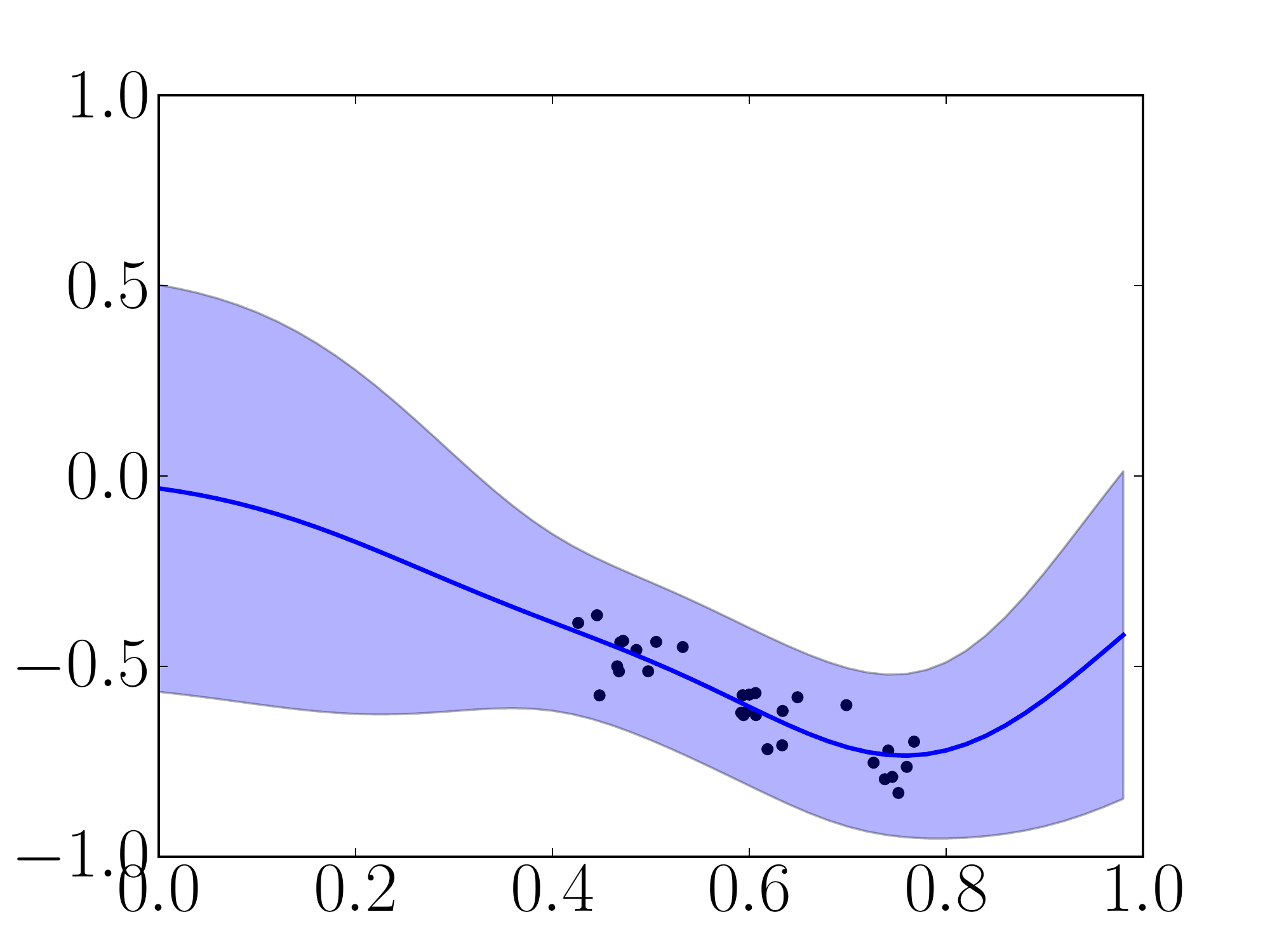}
    \end{subfigure}
\begin{subfigure}[b]{0.16\columnwidth}
      \centering
      \includegraphics[width=\columnwidth]{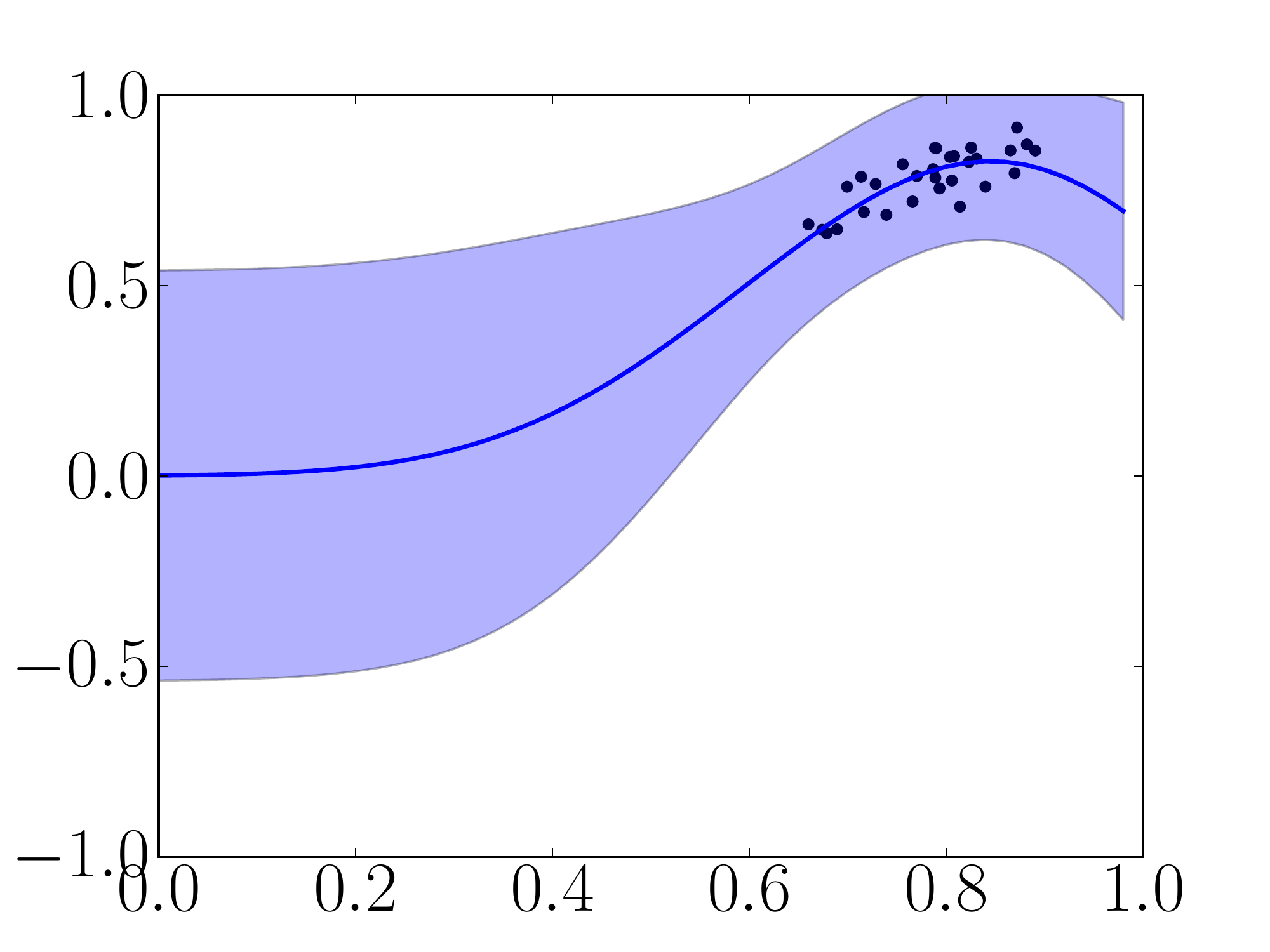}
    \end{subfigure}
\begin{subfigure}[b]{0.16\columnwidth}
      \centering
      \includegraphics[width=\columnwidth]{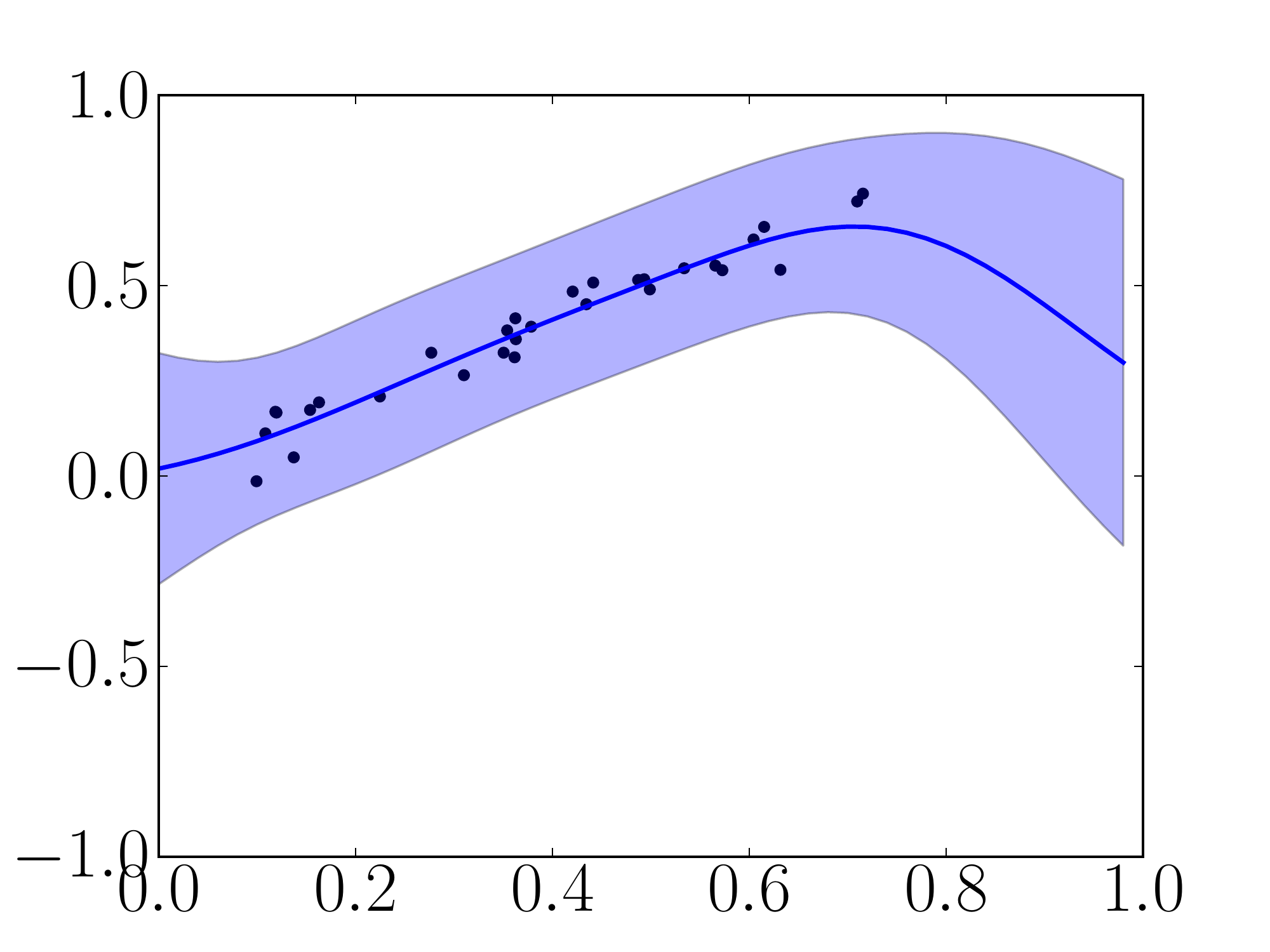}
    \end{subfigure}
\begin{subfigure}[b]{0.16\columnwidth}
      \centering
      \includegraphics[width=\columnwidth]{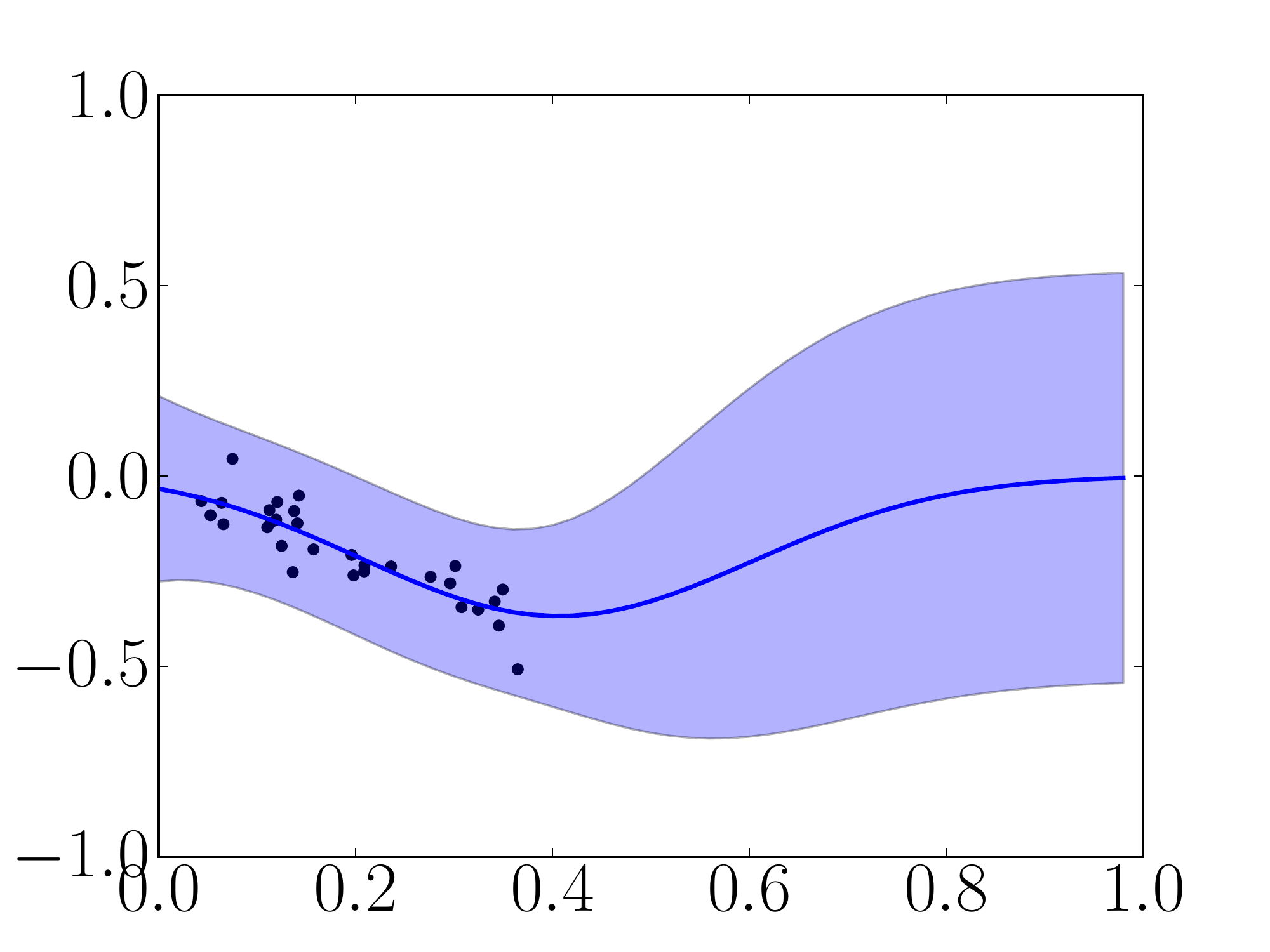}
    \end{subfigure}
    \caption{Six example GPs to be clustered}\label{fig:datagps}
\end{figure}
\begin{figure}[t]
  \centering
  \captionsetup{font=scriptsize}
  \begin{subfigure}[b]{0.16\columnwidth}
      \centering
      \includegraphics[width=\columnwidth]{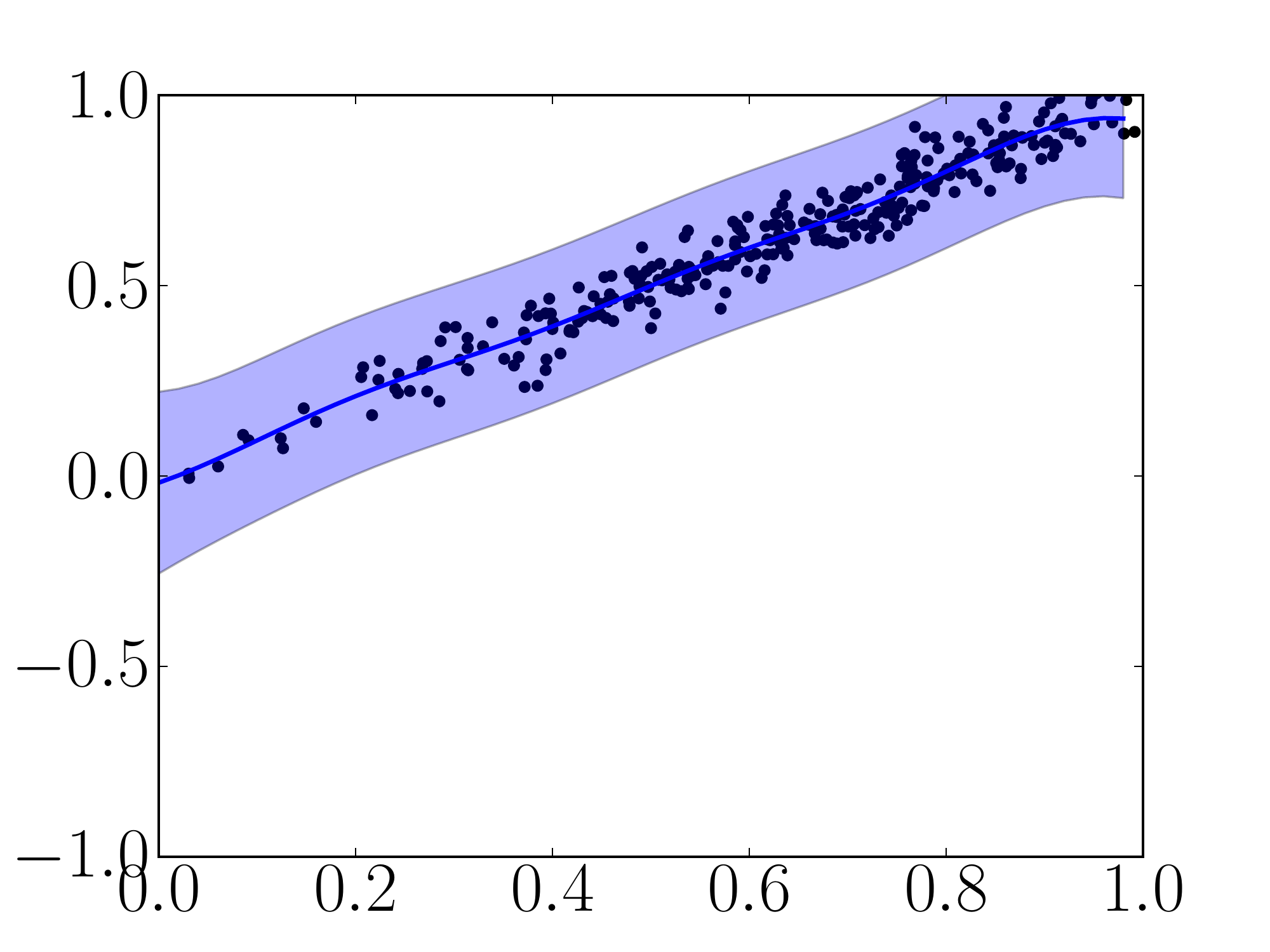}
    \end{subfigure}
\begin{subfigure}[b]{0.16\columnwidth}
      \centering
      \includegraphics[width=\columnwidth]{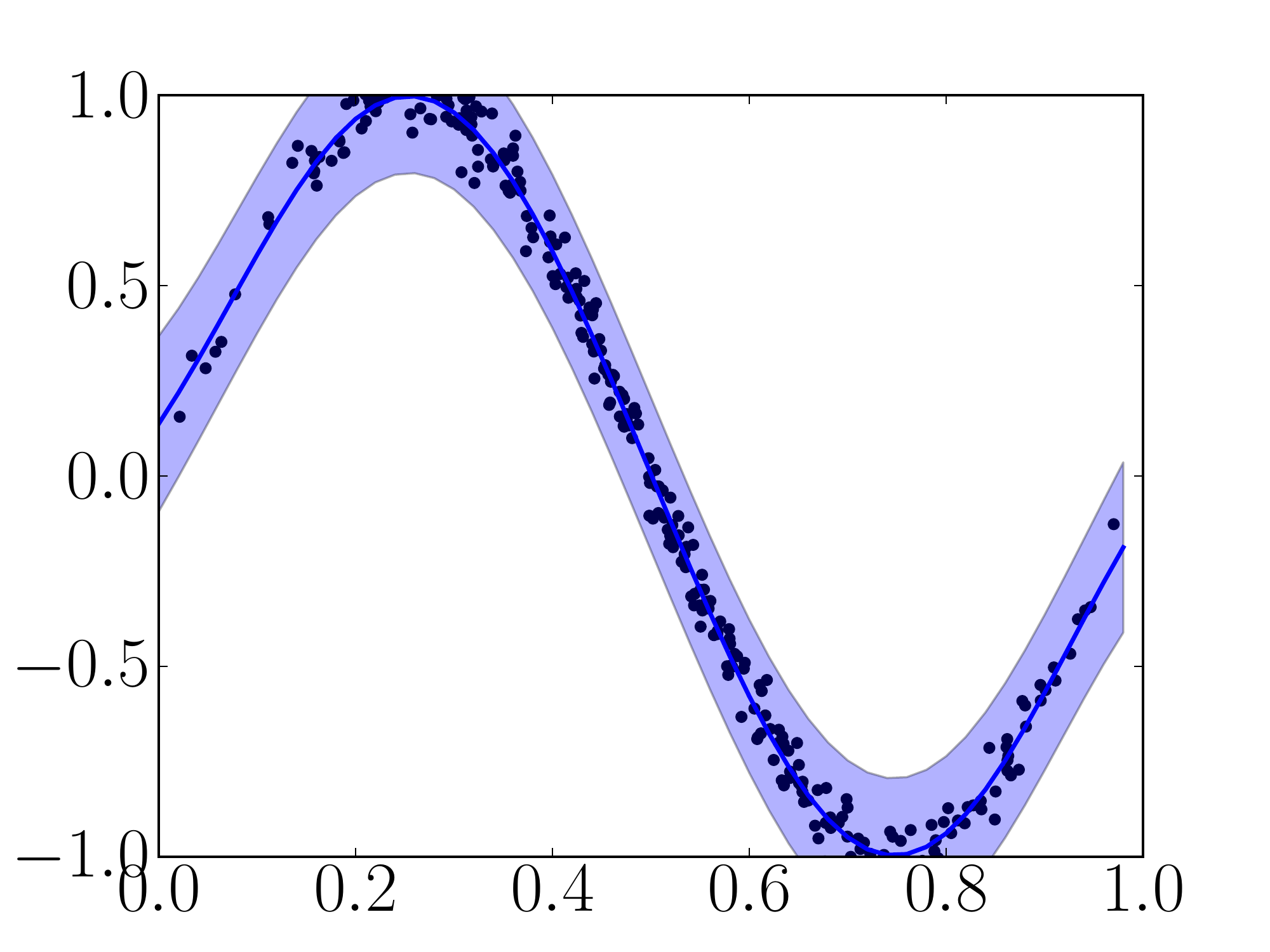}
    \end{subfigure}
\begin{subfigure}[b]{0.16\columnwidth}
      \centering
      \includegraphics[width=\columnwidth]{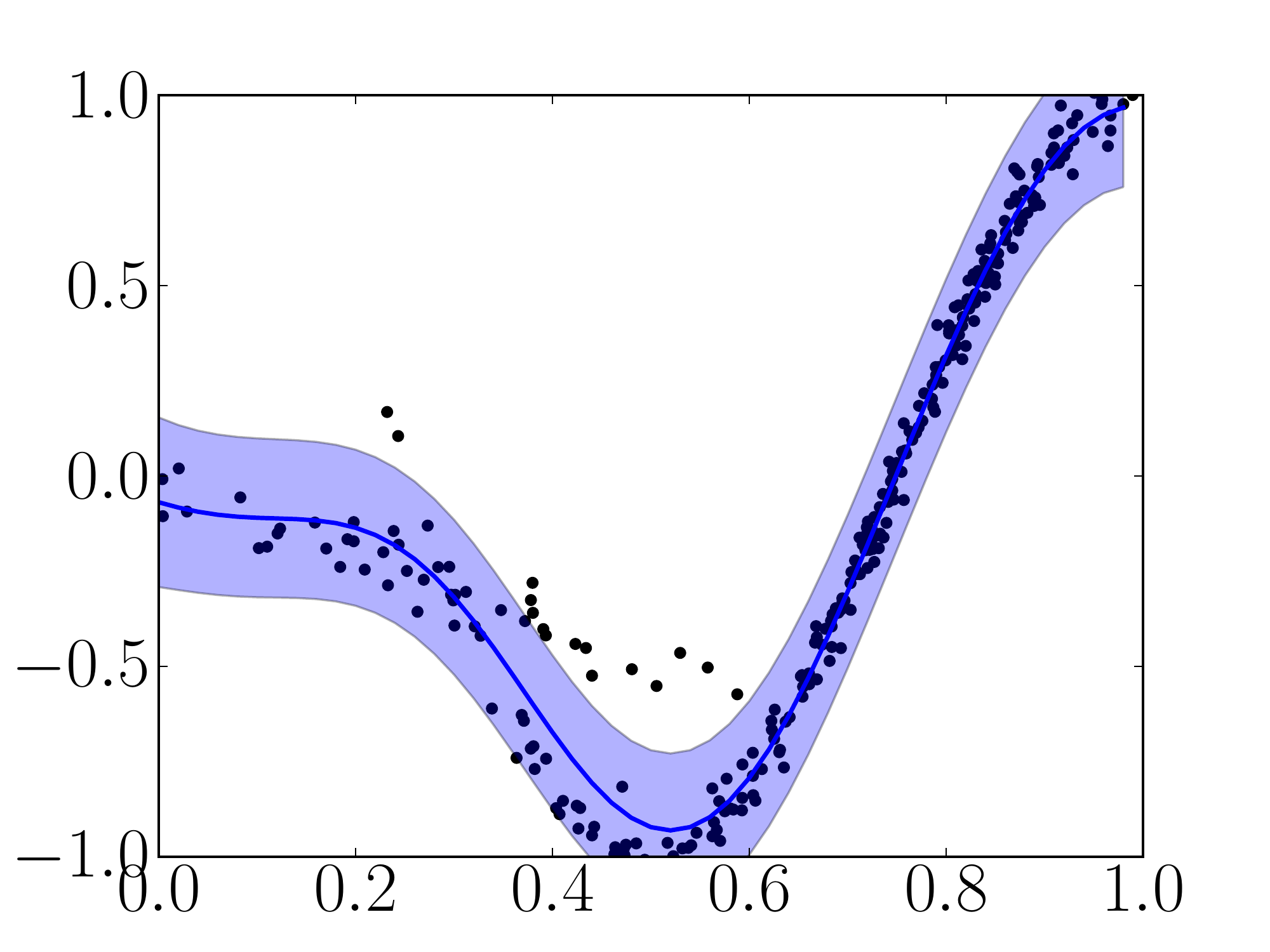}
    \end{subfigure}
\begin{subfigure}[b]{0.16\columnwidth}
      \centering
      \includegraphics[width=\columnwidth]{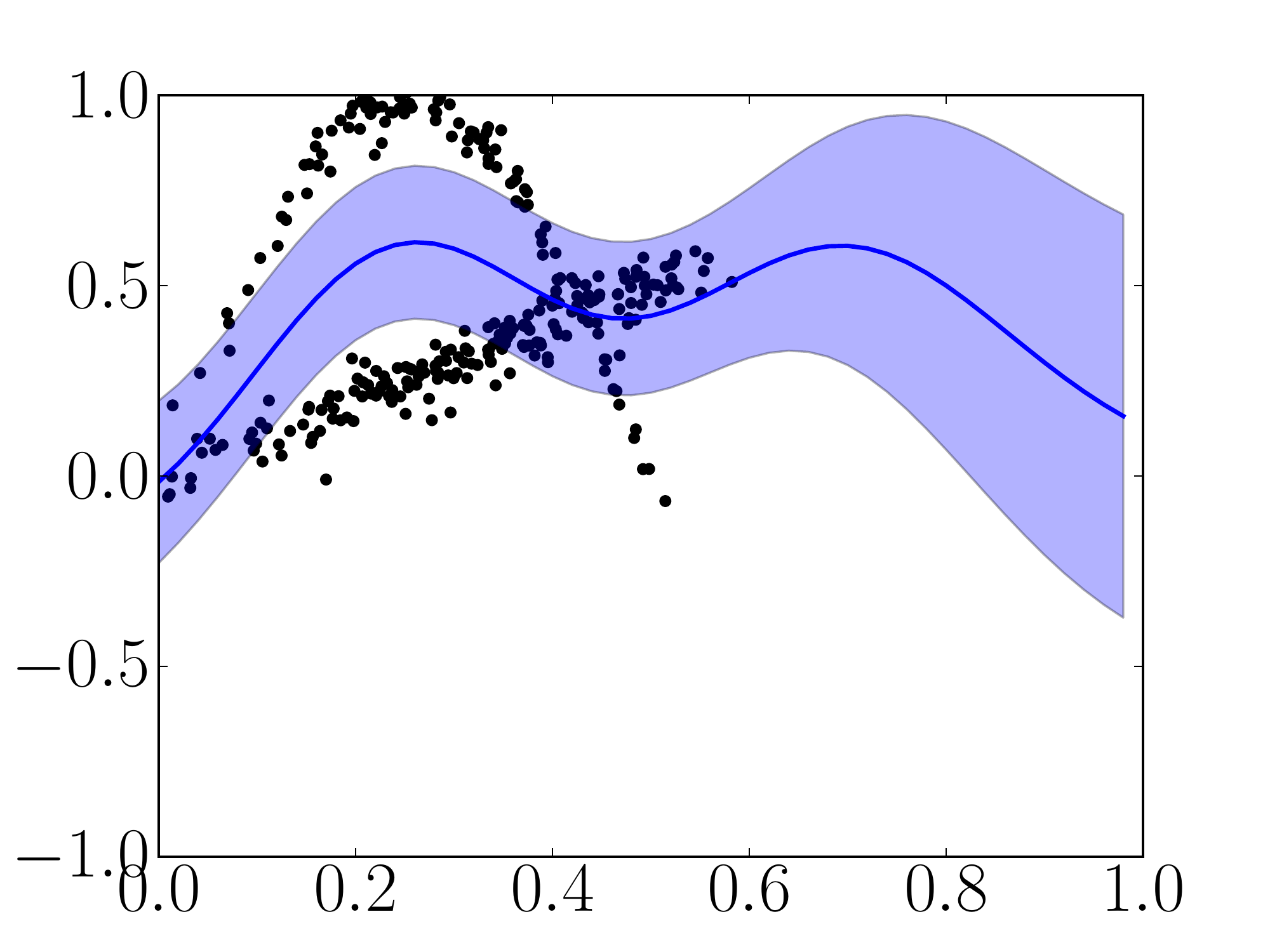}
    \end{subfigure}
\begin{subfigure}[b]{0.16\columnwidth}
      \centering
      \includegraphics[width=\columnwidth]{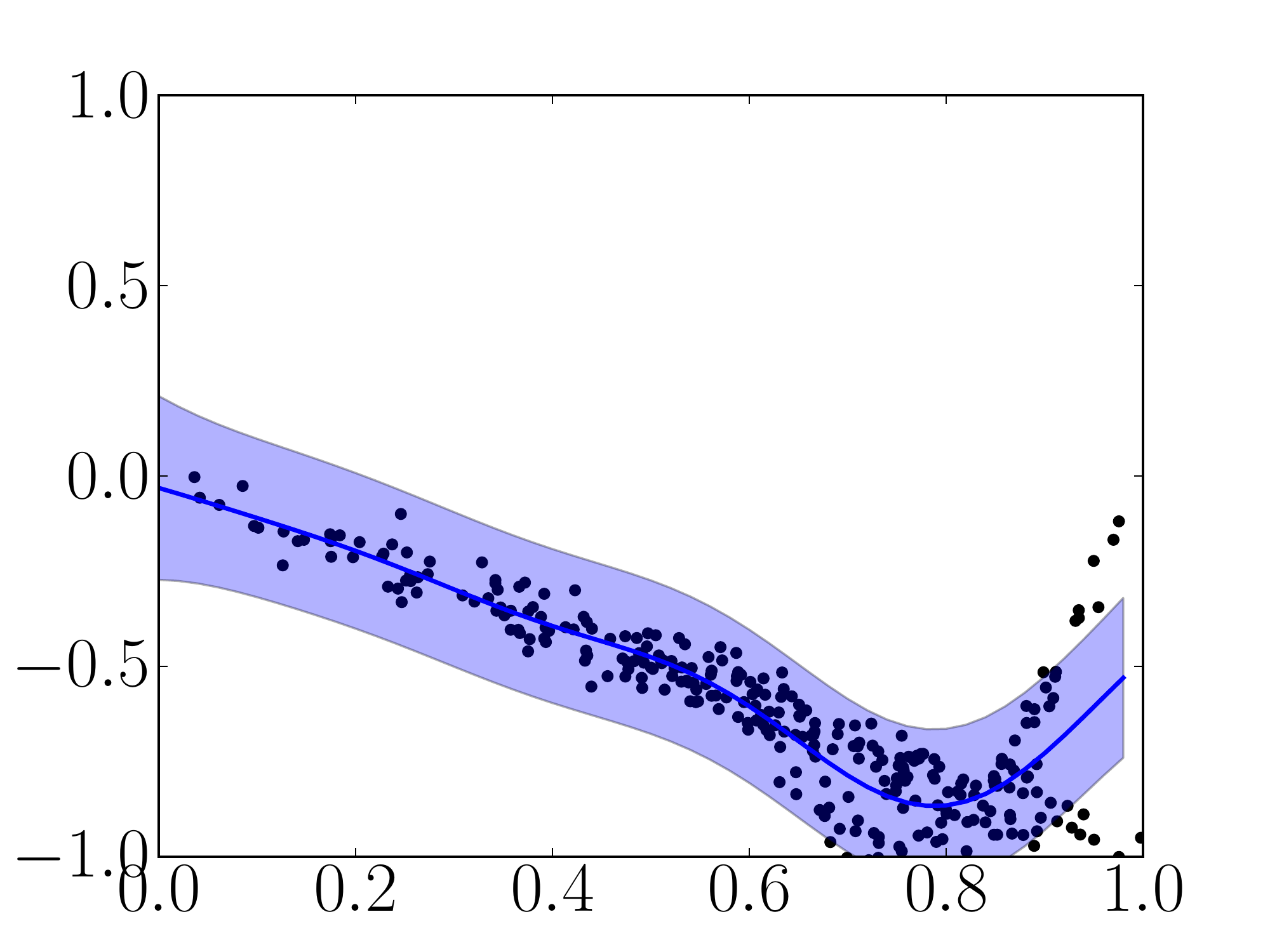}
    \end{subfigure}
\begin{subfigure}[b]{0.16\columnwidth}
      \centering
      \includegraphics[width=\columnwidth]{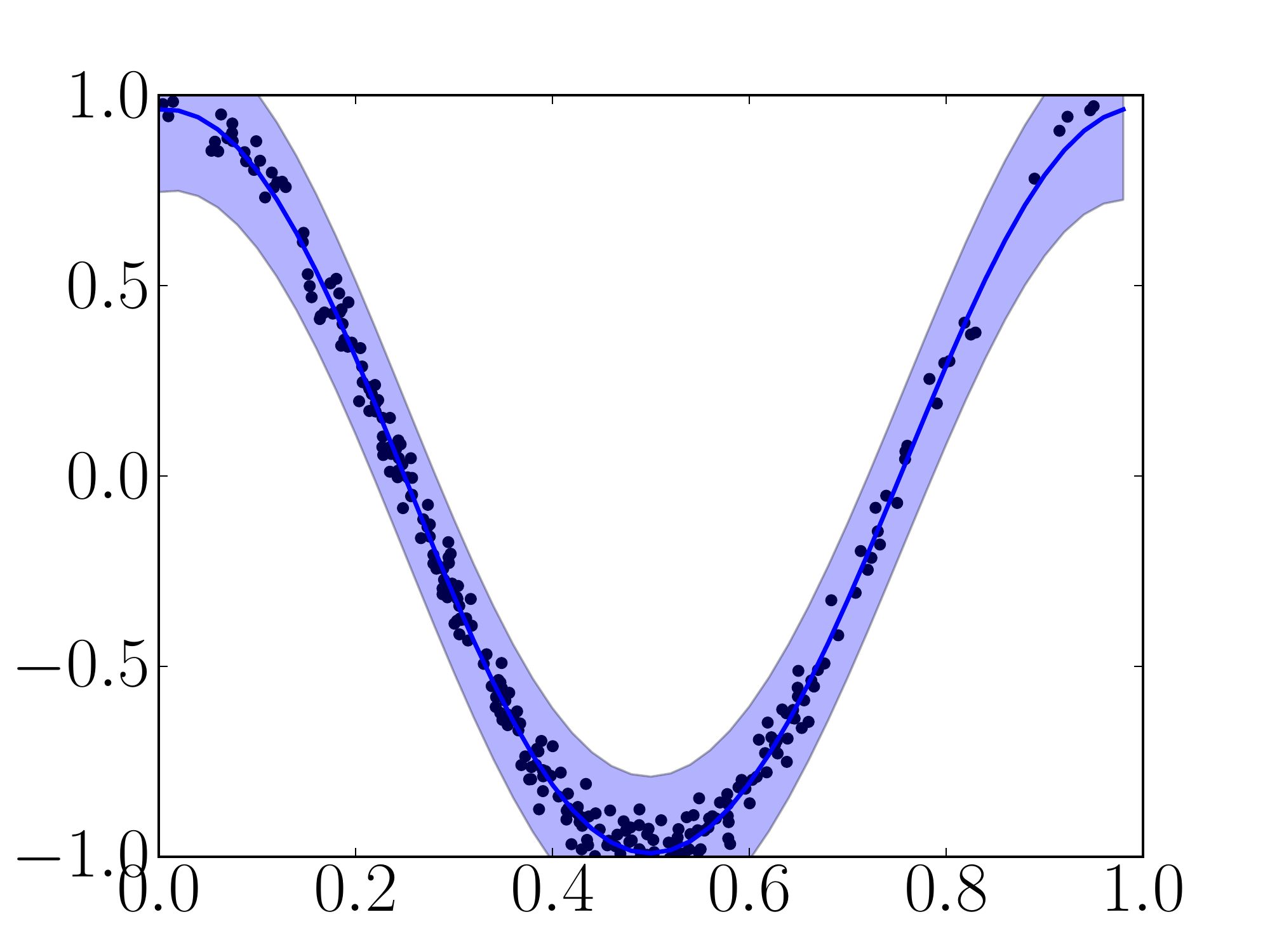}
    \end{subfigure}
    \caption{Typical set of latent functions discovered by D-Means}\label{fig:dmgps}
\end{figure}
\begin{figure}[H]
  \centering
  \captionsetup{font=scriptsize}
  \begin{subfigure}[b]{0.23\columnwidth}
      \centering
      \includegraphics[width=\columnwidth]{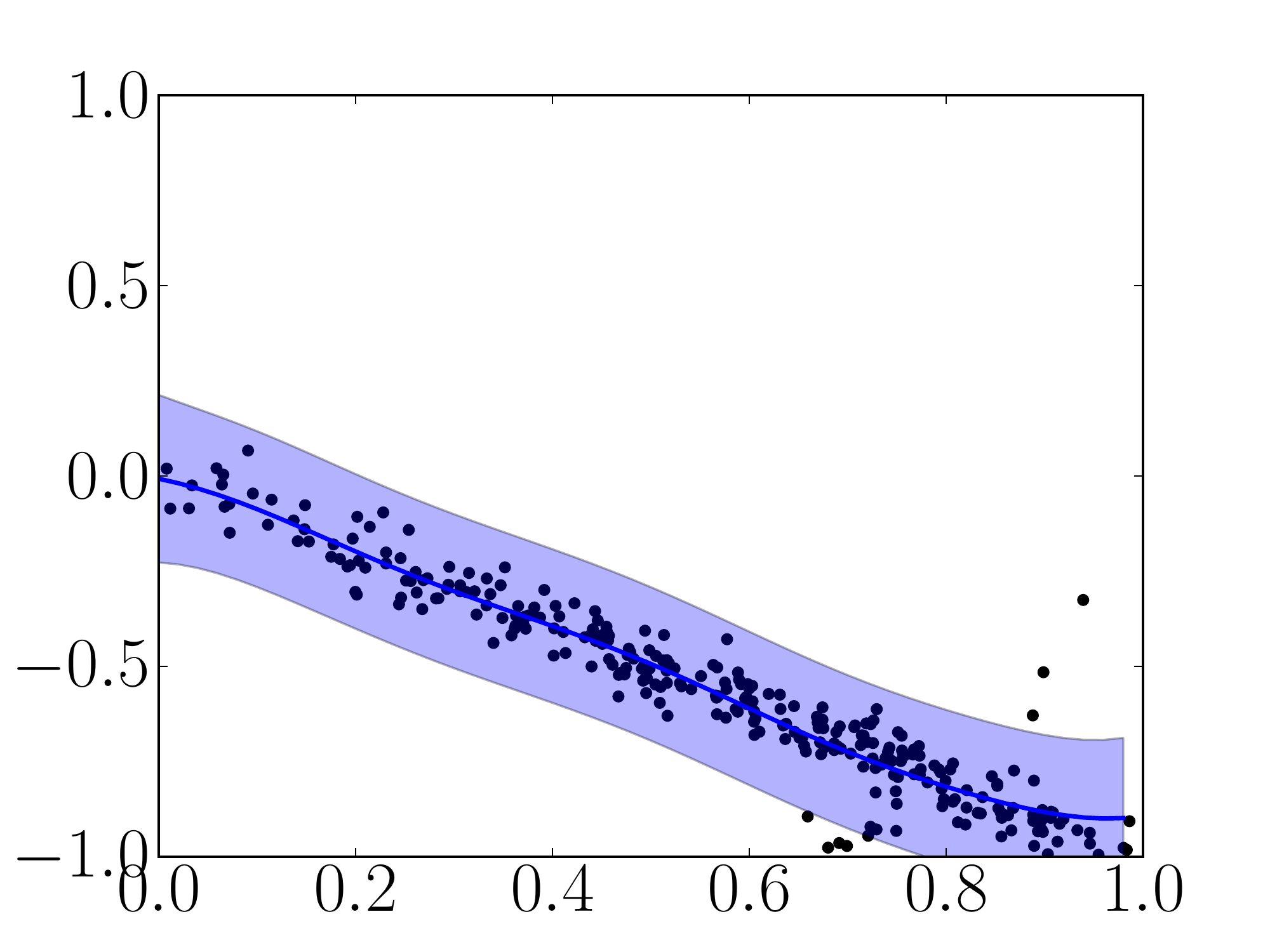}
    \end{subfigure}
\begin{subfigure}[b]{0.23\columnwidth}
      \centering
      \includegraphics[width=\columnwidth]{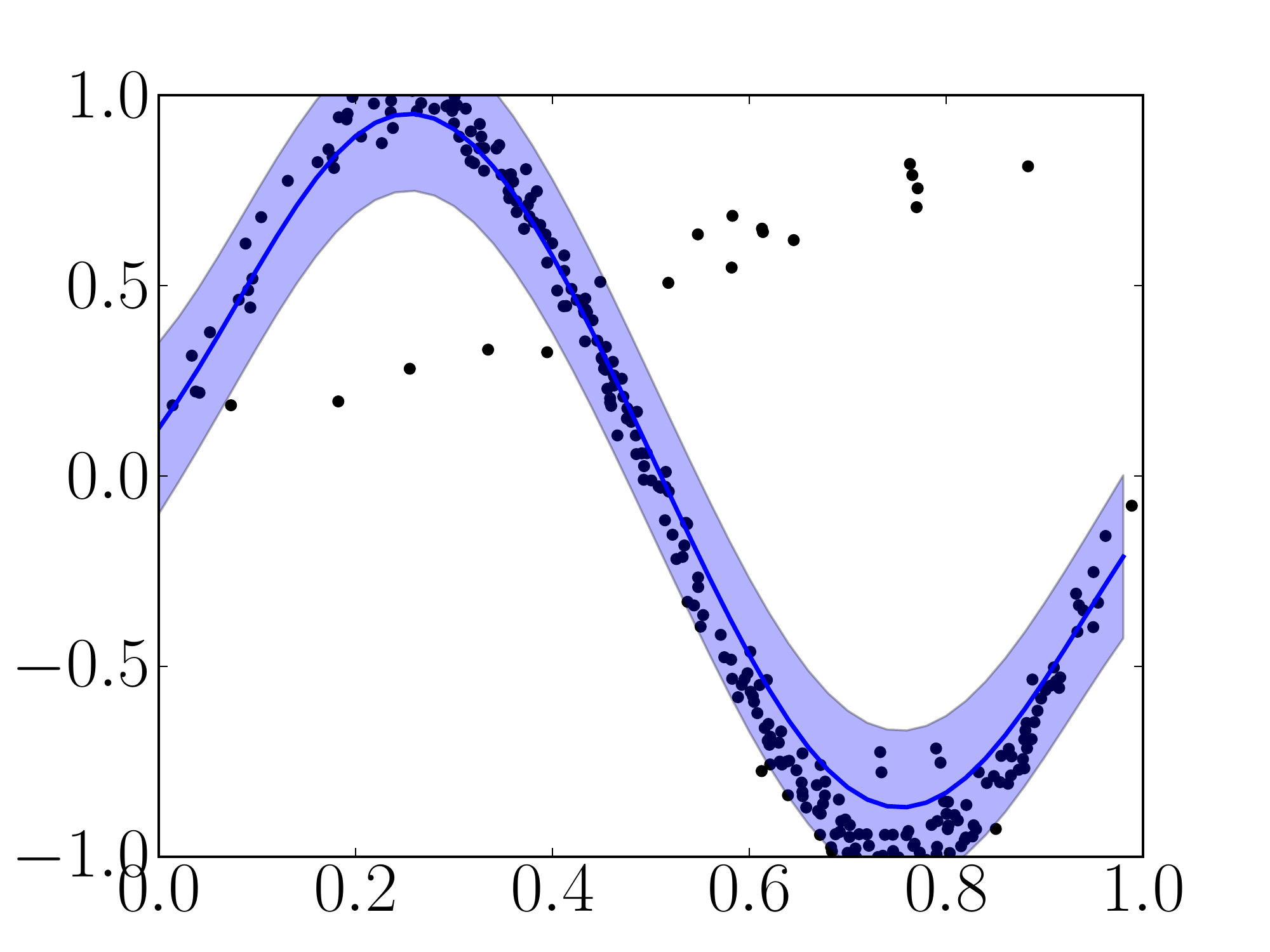}
    \end{subfigure}
\begin{subfigure}[b]{0.23\columnwidth}
      \centering
      \includegraphics[width=\columnwidth]{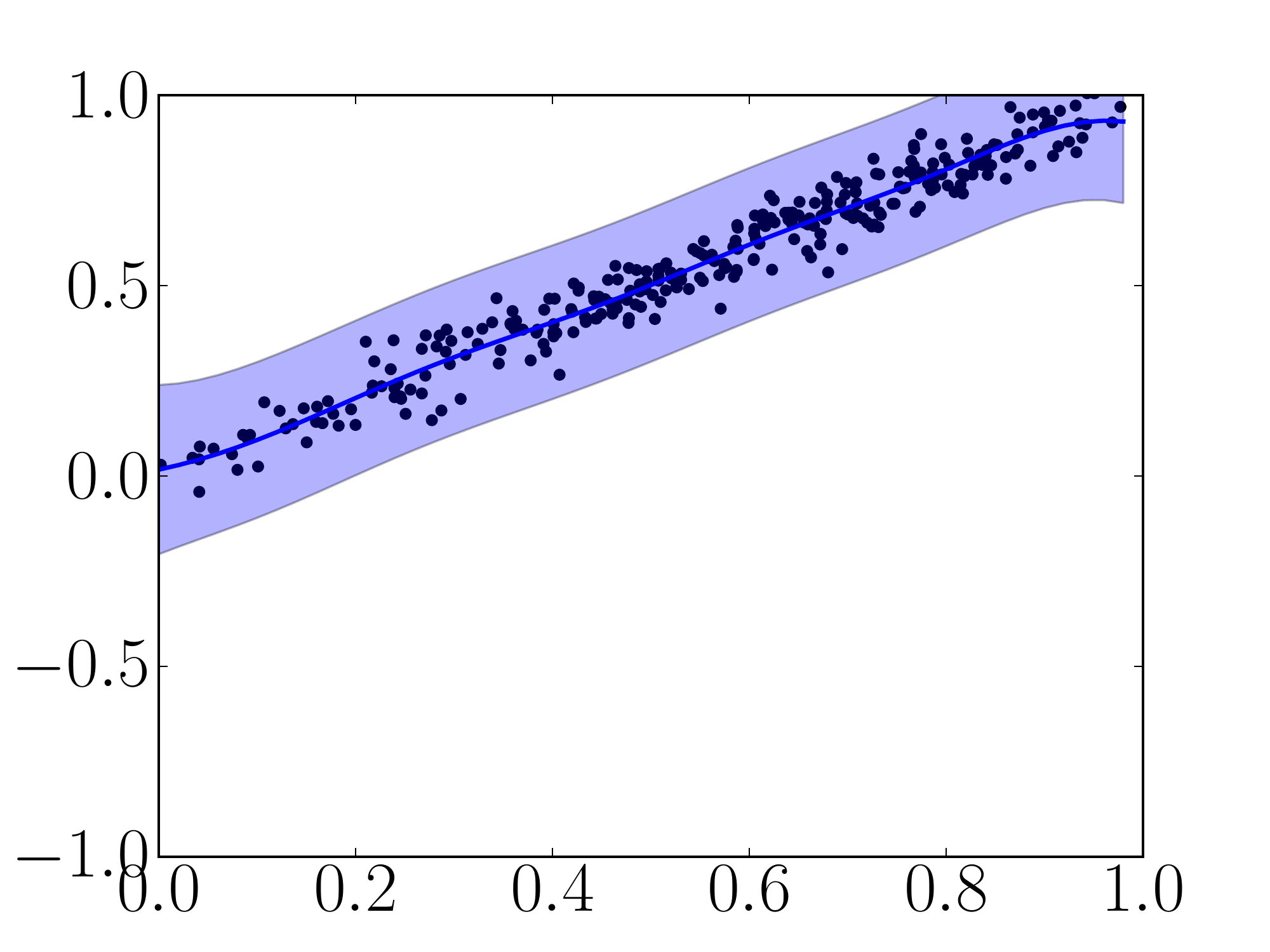}
    \end{subfigure}
\begin{subfigure}[b]{0.23\columnwidth}
      \centering
      \includegraphics[width=\columnwidth]{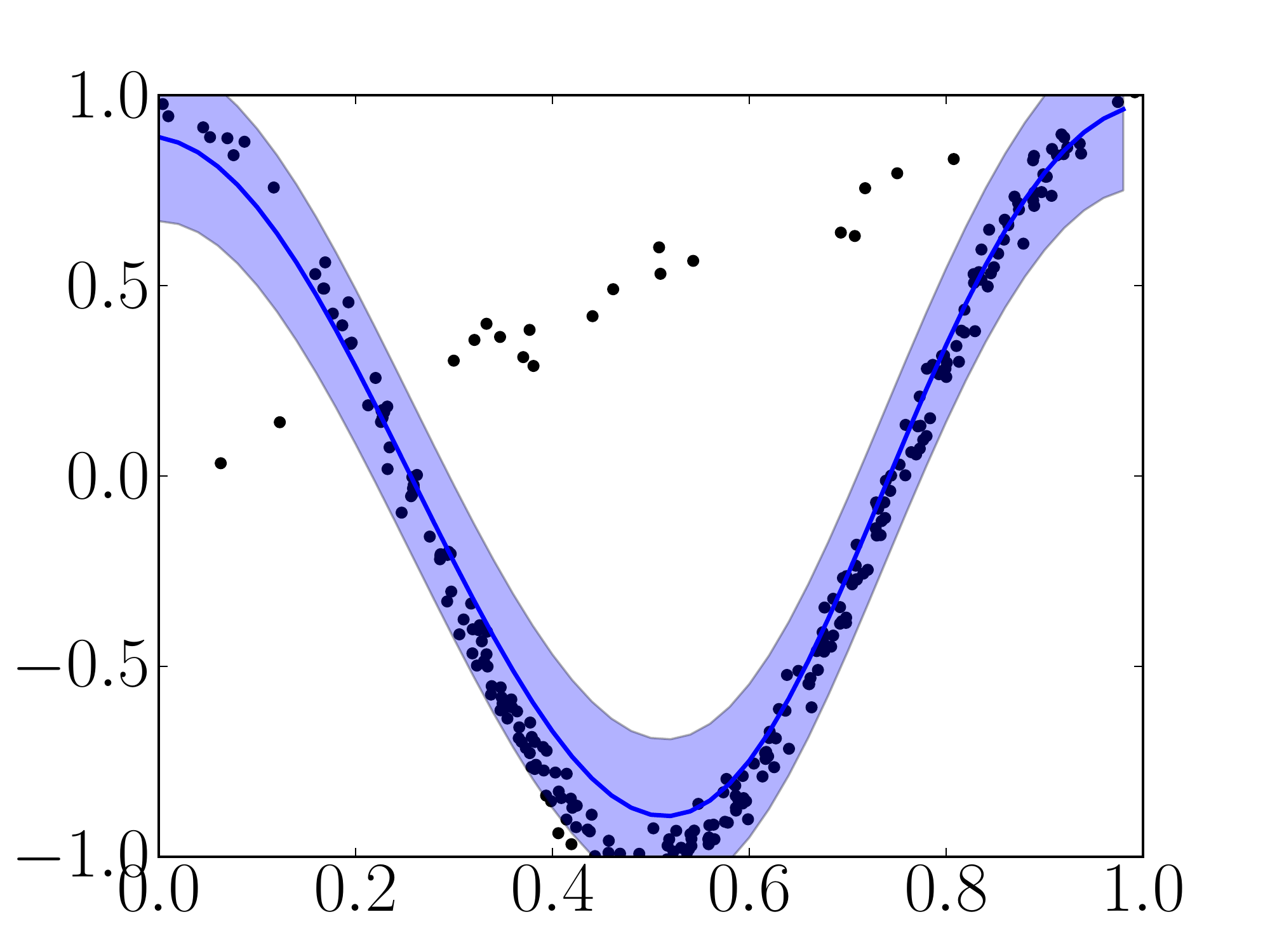}
    \end{subfigure}
    \caption{Typical set of latent functions discovered by SD-Means}\label{fig:sdmgps}
\end{figure}
\begin{figure}[H]
  \captionsetup{font=scriptsize}
  \centering
  \begin{subfigure}[b]{0.3\columnwidth}
      \centering
      \includegraphics[width=\columnwidth]{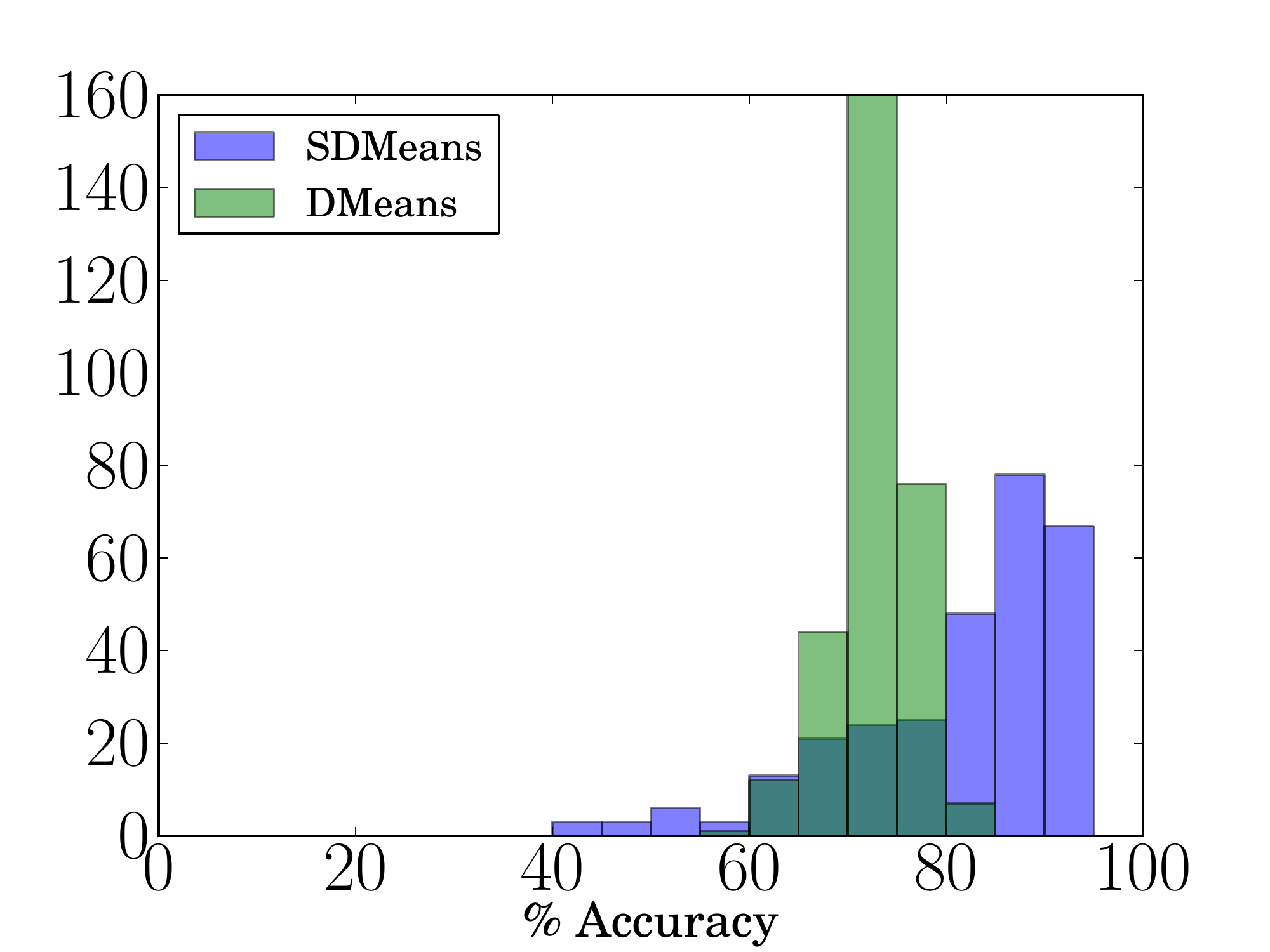}
      \caption{}\label{fig:gpacc}
    \end{subfigure}
\begin{subfigure}[b]{0.3\columnwidth}
      \centering
      \includegraphics[width=\columnwidth]{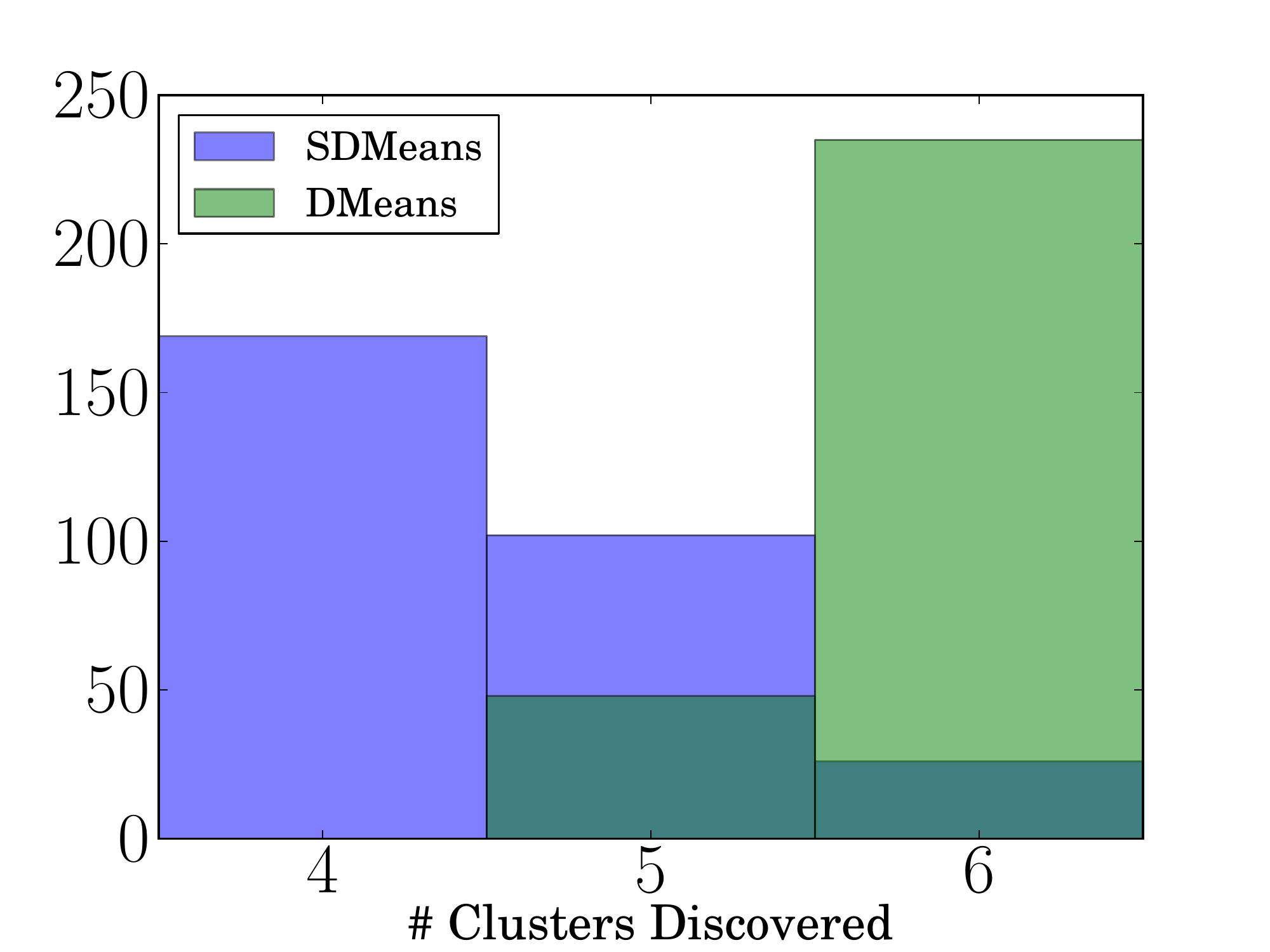}
      \caption{}\label{fig:gpnclus}
    \end{subfigure}
  \begin{subfigure}[b]{0.3\columnwidth}
      \includegraphics[width=\columnwidth]{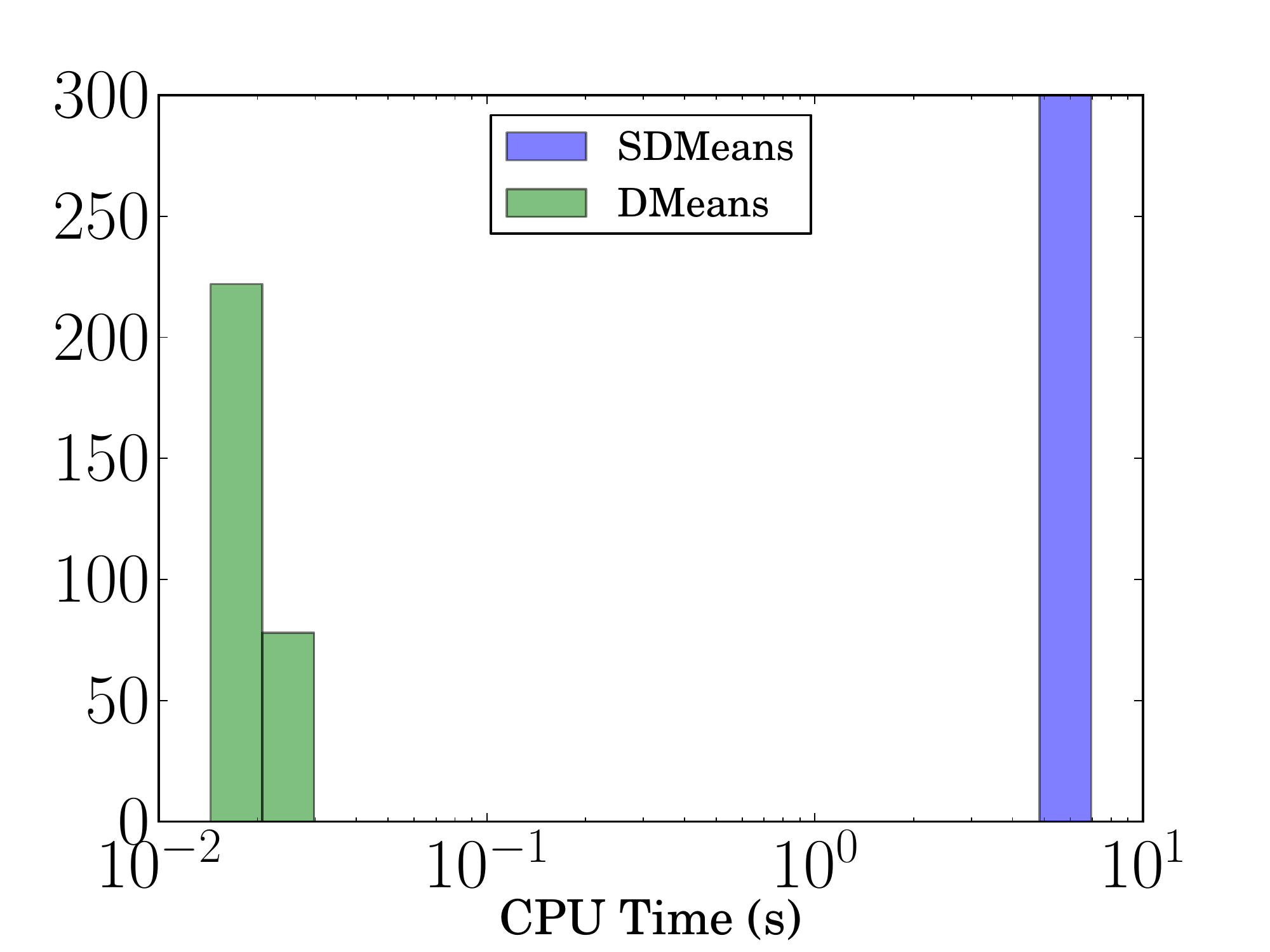}
      \caption{}\label{fig:gptime}
    \end{subfigure}
    \caption{(\ref{fig:gpacc}): Histogram of label accuracy for D-Means and SD-Means. (\ref{fig:gpnclus}):
  Histogram of the number of clusters learned for D-Means and SD-Means.
  (\ref{fig:gptime}): Histogram of computational time for D-Means and SD-Means.}\label{fig:gpaccnclus}
\end{figure}
}

Clustering Gaussian processes (GPs) is problematic for many inference algorithms 
for a number of reasons. Sampling algorithms that update cluster parameters (such as
DP-GP Gibbs sampling~\citep{Joseph11_AR}) require
the inversion of kernel matrices for each sample, which is computationally intractable for large datasets.
Many other algorithms, such as D-Means, require discretization of the GP mean
function in order to compute distances, which incurs an exponential growth in
computational complexity with the input dimension of the GP, 
increased susceptibility to local optima, and incapability
to capture uncertainty in the areas of the GP with few measurements.
SD-Means, on the other hand, can capture the GP uncertainty through
the kernel function, and since cluster parameters are not instantiated
while clustering each batch of data it does not incur the cost of repeated GP
regression.

This experiment involved clustering noisy sets of observations
from an unknown number $K$ of latent functions $f_k(\cdot)$. The uncertainty
in the functions $f_k(\cdot)$ was captured by modelling them as GPs. 
The data were generated from $K=4$ functions: $f_1(x) = x$, $f_2(x) = -x$, $f_3(x) =
\sin(x)$, and $f_4(x) = \cos (x)$. For each data point to be clustered,
a random label $k \in \{1, 2, 3, 4\}$ was sampled uniformly, along
with a random interval in $[0, 1]$. Then 30 $x, y$ pairs
were generated within the interval with uniformly random
$x$ and $y = f_k(x) + \epsilon$, $\epsilon \sim \mathcal{N}(0, 0.2^2)$.
These $x, y$ pairs were used to train a Gaussian process,
which formed a single datapoint in the clustering procedure.
Measurements were made in a random interval to ensure that large
areas of the domain of each GP had a high uncertainty, in order to demonstrate
the robustness of SD-Means to GP uncertainty.
Example GPs generated from this procedure, with $1\sigma$ measurement
confidence bounds, are shown in Figure \ref{fig:datagps}.
D-Means and SD-Means were used to cluster a dataset of 1000 such GPs, broken
into batches of 100 GPs. For D-Means, the GP mean function was discretized
on $[0, 1]$ with a grid spacing of $0.02$, and the discretized mean functions
were clustered. For SD-Means, the kernel function between two GPs $i$ and $j$ was
the probability that their respective kernel points were generated from the same latent
GP (with a prior probability of 0.5),
\begin{align}
  \krn{(X_i, Y_i)}{(X_j, Y_j)} &= \frac{\mathrm{GP}(Y_i, Y_j | X_i,
X_j)}{\mathrm{GP}(Y_i, Y_j |
X_i, X_j) + \mathrm{GP}(Y_i|X_i)\mathrm{GP}(Y_j|X_j)}
\end{align}
where $X_i, Y_i$ are the domain and range of the kernel points for GP $i$, 
and $\mathrm{GP}\left(\cdot | \cdot\right)$ is the Gaussian process marginal likelihood
function. 

The results from this exercise are shown in Figures \ref{fig:dmgps} -
\ref{fig:gpaccnclus}. Figures \ref{fig:dmgps} and \ref{fig:sdmgps} demonstrate
that SD-Means tends to discover the 4 latent functions used to generate the
data, while D-Means often introduces erroneous clusters. Figure \ref{fig:gpaccnclus}
corroborates the qualitative results with a quantitative comparison over 300
trials, demonstrating that while D-Means is roughly 2 orders of magnitude
faster, SD-Means generally outperforms D-Means in label accuracy and the selection of the number of
clusters. This is primarily due to the fact that D-Means clusters data using
only the GP means; as the GP mean function can be a 
poor characterization of the latent function (shown in Figure
\ref{fig:datagps}), D-Means can perform poorly. SD-Means, on the other hand,
uses a kernel function that accounts for regions of high uncertainty in the
GPs, and thus is generally more robust to such uncertainty.

\subsection{Aircraft ADSB Trajectories}
\begin{figure}[t]
\captionsetup{font=scriptsize}
\begin{center}
  \begin{subfigure}[c]{.4\textwidth}
  \includegraphics[width=\textwidth, trim=60 210 80 60,
  clip]{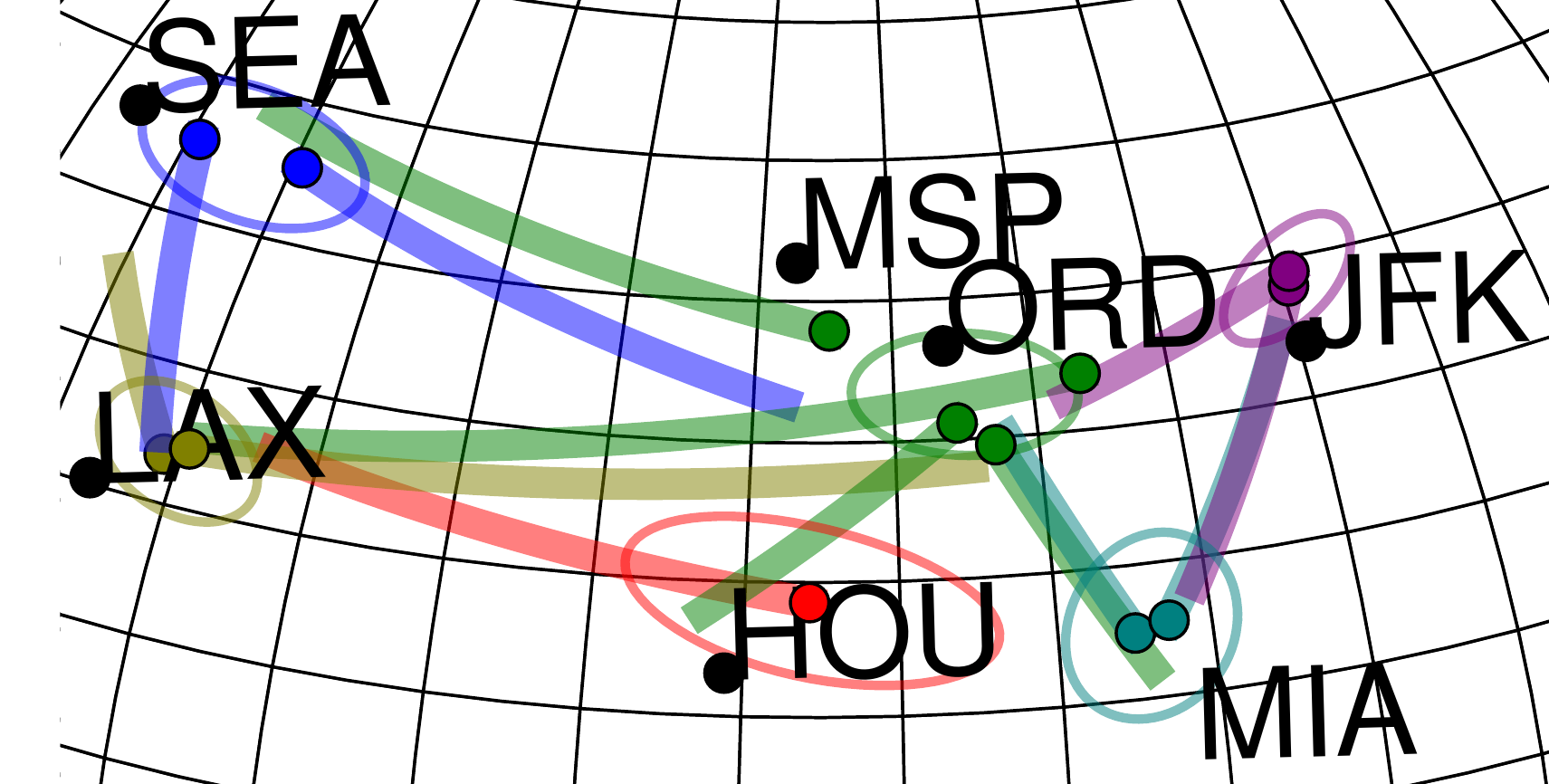}
  \caption{}\label{fig:adsbflows}
\end{subfigure}
  \begin{subfigure}[c]{.9\textwidth}
    \includegraphics[width=\textwidth]{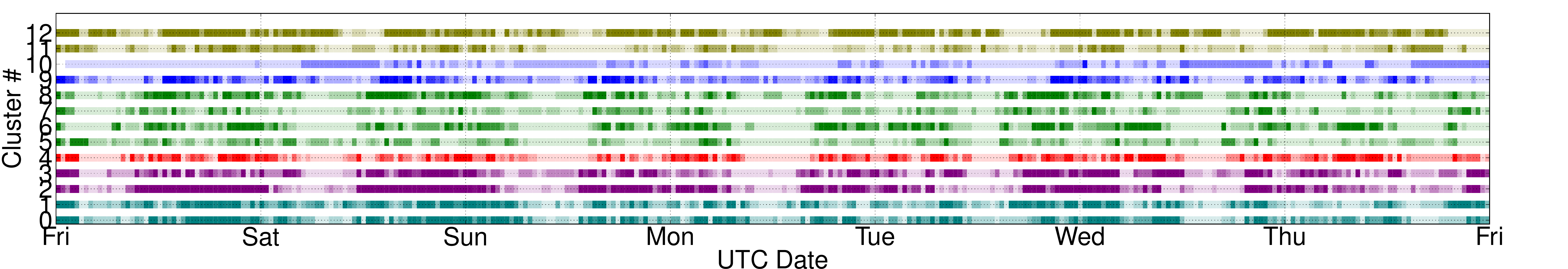}
    \caption{}\label{fig:adsbcounttrack}
  \end{subfigure}
  \caption{Results of the GP aircraft trajectory clustering.
    (\ref{fig:adsbflows}): A map
    (labeled with major US city airports) showing the overall aircraft 
    flows for 12 trajectories, with colors and 1$\sigma$ confidence ellipses corresponding to takeoff
    region (multiple clusters per takeoff region), colored dots indicating
    mean takeoff position for each cluster, and lines indicating the mean
    trajectory for each cluster.
    (\ref{fig:adsbcounttrack}): A track of plane counts for the 12 clusters during the week, with color intensity proportional to
  the number of takeoffs at each time.}\label{fig:trajresults}
  \end{center}
\end{figure}



 In this experiment, the D-Means algorithm was 
used to discover the typical spatial and temporal patterns in the motions of
commercial aircraft. The experiment was designed to demonstrate the capability of D-Means
to find spatial and temporal patterns in a streaming source of real world
data over a long duration. 

The data was collected and processed using the following procedure.
Automatic dependent surveillance-broadcast (ADS-B) data,
including plane identification, timestamp, latitude, longitude, heading and speed, was
collected from all transmitting planes across the United States
during the week from 2013-3-22 1:30:0 to 2013-3-28 12:0:0 UTC. 
Then, individual ADS-B messages were connected together based on their plane
identification and timestamp to form trajectories, and erroneous trajectories
were filtered based on reasonable spatial/temporal bounds, yielding 17,895 unique trajectories.
Then, each trajectory was smoothed by training a Gaussian process, where the latitude and
longitude of points along the trajectory were the inputs, and the North 
and East components of plane velocity at those points were
the outputs. Next, the mean latitudinal and longitudinal velocities from the
Gaussian process were queried
for each point on a regular lattice across the United States (10 latitudes and 20
longitudes), and used to create a 400-dimensional feature vector for each
trajectory. Of the resulting 17,895 feature vectors, 600 were hand-labeled,
including a confidence weight between 0 and 1. The feature vectors 
were clustered in half-hour batch windows using D-Means (3 restarts), MCGMM Gibbs sampling (50 samples),
and the DP-Means algorithm~\citep{Kulis12_ICML} (3 restarts, run on the entire dataset in a single
batch). 

\begin{wraptable}[7]{r}{.4\textwidth}
  \centering \vspace*{-.2in}
  \captionsetup{font=scriptsize}
  \caption{Mean computational time \& accuracy on hand-labeled aircraft
  trajectory data}\label{tab:trajresults}
  {\small
  \begin{tabular}{l|c|l}
    \textbf{Alg.} & \textbf{\% Acc.} & \textbf{Time (s)}\\
    \hline
\rule{0pt}{10pt}D-Means & $55.9$ & $2.7\times 10^2$\\ 
    DP-Means &  $55.6$& $3.1\times 10^3$\\ 
    Gibbs & $36.9$ & $1.4\times 10^4$ 
  \end{tabular}
}
\end{wraptable}
The results of this exercise
are provided in Figure \ref{fig:trajresults} and Table \ref{tab:trajresults}. 
Figure \ref{fig:trajresults} shows the spatial and temporal properties 
of the 12 most popular clusters discovered by D-Means, demonstrating
that the algorithm successfully identified major flows of commercial
aircraft across the United States. Table \ref{tab:trajresults} corroborates
these qualitative results with a quantitative comparison of
the computation time and accuracy for the three algorithms tested over 20 trials.
The confidence-weighted accuracy was computed by taking the ratio between the sum of the weights for
correctly labeled points and the sum of all weights. The MCGMM Gibbs sampling
algorithm was handicapped as described in the synthetic experiment section. Of
the three algorithms, D-Means provided the highest labeling accuracy, while requiring
the least computational time by over an order of magnitude.

%
%

\subsection{Video Color Quantization}

Video is an ubiquitous form of naturally batch-sequential data --
the sequence of image frames in a video can be considered a sequence of batches
of data. In this experiment, the colors of a video were quantized via clustering.
Quantization, in which each image pixel is assigned to one of a small palette of
colors, can be used both as a component of image compression algorithms and as a visual effect. 
To find an appropriate palette for a given image,
a clustering algorithm may be applied to the set of colors from all pixels in 
the image, and the cluster centers are extracted to form the palette. 

Video frame sequences typically undergo smooth evolution interspersed with
discontinuous jumps; D-Means is well-suited to clustering such data, with
motion processes to capture the smooth evolution, and birth/death processes
to capture the discontinuities. As a color quantization algorithm, D-Means
yields a smoothly varying palette of colors, rather than a palette for each
frame individually. The parameters used were $\lambda = 800$, $T_Q = 15$,
$K_\tau = 1.1$. For comparison, two other algorithms were
run on each frame of the same video individually: DP-Means
with $\lambda = 800$ and K-Means with $K=20$.
The video selected for the experiment was Sintel\footnote{The Durian Open Movie Project, Creative Commons Attribution 3.0,
\url{https://durian.blender.org}}, consisting of 17,904 frames at $1280\times 544$ resolution. 

The results of the experiment are shown in Figure \ref{fig:video} with
example quantized frames shown in Figure \ref{fig:vid-variety}. The major
advantage of using D-Means over DP-Means or K-Means is that color flicker in the
quantization is significantly reduced, as evidenced by
Figures \ref{fig:vid-numclus}-\ref{fig:vid-framediffs}. This is due to the
cluster motion modeling capability of D-Means -- by enforcing that cluster
centers vary smoothly over time, the quantization in sequential frames is forced
to be similar, yielding a temporally smooth quantization. In contrast, the 
cluster centers for K-Means or DP-Means can
differ significantly from frame to frame, causing noticeable flicker to occur,
thus detracting from the quality of the quantization. 

\begin{figure}[h!]
  \centering
  \captionsetup{font=scriptsize}
\begin{subfigure}[b]{\columnwidth}
      \centering
      \includegraphics[width=\columnwidth]{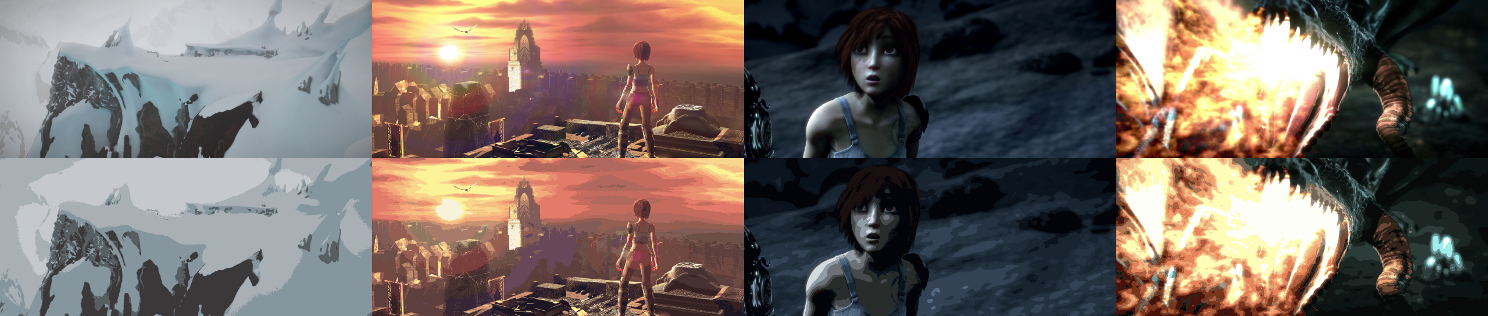}
  \captionsetup{font=scriptsize}
      \caption{}\label{fig:vid-variety}
    \end{subfigure}


\begin{subfigure}[b]{\columnwidth}
  \input{vidclusfig-arxiv}
  \vspace{-.7cm}
  \captionsetup{font=scriptsize}
      \caption{}\label{fig:vid-numclus}
    \end{subfigure}

    \begin{subfigure}[b]{\columnwidth}
      \centering
      \includegraphics[width=\columnwidth]{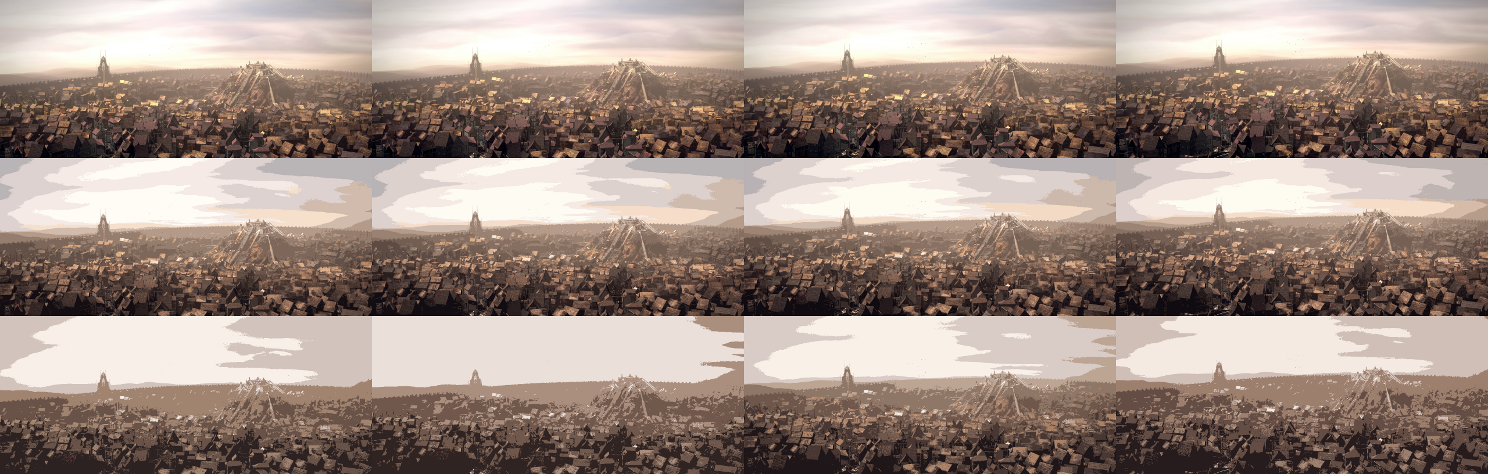}
  \captionsetup{font=scriptsize}
      \caption{}\label{fig:vid-flicker}
    \end{subfigure}

\begin{subfigure}[b]{\columnwidth}
      \centering
      \includegraphics[width=\columnwidth]{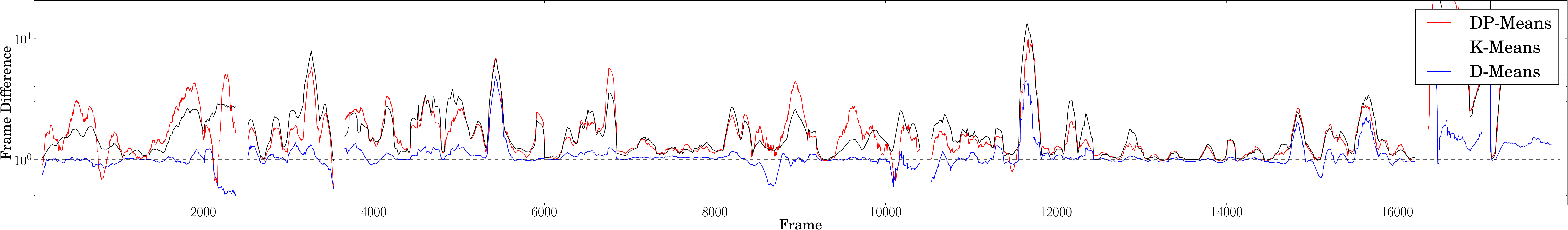}
  \captionsetup{font=scriptsize}
      \caption{}\label{fig:vid-framediffs}
    \end{subfigure}
    \caption{Video color quantization results. (\ref{fig:vid-variety}): Four frames from the original video (top row) and quantized video
    (bottom row) with D-Means.
    (\ref{fig:vid-numclus}): The number of clusters over the whole video (top) 
    and a zoomed window on frames 3550-4350 (bottom). Note 
    the smoothness of the number of clusters over time using D-Means with respect to
    DP-Means. (\ref{fig:vid-flicker}): A sequence of 4 frames from the original
    video (top row), quantized video with D-Means (middle row), and quantized
    video with DP-Means (bottom row). Note the flicker in the quantization,
    particularly in the upper region of the frames, when
    using DP-Means in contrast to using D-Means.
    (\ref{fig:vid-framediffs}): The difference between each frame and the
  previous frame for quantized video, normalized by the same quantity for the
original video. D-Means tracks the original video (dashed horizontal black line) more closely than both
K-Means and DP-Means. Missing values are when the original video had no
frame difference (normalization by 0).}\label{fig:video}
\end{figure}

\subsection{Human Segmentation in Point Clouds}
\begin{figure}[t!]
\captionsetup{font=scriptsize}
\centering
\begin{center}
  \begin{subfigure}[t]{.24\linewidth}
    \raisebox{.1cm}{\includegraphics[width=\textwidth]{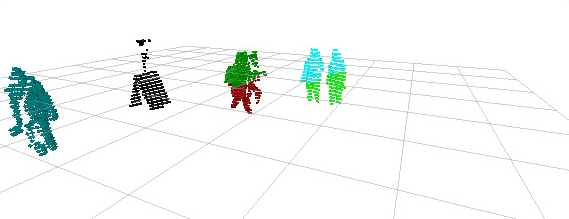}}
  \end{subfigure}
  \begin{subfigure}[t]{.24\linewidth}
    \raisebox{.1cm}{\includegraphics[width=\textwidth]{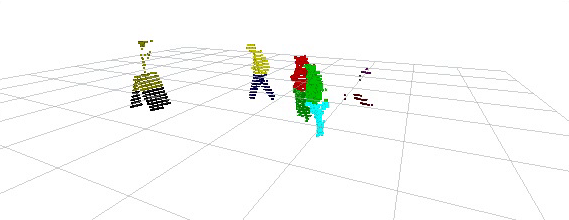}}
  \end{subfigure}
\begin{subfigure}[t]{.24\linewidth}
    \raisebox{.1cm}{\includegraphics[width=\textwidth]{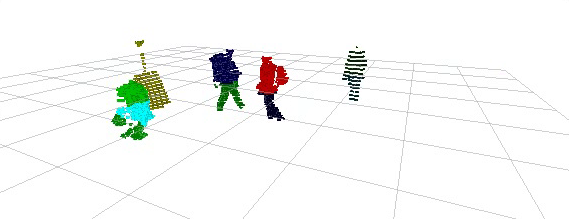}}
  \end{subfigure}
\begin{subfigure}[t]{.24\linewidth}
    \raisebox{.1cm}{\includegraphics[width=\textwidth]{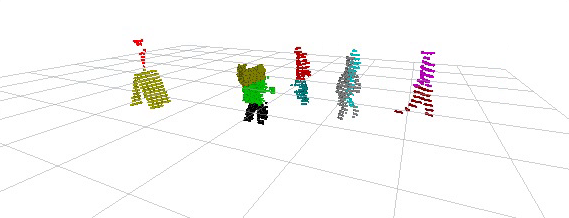}}
  \end{subfigure}
  \begin{subfigure}[t]{.24\linewidth}
    \raisebox{.1cm}{\includegraphics[width=\textwidth]{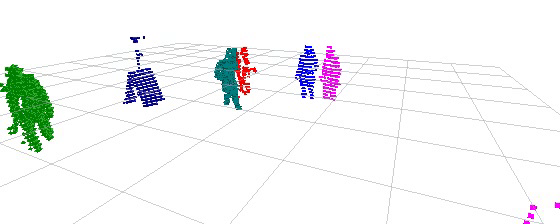}}
  \end{subfigure}
  \begin{subfigure}[t]{.24\linewidth}
    \raisebox{.1cm}{\includegraphics[width=\textwidth]{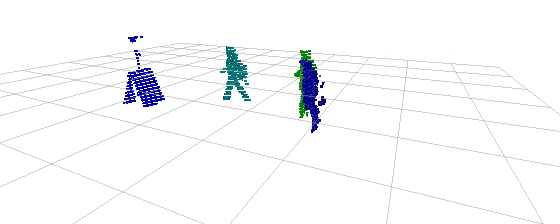}}
  \end{subfigure}
\begin{subfigure}[t]{.24\linewidth}
    \raisebox{.1cm}{\includegraphics[width=\textwidth]{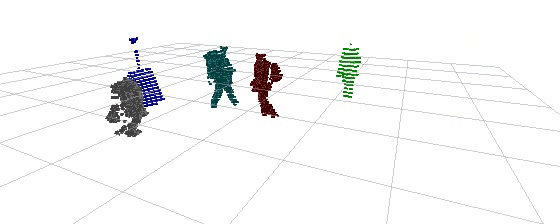}}
  \end{subfigure}
\begin{subfigure}[t]{.24\linewidth}
    \raisebox{.1cm}{\includegraphics[width=\textwidth]{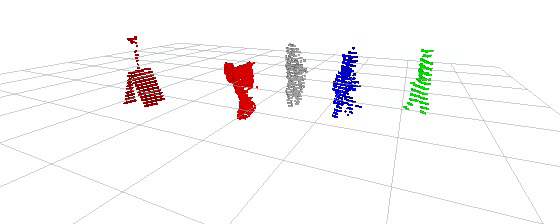}}
  \end{subfigure}
    \centering\caption{Comparison of D-Means (top row) and SD-Means (bottom row) on four
    frames of the point cloud dataset. Color indicates cluster assignment. Note
  that D-Means is forced to either oversegment humans or erroneously group them
together, while SD-Means correctly clusters them.}\label{fig:ptcld}
\end{center}
\end{figure}

\begin{figure}
\captionsetup{font=scriptsize}
\centering
\begin{center}
  \begin{subfigure}[t]{.24\linewidth}
    \raisebox{.1cm}{\includegraphics[width=\textwidth]{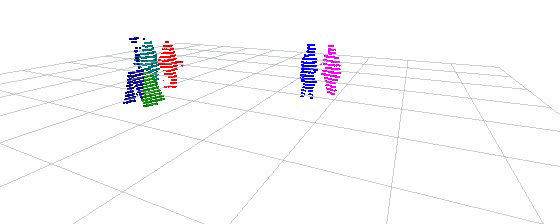}}
  \end{subfigure}
  \begin{subfigure}[t]{.24\linewidth}
    \raisebox{.1cm}{\includegraphics[width=\textwidth]{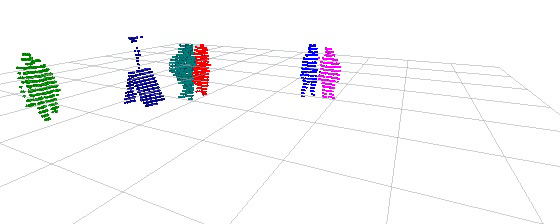}}
  \end{subfigure}
\begin{subfigure}[t]{.24\linewidth}
    \raisebox{.1cm}{\includegraphics[width=\textwidth]{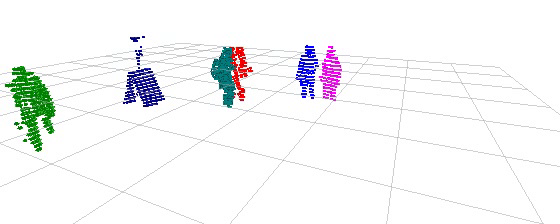}}
  \end{subfigure}
\begin{subfigure}[t]{.24\linewidth}
    \raisebox{.1cm}{\includegraphics[width=\textwidth]{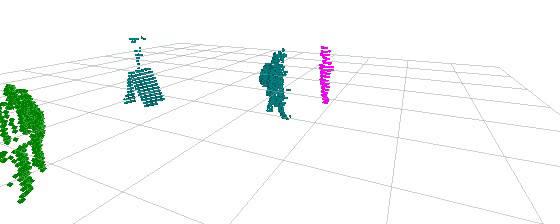}}
  \end{subfigure}
    \centering \caption{Cluster assignments from SD-Means for a sequence of four
      frames from
    the point cloud dataset. Color indicates cluster assignment. Note the
  temporal consistency of the clustering, and the flexibility to temporarily
remove clusters when humans are occluded.}\label{fig:ptcld2}
\end{center}
\end{figure}

Batch-sequential data is also prevalent in the field of mobile robotics, where
point cloud sequences provided by LiDAR
and RGB-D sensors are commonly used for perception.
In this experiment, D-Means and SD-Means were used to cluster 120 seconds of
LiDAR data, at 30 frames per second, of people walking in a foyer.
The experiment was designed to demonstrate the capability of SD-Means
to simultaneously cluster nonspherical shapes and track them over time. 
For comparison of the two algorithms, the point clouds were hand-labeled
with each human as a single cluster.

\begin{wraptable}[7]{r}{.4\textwidth}
  \centering 
  \captionsetup{font=scriptsize}
  \caption{Computational time \& accuracy on point cloud dataset.}\label{tab:ptcldresults}
  {\small
  \begin{tabular}{l|c|l}
    \textbf{Alg.} & \textbf{\% Acc.} & \textbf{Time (s)}\\
    \hline
\rule{0pt}{10pt}SD-Means & $86.8$ & $816.7$\\ 
    D-Means &  $53.6$& $38.1$\\
  \end{tabular}
}
\end{wraptable}
The results are shown in Figures \ref{fig:ptcld} - \ref{fig:ptcld2}
and Table \ref{tab:ptcldresults}. Figure \ref{fig:ptcld} shows 
that D-Means cannot capture the complex cluster structure present in the
dataset -- due to the spherical cluster assumption made by the algorithm, changing 
the parameter $\lambda$ causes D-Means to either over- or undersegment 
the humans. In contrast, SD-Means captures the structure and segments the humans
accordingly. Figure \ref{fig:ptcld2} demonstrates the ability of SD-Means to
correctly capture temporal evolution of the clusters in the dataset, due to its
derivation from the MCGMM. Table \ref{tab:ptcldresults} corroborates these findings numerically,
by showing that SD-Means provides much higher point labeling accuracy
on the dataset than D-Means (using 
the same label accuracy computation as outlined in previous experiments).

\section{Conclusions}\label{sec:conclusions}
This work presented the first small-variance analysis for Bayesian nonparametric
models that capture batch-sequential data containing temporally evolving clusters. 
While previous small-variance analyses have exclusively focused on models
trained on a single, static batch of data, this work provides evidence that the
small-variance methodology yields useful results for evolving models.
In particular, the present work developed two new clustering algorithms: The first, D-Means, was derived from a 
small-variance asymptotic analysis of the low-variance MAP filtering problem for
the MCGMM; and the second, SD-Means, was derived from a
kernelization and spectral relaxation of
the D-Means clustering problem. 
%
%
Empirical results from both synthetic and real data 
experiments demonstrated the performance of the algorithms compared
with contemporary probabilistic and hard clustering 
algorithms.

\section{Acknowledgements}
We would like to acknowledge support for this project
from the National Science Foundation (NSF CAREER Award 1559558) and
the Office of Naval Research (ONR MURI grant N000141110688). We would
also like to thank Miao Liu for discussions and help with
algorithmic implementations.

\newpage

\appendix
\section{The Equivalence Between Dynamic Means and Small-Variance
MCGMM Gibbs Sampling}\label{app:gibbsderiv}

\subsection{Setting Labels Given Parameters}\label{subsec:samplinglabelsgivenparams}
In the label sampling step, a datapoint $y_{it}$ can either create a new
cluster, join a current cluster, or revive an old, transitioned cluster. 
Thus, the label assignment probabilities are
\begin{align}
  &p(z_{it} = k | \dots) \propto \\
  &\left\{\begin{array}{c c}
\alpha(2 \pi (\sigma+\rho))^{-d/2} \exp{\left({-\frac{|| y_{it}
  - \phi
||^2}{2 (\sigma+\rho)}}\right)} & k = K_t+1 \\
   (c_{kt} + n_{kt})(2 \pi \sigma)^{-d/2} \exp{\left({-\frac{||
   y_{it} - \theta_{kt}
 ||^2}{2 \sigma}}\right)} & n_{kt} > 0\\
 q^\dtk c_{kt} (2 \pi (\sigma+\xi\dtk + \rho_{kt}))^{-d/2} \exp{\left({-\frac{|| y_{it} -
 \phi_{kt}
||^2}{2 (\sigma+\xi\dtk + \rho_{kt})}}\right)} & n_{kt} = 0 , \,
  \end{array}\right. 
\end{align}
and the parameter scaling in
(\ref{eq:smallvariancescaling}) yields
the following deterministic assignments as $\sigma\rightarrow 0$:
\begin{align}
  z_{it} = \argmin_{k}
          \left\{\begin{array}{l l}
          ||y_{it} - \theta_{kt}||^2 \quad &\text{if $\theta_k$ instantiated}\\
          Q\dtk + \frac{\gamma_{kt}}{\gamma_{kt}+1}|| y_{it} - \phi_{kt}||^2 \quad &\text{if
            $\theta_k$ old, uninstantiated}\\
          \lambda \quad &\text{if $\theta_k$ new} 
        \end{array}\right. . 
\end{align}
which is exactly the label assignment step from Section \ref{sec:dynmeans}.

\subsection{Setting Parameters Given Labels}\label{subsec:samplingparamsgivenlabels} 
In the parameter sampling step, the parameters are sampled using the
distribution 
\begin{align}
  p(\theta_{kt} | \{y_{it} : z_{it} = k\}) \propto p(\{y_{it} : z_{it}
= k\} | \theta_{kt}) p(\theta_{kt}).
\end{align}
Suppose there are $\dtk$ 
time steps where cluster $k$ was not observed, but
there are now $n_{kt}$ data points assigned 
to it in the current time step. Thus,
\begin{align}
p(\theta_{kt}) = \int_{\theta} T(\theta_{kt}
|\theta)p(\theta)\,
\mathrm{d}\theta, \, \, \theta \sim \mathcal{N}(\phi_{kt}, \rho_{kt}).
\end{align}
where $\phi_{kt}$ and $\rho_{kt}$ capture the uncertainty in cluster $k$
$\dtk$ timesteps ago when it was last observed.
Again using conjugacy of normal likelihoods and priors,
\begin{align}
\begin{aligned}
  \theta_{kt} | \{y_{it}: z_{it} = k\} &\sim
  \mathcal{N}\left( \theta_{\text{post}},
  \sigma_{\text{post}}\right)\\ 
  \theta_{\text{post}} =  \sigma_{\text{post}}\left(\frac{\phi_{kt}}{\xi \dtk +
    \rho_{kt}} + 
    \frac{\sum_{i:z_{it}=k}
y_{it}}{\sigma}\right)&, \,
\sigma_{\text{post}} =
\left(\frac{1}{\xi \dtk + \rho_{kt}} + \frac{n_{kt}}{\sigma}\right)^{-1}.
\end{aligned}\label{eq:transpost}
\end{align}
Similarly to the label assignment step, let $\xi = \tau \sigma$. Then as long as
$\rho_{kt} = \sigma/w_{kt}$, $w_{kt} > 0$ 
(which holds if equation (\ref{eq:transpost}) is used
to recursively keep track of the parameter posterior), taking the asymptotic limit 
of this as $\sigma\rightarrow 0$ yields
\begin{align}
  \theta_{kt} &= \frac{\phi_{kt} (w_{kt}^{-1}+\dtk \tau)^{-1} +
  \sum_{i\in\indices_{kt}}y_{it} } { (w_{kt}^{-1}+\dtk \tau)^{-1} +  n_{kt}} = 
  \frac{\gamma_{kt}\phi_{kt} + \sum_{i\in\indices_{kt}}y_{it}}
    {\gamma_{kt} + n_{kt}}
\end{align}

\section{Proof of Stability of the Old Cluster Center Approximation}\label{app:phiapproxproof}
\begin{proofof}{Theorem \ref{thm:phierrorbound}}
  The error of the approximation $\epsilon_{t+1}$ at time $t+1$ can be bounded
  given an error bound for the optimization $\epsilon$ and the error 
  bound at time $t$:
\begin{align}
  \begin{aligned}
  \epsilon_{t+1}&=\|\phi_{k(t+1)}-\phi^\star_{k(t+1)}\|_2
  = \|\phi_{k(t+1)} - \bar{\phi}_{k(t+1)}
  +\bar{\phi}_{k(t+1)}-\phi^\star_{k(t+1)}\|_2.
\end{aligned}
\end{align}
Using the fact that $\bar{\phi}_{k(t+1)}$ is the dense approximation of $\phi^\star_{k(t+1)}$,
i.e.~including all data assigned to cluster $k$ at timestep $t$, 
this is equal to
\begin{align}
  \begin{aligned}
&= \|\phi_{k(t+1)}-\bar{\phi}_{k(t+1)}
+\left(\frac{\gamma_{kt}(\phi_{kt}-\phi^\star_{kt})}{\gamma_{kt}+n_{kt}}\right)\|_2\\
&\leq \|\phi_{k(t+1)}-\bar{\phi}_{k(t+1)}\|_2
+\frac{\gamma_{kt}}{\gamma_{kt}+n_{kt}}\|\phi_{kt}-\phi^\star_{kt}\|_2\\
&\leq \epsilon+\frac{\gamma_{kt}}{\gamma_{kt}+n_{kt}}\epsilon_t.
\end{aligned}
\end{align}
Now using the facts that $\gamma_{kt} = \left(w_{kt}^{-1} + \tau
\dtk\right)^{-1}$, $w_{kt} > 0$, $\dtk \geq 1$, and $n_{kt} \geq 1$,
\begin{align}
 \epsilon_{t+1}&\leq \epsilon
  +\frac{1}{1+\tau}\epsilon_t.
\end{align}
Expanding the recursion and using the geometric series formula,
\begin{align}
  \begin{aligned}
  \epsilon_t &\leq \sum_{i=0}^{t-1}\left(\frac{1}{1+\tau}\right)^{i}\epsilon =
\epsilon\frac{1-\left(\frac{1}{1+\tau}\right)^{t}}{1-\left(\frac{1}{1+\tau}\right)}\leq
    \epsilon\left(1+\frac{1}{\tau}\right).
  \end{aligned}
\end{align}
\end{proofof}

\section{Proof of Integrality/Optimality of the Cluster Matching}\label{app:clustermatchingproof}
\begin{lemma} [Ghouila-Houri 1962]\label{lem:totuni} Consider a matrix $A$ with $n$ rows of the form
  $r \in \mathbb{R}^{1\times n}$. 
  $A$ is totally unimodular
  iff for every subset $R$ of the rows, there is a partition of $R$,
  $R_+\cup R_- = R$, $R_+ \cap R_- = \emptyset$, such that
  \begin{align}
    \left[ \sum_{r \in R_+} r\right] - \left[\sum_{r \in R_-}r\right] \in
    \{0, 1, -1\}^{1\times n}.
  \end{align}
\end{lemma}
\begin{proof}
  This result is proven in~\citep{Schrijver03}. 
\end{proof}
\begin{proofof}{Lemma \ref{lem:totallyunimodular}}
The constraints of the linear program (\ref{eq:clustermatching})
are as follows:
\begin{align}
  \begin{aligned}
    &\sum_{k=1}^{K_{t-1}} c_{lk} \leq 1, \, \, 1 \leq j \leq  \left|\clusa_t\right|\\
     &\sum_{l=1}^{\left|\clusa_t\right|} c_{lk} \leq 1, \, \, 1 \leq k \leq K_{t-1}\\
     &c \geq 0, \, \,  c \in \mathbb{R}^{\left|\clusa_{t}\right|\times K_{t-1}}.
  \end{aligned}
\end{align}
To show that this set of constraints has a totally unimodular
matrix, we use Lemma \ref{lem:totuni}. Given any subset $R$ of
the constraints, let $R'_+$ consist of the coefficient vectors from all constraints of the form
$\sum_{l=1}^{\left|\clusa_t\right|} c_{lk} \leq 1$, and
let $R'_-$ consist of the coefficient vectors from all constraints of the form $\sum_{k=1}^{K_{t-1}} c_{lk} \leq 1$. Then since
every variable $c_{lk}$ has a $1$ coefficient in exactly one constraint
of the former type and a $1$ coefficient in exactly one constraint of the latter
type, 
\begin{align}
    \left[ \sum_{r \in R'_+} r\right] - \left[\sum_{r \in R'_-}r\right] \equiv
    \tilde{r} \in \{0, 1, -1\}^{1\times N}.
\end{align}
Then, let $R''_+$ consist of the rows in $R$ corresponding to the individual
nonnegativity constraints $c_{lk} \geq 0$ in which $\tilde{r}$ has a 
$-1$ or $0$ entry in the column for $c_{lk}$, and let $R''_-$ be the remaining
nonnegativity constraint rows. Finally, let $R_+ = R'_+\cup R''_+$, and let
$R_- = R'_-\cup R''_-$. With this setting of $R_+/R_-$, the condition 
in Lemma \ref{lem:totuni} holds, and therefore the constraint matrix is totally unimodular.
\end{proofof}
\begin{proofof}{Theorem \ref{thm:clustermatching}} By adding slack variables
  $s_l = 1-\sum_{k=1}^{K_{t-1}} c_{lk}$, the optimization
  (\ref{eq:clustermatching}) can be rewritten in the following form:
   \begin{align}
     \begin{aligned}\label{eq:cmatch2}
    c^\star = \argmin_{c, s} \quad & \sum_{l=1}^{\left|\clusa_t\right|}
    \left(s_l m_l 
    + \sum_{k=1}^{K_{t-1}} c_{lk} m_{lk}\right)\\
     \mathrm{s.t.}\quad & \sum_{k=1}^{K_{t-1}} c_{lk} + s_l= 1, \, \, 1 \leq l \leq
     \left|\clusa_t\right|\\
     &\sum_{k=1}^{\left|\clusa_t\right|} c_{lk} \leq 1, \, \, 1 \leq k \leq
     K_{t-1}\\
     &c, s\geq 0, \quad c \in \mathbb{R}^{\left|\clusa_{t}\right|\times K_{t-1}},
     \quad s \in
     \mathbb{R}^{\left|\clusa_t\right|}
  \end{aligned}
  \end{align}
  where $m_l$ is the cost of making temporary cluster $l$ a new cluster, and
  $m_{lk}$ is the cost of reinstantiating
  old cluster $k$ with temporary cluster $l$,
    \begin{align}
      \begin{aligned}
      m_l &= \lambda - \frac{1}{n_l}\sum_{i,
      j\in\indices_r}\krnmat{Y}{Y}_{ij}\\
m_{lk} &= Q\dtk +
  \frac{\gamma_{kt}n_{l}\krnmat{\Phi}{\Phi}_{kk}}{\gamma_{kt}+n_{l}}
    -\!\sum_{i\in\indices_{l}}\!\frac{2\gamma_{kt}\krnmat{Y}{\Phi}_{ik}+\sum_{j\in\indices_{l}}\krnmat{Y}{Y}_{ij}}{\gamma_{kt}+n_{l}}.
      \end{aligned}
    \end{align}
However, by Lemma \ref{lem:totallyunimodular}, the optimization
(\ref{eq:clustermatching}) has a totally unimodular constraint matrix, and as such it
produces an integer solution.  Thus, the constraints 
$c, s\geq 0, \, c \in \mathbb{R}^{\left|\clusa_{t}\right|\times K_{t-1}}, 
\, s \in \mathbb{R}^{\left|\clusa_t\right|}$ in (\ref{eq:cmatch2}) can be replaced with integer constraints
$c \in \{0, 1\}^{\left|\clusa_{t}\right|\times K_{t-1}}, \, s \in \{0, 1\}^{\left|\clusa_t\right|}$.
By inspection of $m_l$ and $m_{lk}$, the resulting integer optimization directly minimizes (\ref{eq:krnmin}) with
respect to matchings between temporary clusters $l=1, \dots, \left|\clusa_t\right|$
and old/new clusters in the current timestep $t$.
\end{proofof}

\vskip 0.2in
\bibliographystyle{plainnat}
\bibliography{main}

\end{document}